\documentclass[10pt]{article}

\usepackage{setspace}

\usepackage{amsmath,amsthm, amssymb, latexsym}
\usepackage{amsfonts}
\usepackage{graphicx}
\usepackage{multicol}
\usepackage{caption}
\usepackage{color}
\usepackage{xcolor}
\usepackage{graphicx}
\usepackage{amsmath}
\usepackage{multicol}
\usepackage{float}
\usepackage[a4paper, total={6in, 8in}]{geometry}
\usepackage{color}
\usepackage{subfigure}
\usepackage{url}
\usepackage{bbm}
\usepackage{tabularx}
\usepackage{xcolor}
\usepackage{grffile}
\usepackage[english]{babel}
\usepackage{hyperref}
\usepackage{tikz}
\usepackage{verbatim}
\usepackage{algorithm}
\usepackage{algpseudocode}
\usetikzlibrary{tikzmark,calc,,arrows,shapes,decorations.pathreplacing}
\tikzset{every picture/.style={remember picture}}
\newtheorem{Theorem}{Theorem}[section]
\newtheorem{Definition}[Theorem]{Definition}
\newtheorem{Proposition}[Theorem]{Proposition}
\newtheorem{Lemma}[Theorem]{Lemma}
\newtheorem{Notation}[Theorem]{Notation}
\newtheorem*{lemma-non}{Lemma}
\newtheorem{Corollary}[Theorem]{Corollary}
\newtheorem*{corollary-non}{Corollary}

\newtheorem{Remark}[Theorem]{Remark}
\newtheorem*{remark-non}{Remark}
\newtheorem{Example}[Theorem]{Example}
\newtheorem{Assumption}[Theorem]{Assumption}
\def\qed{\hfill\hbox{\hskip 6pt\vrule width6pt height7pt
		depth1pt  \hskip1pt}\bigskip}
\newcommand{\N}{\mathbb{N}}
\newcommand{\Z}{\mathbb{Z}}
\newcommand{\R}{\mathbb{R}}

\newcommand{\E}{\mathbb{E}}
\newcommand{\bb}[1]{\boldsymbol{#1}}

\newcommand{\ra}{\rightarrow}
\newcommand{\norm}[1]{\lVert #1 \rVert}

\definecolor{crimson}{RGB}{220, 20, 60}

\begin{document}
	
	\title{Mixed moving average field guided learning for spatio-temporal data}
	\author{\small{Imma Valentina Curato\footnote{TU Chemnitz, Faculty of Mathematics, Reichenhainer Str. 39, 09126 Chemnitz, Germany. {\sc E-mail:} imma-valentina.curato@math.tu-chemnitz.de.} \, , Orkun Furat\footnote{Ulm University, Institute of Stochastics, Helmholtzstrae 18, 89069 Ulm, Germany. {\sc E-mail:} orkun.furat@uni-ulm.de.} \,,  Lorenzo Proietti \footnote{TU Chemnitz, Faculty of Mathematics, Reichenhainer Str. 39, 09126 Chemnitz, Germany. {\sc E-mail:} lorenzo.proietti@math.tu-chemnitz.de.} and Bennet Str\"oh \footnote{Imperial College, Department of Mathematics, South Kensington Campus, SW7 2AZ London, United Kingdom. {\sc E-mail:} b.stroh@imperial.ac.uk.}}}
	
\onehalfspacing
	
	\maketitle
	
	\textwidth=160mm \textheight=225mm \parindent=8mm \frenchspacing
	\vspace{3mm}

\begin{abstract}
Influenced mixed moving average fields are a versatile modeling class for spatio-temporal data. However, their predictive distribution is not generally known. Under this modeling assumption, we define a novel spatio-temporal embedding and a theory-guided machine learning approach that employs a generalized Bayesian algorithm to make ensemble forecasts. We use Lipschitz predictors to determine fixed-time and any-time PAC Bayesian bounds in the batch learning setting. Performing causal forecast is a highlight of our methodology as its potential application to data with temporal and spatial short and long-range dependence. We then test the performance of our learning methodology by using linear predictors and data sets simulated from a spatio-temporal Ornstein-Uhlenbeck process.
\end{abstract}

	\section{Introduction}
\label{intro}

Analyzing spatio-temporal data introduces various methodological challenges. These include determining models that can account for the serial correlation observed along their temporal and spatial dimensions and that, at the same time, can also enable forecasting tasks. 
Statistical models as Gaussian processes \cite{Bai12}, \cite{DK13}, \cite{GP}, and \cite{Stein}; spatio-temporal kriging \cite{CW11}, and \cite{Monte}; space-time autoregressive moving average models \cite{MovingAverage}; point processes \cite{Point}, and hierarchical models \cite{CW11} are very versatile in modeling the spatio-temporal correlation observed in the data and can deliver forecasts once the variogram or the data distribution (up to a set of parameters) is carefully chosen in relation to the studied phenomenon and practitioners' experience. In a nutshell, such choices allow access to the models' predictive distribution. 

Suppose we want to avoid making any explicit choice regarding the data distribution. In this case, we can alternatively use deep learning methodologies to perform forecasting tasks in a spatio-temporal framework, see \cite{Amato}, \cite{Nature}, \cite{Shi17}, \cite{Shi18} for a review, or a video frame prediction algorithm as in \cite{MS16} and \cite{AImink}. Deep learning techniques can successfully extract spatio-temporal features and learn the \emph{inner law} of an observed spatio-temporal system. However, these models lack interpretability, i.e., it is not possible to infer the correlation and causal relationship between variables in different space-time points that the models consider, and typically, no proof of their generalization performances is available in a spatio-temporal framework for dependent data. On the other hand, video prediction algorithms presented in \cite{MS16} and \cite{AImink} retain a \emph{causal interpretation} of the relationship between different space-time points. However, as in the case of deep learning algorithms, there is no  proof of their generalization performances. 

This paper proposes a novel theory-guided machine learning methodology for spatio-temporal data that enables one-time ahead ensemble forecasts based on moment assumptions (no further assumptions on the data distribution are needed), an opportune spatio-temporal embedding, and a generalized Bayesian algorithm. A theory-guided machine learning methodology is a hybrid procedure that employs a stochastic model in synergy with a learning algorithm. Such methodologies have started to gain prominence in several scientific disciplines such as earth science, quantum chemistry, bio-medical science, climate science, and hydrology modeling as, for example, described in \cite{T24}, \cite{theoryg}, \cite{PI}, \cite{Intrinsic}, and \cite{Nature}. It is important to emphasize that in these works the data are typically considered outputs of (deterministic) dynamical systems driven by partial differential equations. We define a theory guided machine learning methodologies for data generated by a random field.

We call our methodology \emph{mixed moving average field guided learning} or MMAF-guided learning. In particular, we analyze \emph{raster data cubes} \cite{Raster}, which are described for dimension $d=2$ in Section \ref{sec4.2}, and that are nowadays generated in environmental monitoring, from satellite observations, and climate and weather numerical models' outputs. Our methodology applies, in general, to raster data having spatial dimension $d \geq 1$, and we assume that the such data are generated by an \emph{influenced mixed moving average field} (MMAF, in short), see Definition \ref{mmaf}. Such a class of random fields has been introduced in \cite{CSS20} and allows modeling the correlation and the causal relationship in different space-time points using ambit sets, see \cite{Ambit} \cite{BNS11} \cite{STOU}, and \cite{MSTOU}. Such models have been so far employed to model data in environmental monitoring \cite{HIG02,STOU}, imaging analysis \cite{brain}, and electricity networks \cite{Graph22}. They allow modeling Gaussian and non-Gaussian distributed data; they can be non-Markovian and have non-separable covariance functions. Moreover, they are stationary and $\theta$-lex weakly dependent, as proven in \cite[Section 3.3]{CSS20}, and allow modeling temporal and spatial short and long-range dependence. A drawback of employing MMAF in forecasting tasks is that their predictive distribution is not explicitly known. To our knowledge, the only available results on the predictive distribution of an MMAF in a spatio-temporal framework can be found in \cite[Theorem 13]{STOU} for a Gaussian spatio-temporal Ornstein-Uhlenbeck process.

We then select a generalized Bayesian algorithm, i.e., a so-called \emph{randomized estimator}, which is a regular conditional probability on the class of the Lipschitz functions $\mathcal{H}$, see Definition \ref{datadep}. The latter is then employed to make ensemble forecasts. We call a function $h \in \mathcal{H}$ a \emph{predictor}. Linear models, neural network architectures with feed-forward and convolutional modules \cite{LipschitzConv,DeepLipschitz}, and Lipschitz modifications of the transformer architecture \cite{LipschitzSelfAttention, Lipformer} are between the predictors belonging to $\mathcal{H}$. The selection of a \emph{randomized estimator} is guided by the design of a spatio-temporal embedding of an observed \emph{raster data cube}, which gives us the training data set on which the estimator can be learned. The precise construction of the spatio-temporal embedding is given in Section \ref{sec4.2} and depends on several parameters that must be opportunely tuned. Therefore, given a \emph{raster data cube}, there exist different ways to pre-process them into a training data set. Such a procedure ensures casual forecasts and good generalization performance, which, in the context of our paper, means obtaining non-vacuous PAC Bayesian bounds. In the paper, we present how to \emph{guide} the design of a Dirac delta mass concentrated on the Empirical Risk Minimizer and a \emph{randomized Gibbs estimator}. 



\subsection{Setting}
\label{setting}
Let $\bb{S}:= ((\bb{X}_i,\bb{Y}_i)^\top)_{i \in \Z}$ be a random vector defined on the canonical probability space $(\Omega, \mathcal{F}, \mathbb{P})$, see \cite[Chapter 7]{Billi} for more details on its definition. Here, each $(\bb{X}_i,\bb{Y}_i)$ is identically distributed and has values in $\mathcal{X}\times \mathcal{Y}$ (Euclidean Spaces). In particular, $\Omega$ is the space of all possible trajectories or realizations of the process $\bb{S}$. We further assume that $\bb{S}$ is a sample from an MMAF with finite second moments as carefully described in Section \ref{sec4.2}. $\bb{S_m}$ indicates a finite dimensional distribution of the process $\bb{S}$ of length $m$, and a \emph{training data set} is one of its realization which we indicate with $S_m:=((X_i,Y_i)^{\top})_{i=1}^m$ throughout. We also call $S:=((X_i,Y_i)^{\top})_{i \in \Z}$ a realization from $\bb{S}$.

Let $\mathcal{H}$ be the set of all Lipschitz functions $h:\mathcal{X}\to\mathcal{Y}$, and $L:\mathcal{X}\times \mathcal{Y} \to [0,\infty]$ a loss function. We define the \emph{generalization error (out-of-sample risk)} as
\begin{equation}
	\label{gen_err}
	R(h)= \E[ L(h(\bb{X}),\bb{Y}) ], \
\end{equation}
where $(\bb{X},\bb{Y})$ indicate a general example belonging to $\bb{S}$, and the \emph{empirical error (in-sample risk)}
\begin{equation}
	\label{emp_err}
	r(h,\omega)=\frac{1}{m} \sum_{i=1}^m L(h(\bb{X_i}(\omega)),\bb{Y_i}(\omega)),
\end{equation}
which is determined for a particular realization $\omega \in \Omega$ of the finite dimensional distribution $\bb{S}_m$. 
The function $L$ is used to measure the discrepancy between a predicted output $h(\bb{X})$ and the true output $\bb{Y}$. Using the in-sample risk, we measure the performance of a given predictor $h$ just over an observed training data set $S_m$. In contrast, the out-of-sample risk gives us the performance of a predictor depending on the unknown distribution of the data $\mathbb{P}$. We then need a guarantee that a selected predictor will perform well when used on a set of out-of-sample observations, i.e., not belonging to $S_m$. We can also rephrase the problem as finding a predictor $h$ for which the difference between the out-of-sample and in-sample risk $R(h)-r(h,\omega)$ is as small as possible. We call the latter \emph{generalization gap}. The classical PAC framework aims to find a bound on the generalization gap that holds with high probability $\mathbb{P}$; see, for example, \cite{SRM} and \cite{VP2000}. Such probability inequality is also called a \emph{generalization bound}. The acronym PAC stands for Probably Approximately Correct and may be traced back to \cite{V84}. A PAC inequality states that with an arbitrarily high probability (hence "probably"), the performance (as provided by the generalization gap) of a learning algorithm is upper-bounded by a term decaying to an optimal value as more data is collected (hence "approximately correct"). Note that we drop the dependence on $\omega \in \Omega$ in $r(h)$ to ease the notations in the following.

In the paper, we use a PAC Bayesian approach, also known as \emph{generalized Bayesian approach}. First, we select a \emph{reference distribution} $\pi$ on the space $(\mathcal{H},\mathcal{T})$, where $\mathcal{T}$ indicates a $\sigma$-algebra on the space $\mathcal{H}$. The reference distribution gives a \emph{structure} on the space $\mathcal{H}$, which we can interpret as our belief that certain predictors will perform better than others. The choice of $\pi$, therefore, is an indirect way to make the size of $\mathcal{H}$ come into play; see \cite[Section 3]{C2004} for a detailed discussion on the latter point. Therefore, $\pi$ belongs to $\mathcal{M}_+^1(\mathcal{H})$, which denotes the set of all probability measures on the measurable set $(\mathcal{H},\mathcal{T})$. We then aim to determine a \emph{randomized estimator}. To introduce the latter, we need first the following definition.
\begin{Definition}
	\label{regular}
	Let $(\Omega, \mathcal{F},\mathbb{P})$ be a probability space,   $(\mathcal{H},\mathcal{T})$ a measurable space,  $\bb{\hat{h}}:\Omega \to \mathcal{H}$ a random element, and $\mathcal{G}$ a sub-$\sigma$-algebra of $ \mathcal{F}$. The function   
	
	\begin{align*}
		\mathbb{P}_{\bb{\hat{h}}|\mathcal{G}}(\cdot|\mathcal{G})(\cdot): & \,\Omega \times \mathcal{T} \to [0,1]\\
		&(\omega,E) \to \mathbb{P}_{\bb{\hat{h}}|\mathcal{G}}(E|\mathcal{G})(\omega)
	\end{align*}
	
	is a \emph{regular conditional distribution} of $\bb{\hat{h}}$ given $\mathcal{G}$ if:
	\begin{itemize}
		\item for any $E \in \mathcal{T}$, the map $\omega \to \mathbb{P}_{\bb{\hat{h}}|\mathcal{G}}(E|\mathcal{G})(\omega) $ is $\mathcal{G}$-measurable and a variant of the conditional probability $\mathbb{P}(\bb{\hat{h}} \in E | \mathcal{G})$, i.e.,
		\[
		\mathbb{P}_{\bb{\hat{h}}|\mathcal{G}}(E|\mathcal{G})(\omega)= \mathbb{P}(\bb{\hat{h}} \in E | \mathcal{G})(\omega) \,\,\, \text{a.s.}
		\]
		\item for any $\omega \in \Omega$, $\mathbb{P}_{\bb{\hat{h}}|\mathcal{G}}(\cdot|\mathcal{\mathcal{G}})(\omega)$  is a probability measure on $(\mathcal{H},\mathcal{T})$.
	\end{itemize}
\end{Definition}

We assume throughout that the measurable space $(\mathcal{H},\mathcal{T})$ is a Borel space, i.e. $\mathcal{T}$ is a countably generated $\sigma$-algebra. Then, the regular conditional distribution of $\bb{\hat{h}}$ given $\mathcal{G}$  exists, see  \cite[Theorem 5, Chapter II.7]{shiryaev}. We can now give a formal definition of a \emph{randomized estimator} that follows \cite[Section 1]{C2007}.

\begin{Definition}
	\label{datadep}
	Let $\bb{S}=((\bb{X}_i,\bb{Y}_i)^{\top})_{i \in \Z}$ be a random vector on the probability space $(\Omega, \mathcal{F}, \mathbb{P})$, $(\mathcal{H},\mathcal{T})$ a Borel space, and  $\mathcal{G}_m=\sigma\{\bb{S}_m\}$.  We define a \emph{randomized estimator} $\hat{\rho}$ as the regular conditional distribution of $\bb{\hat{h}}$ given $\mathcal{G}_m$, i.e., for all $\omega \in \Omega$, and $E \in \mathcal{T}$, 
	\[
	\hat{\rho}(E,\omega):= \mathbb{P}_{\bb{\hat{h}}|\mathcal{G}_m}(E|\mathcal{G}_m)(\omega).
	\] 
\end{Definition}

From now on, we indicate with $\pi[\cdot]$, $\hat{\rho}[\cdot]$ the expectations with respect to the reference distribution and the randomized estimator. The latter is $\mathbb{P}$-almost surely a conditional expectation w.r.t. $\mathcal{G}_m$. We simply indicate with $\E[\cdot]$ the expectation w.r.t. the probability distribution $\mathbb{P}$. Moreover, we call $\hat{\rho}[R(h)]$ and $\hat{\rho}[r(h)]$ \emph{the average generalization error and the average empirical error}, respectively. 

To evaluate the generalization performance of a randomized estimator $\hat{\rho}$, we determine a so-called \emph{PAC Bayesian bound}, which is a bound on the \emph{(average) generalization gap} defined as  $\hat{\rho}[R(h)]-\hat{\rho}[r(h)]$ holding with high probability $\mathbb{P}$. PAC-Bayesian bounds have proven over the past two decades successful in addressing various learning problems such as classification, sequential or batch learning, and deep learning \cite{Deep1,Primer}. 

MMAF-guided learning applies to bounded loss functions, see Remark \ref{accuracy} for more details on this point. Throughout, for $\epsilon >0$ called the \emph{accuracy level}, we define the \emph{truncated absolute loss} as
\begin{align}
	&L^{\epsilon}(h(\bb{X}),\bb{Y}))= L(h(\bb{X}),\bb{Y})) \wedge \epsilon. \label{loss_const}
\end{align}
The generalization error is then indicated with $R^{\epsilon}(h)= \E[ L^{\epsilon}(h(\bb{X}),\bb{Y}) ]$ and the empirical error with $r^{\epsilon}(h,\omega)=\frac{1}{m} \sum_{i=1}^m L^{\epsilon}(h(\bb{X_i}(\omega)),\bb{Y_i}(\omega))$ for $\omega \in \Omega$. We also drop in this case the dependence on $\omega \in \Omega$ in the empirical error's notation, and indicate the average generalization gap with $\hat{\rho}[R^{\epsilon}(h)]-\hat{\rho}[r^{\epsilon}(h)]$. We then derive PAC Bayesian bound for the latter.

\subsection{Contributions and Outline}

The PAC Bayesian bounds proven in the paper are the first results in the literature for $\theta$-lex weakly dependent data, i.e., data generated by a stationary $\theta$-lex weakly dependent random field $\bb{Z}=(\bb{Z}_t(x))_{(t,x) \in \R \times \R^d}$. The latter is a novel notion of dependence introduced in \cite{CSS20}. When considering a random field sampled on $\mathbb{Z} \times \mathbb{Z}^d$ such that $\E[|\bb{Z}_0(0)|^p] < \infty$ for $p >1$, then $\theta$-lex weak dependence is a more general notion than $\alpha_{\infty,v}$ and $\phi_{\infty,v}$-mixing for random fields being $v \in \N \cup \{\infty\}$, and $\alpha$ and $\phi$-mixing in the particular case of stochastic processes, see \cite[Section 2.3]{CSS20}. In particular, an MMAF is a $\theta$-lex weakly dependent random field, which can then be used to model very general frameworks; see Remark \ref{rate_literature} and Appendix A.

We then analyze fixed-time and any-time PAC Bayesian bounds for data generated by an MMAF. In particular, our fixed-time bounds are explicitly stated in the function of \emph{one single} $\theta$-lex coefficient of the underlying field. In the literature, two comparable bounds exist for time series data which are \emph{explicitly} stated in the function of $\alpha$-mixing  and $\theta_{1,\infty}$-coefficient. The former is presented in \cite[Section 3.2]{Hostile}, and is a bound where it appears a series of $\alpha$-mixing coefficients, which cannot be estimated from observed data and may diverge for power decaying coefficients. The latter is analyzed in \cite[Section 3]{AW12} for bounded $\theta_{1,\infty}$-coefficient, which also cannot be estimated from observed data. Instead, we can estimate the decay rate of the $\theta$-lex-coefficients for specific MMAF models, as for example, the spatio-temporal Ornstein Uhlenbeck (STOU, in short) process and its mixed version called MSTOU process defined in \cite{STOU, MSTOU}, respectively. In the paper, we also discuss the range of applicability of such estimation methodology to other types of MMAFs. The knowledge of the decay rate of the $\theta$-lex coefficient of an MMAF is fundamental in our methodology and allows us to \emph{guide} the choice of a \emph{randomized estimator}, as detailed in Section \ref{sec4.4}, and to assess its generalization performance. 

We start by proving a  \textit{fixed-time} PAC Bayesian bound that holds for all $\bb{S_m}$ with $m \geq 2$ and employs a novel exponential inequality for sums of weakly dependent processes. $\theta$-lex weak dependence is a notion of \emph{projective type} related to an $L_1$-norm, see \cite{CSS20} and Remark \ref{mix_rule1}, and \ref{mix_rule2}. In regards to projective type dependence notions for $L_p$-norm and $p \in [1,\infty]$, there have been proven moment inequalities for partial sums of weakly dependent random fields in \cite{C23} for $p \in [2,\infty]$ and for stochastic processes and $p=1$ in \cite{DN07}. To the best of our knowledge, another exponential inequality for a projective type dependence notion was obtained just for $p=\infty$ in \cite{AW12}. 

In order to give a complete overview of the range of applicability of MMAF-guided learning in the function of different choices of the spatio-temporal embedding, it is necessary to introduce a second fixed-time PAC Bayesian bound. The proof of such result involves the use of an any-time PAC Bayesian bound. The latter is proven using the Ville's maximal inequality for non-negative supermartingales \cite{Ville}, and it holds for all countable sequences of examples; this means \emph{simultaneously} for each $\bb{S_m}$ and $m \geq 1$. Several any-time PAC Bayesian bounds exist in the literature for general dependent data frameworks as discussed in \cite{SMartingale1}. However, they do not hold in the so-called \textit{batch learning case for dependent data}, which our proof covers for bounded losses. 

We then combine the any-time bound with the results proven in \cite{Hostile} applied to a particular residual process. Hence, we obtain a fixed-time PAC Bayesian bound where a moment inequality is involved in its proof instead of an exponential one. In our examples and simulation results, the latter bound allows us to define \emph{randomized estimators} using realistic spatio-temporal embeddings, which means training data sets that can best preserve the serial correlation observed in a given \emph{raster data cube}.

The paper is structured as follows. First, we review the MMAF framework in Section \ref{sec1} and describe its causal interpretation. In this section, we also introduce the STOU and MSTOU processes. These are isotropic random fields for which we compute novel bounds for their $\theta$-lex coefficients. The paper uses the latter to show feasible examples of MMAF-guided learning in the case of temporal and spatial short and long-range dependence. We introduce in Section \ref{sec4} the spatio-temporal embedding applied to raster data cubes and the PAC Bayesian bounds. In Section \ref{sec5}, we give a step-by-step description on how to apply in the practice MMAF-guided learning, and discuss the causal interpretation of the ensemble forecasts. We conclude by analyzing the performance of our methodology for a randomized Gibbs estimator on six simulated data sets from an STOU process with a Gaussian and a normal-inverse-Gaussian distributed L\'evy seed. Appendix A contains further details on the dependence notions discussed in the paper and a review of the estimation methodologies for STOU and MSTOU processes. Appendix B contains detailed proofs of the theoretical results presented in the paper. 

\section{Mixed moving average fields}
\label{sec1}

\subsection{Notations}
\label{sec2.1}
Throughout the paper, we indicate with $\N$ the set of positive integers, $\N_0$ the set of non-negative integers, and $\R^+$ the set of non-negative real numbers. As usual, we write $L^p(\Omega)$ for the space of (equivalence classes of) measurable functions $f:\Omega \to \R$ with finite $L_p$-norm $\|f\|_p$. When $\Omega=\R^n$ and $x \in \Omega$,  $\|x\|_1$ and $\norm{x}$ denote the $L_1$-norm and the Euclidean norm, respectively, and we define $\norm{x}_\infty=\max_{j=1,\ldots,n}|x^{(j)}|$, where $x^{(j)}$ represents the component $j$ of the vector $x$. 

To ease the notations in the following sections, we sometimes indicate the index set $\R\times \R^d$ by $ \R^{1+d}$.
$E\subset E^{\prime}$ denotes a not necessarily proper subset $E$ of a set $E^{\prime}$, $|E^{\prime}|$ denotes the cardinality of $E^{\prime}$ and $dist(E,E^{\prime})=\inf_{i\in E, j\in E^{\prime}} \norm{i-j}_\infty$ indicates the distance of two sets $E, E^{\prime}\subset \R^{1+d}$.  Let $n, k \geq 1$, and $F:\R^n\ra \R^k$, we define $\norm{F}_\infty=\sup_{t\in\R^n}\norm{F(t)}$. 
We indicate with $\Gamma=\{i_1,\ldots,i_u\} \in \R^{1+d}$ for $u \in \N$, a sequence of elements in $\R^{1+d}$. We then define the random vector $\bb{Z}_\Gamma=(\bb{Z}_{i_1},\ldots,\bb{Z}_{i_u})$. In general, we use bold notations when referring to random elements.

In the following, Lipschitz continuous is understood to mean globally Lipschitz. For $u\in\N$, $\mathcal{G}_u^*$ is the class of bounded functions from $\R^u$ to $\R$ and $\mathcal{G}_u$ is the class of bounded, Lipschitz continuous functions from $\R^u$ to $\R$ with respect to the distance $\|\cdot\|_1$ and define the Lipschitz constant as
\begin{gather}
	\label{Lipcost}
	Lip(h)=\sup_{x\neq y}\frac{|h(x)-h(y)|}{\|x-y \|_1}.
\end{gather} 

Hereafter, we often use the lexicographic order on $\R^{1+d}$. Let $t$ and $s$ be indicating a temporal and spatial coordinate. For distinct elements $y=( y_{1,t},y_{1,s},\ldots,y_{d,s}) \in  \R^{1+d}$ and $z=(z_{1,t},z_{1,s},\ldots,z_{d,s})\in \R^{1+d}$ we say $y<_{lex}z$ if and only if $y_{1,t}<z_{1,t}$ or $y_{p,s} <z_{p,s}$ for some $p\in\{1,\ldots,d\}$ and $y_{1,t}=z_{1,t}$ and $y_{q,s}=z_{q,s}$ for $q=1,\ldots,p-1$. Moreover, $y\leq_{lex}z$ if $y<_{lex}z$ or $y=z$ holds. Finally, let $z \in \R^{1+d}$, we define the set $V_z=\{y \in \R^{1+d}: y\leq_{lex}z\}$ and  $V_z^r=V_z\cap \{y\in \R^{1+d}: \norm{z-y}_\infty\geq r \}$ for $r>0$. The definition of the set $V_z^r$ is also used when referring to the lexicographic order on $\Z^{1+d}$. 

\subsection{Definition and properties of MMAF}
\label{sec1.2}

Let $I=H \times \R \times \R^d  $, where $H \subset \R^q$ for $q \geq 1$, and the Borel $\sigma$-algebra of $I$ be denoted  by $\mathcal{B}(I)$ and let $\mathcal{B}_b(I)$ contain all its Lebesgue bounded sets.

\begin{Definition}\label{def:levybasis}
	A family of $\R$-valued random variables $\Lambda=\{\Lambda(B):B \in \mathcal{B}_b(I)\}$ is called a L\'evy basis on $(I,\mathcal{B}_b(I))$ if it is an independently scattered and infinitely divisible random measure. This means that:
	\begin{itemize}
		\item[(i)]  For a sequence of pairwise disjoint elements of $\mathcal{B}_b(I)$, say $\{B_i, i \in \N \}:$
		\begin{itemize}
			\item $\Lambda (\bigcup_{i\in\N} B_i)= \sum_{i\in\N}\Lambda(B_i)$ almost surely when $\bigcup_{i\in\N} B_i \in \mathcal{B}_b(I)$
			\item and $\Lambda(B_i)$ and $\Lambda(B_j)$ are independent for $i\neq j$.
		\end{itemize}
		\item[(ii)] Let $B \in \mathcal{B}_b(I)$. Then, the random variable $\Lambda(B)$ is infinitely divisible, i.e., for any $i \in \N$, there exists a law $\mu_i$ such that the law $\mu_{\Lambda(B)}$ can be expressed as $\mu_{\Lambda(B)}=\mu_i^{*i}$, the $i$-fold convolution of $\mu_i$ with itself.
	\end{itemize}
\end{Definition}

\noindent For more details on infinitely divisible distributions, we refer the reader to \cite{S}. In the following, we will restrict ourselves to L\'evy bases which are homogeneous in space and time and factorizable, i.e., L\'evy bases with characteristic function
\begin{equation}\label{equation:fact}
	\E\left[e^{\text{i} u \Lambda(B)} \right]=e^{\Phi(u)\Pi(B)}
\end{equation}
for all $u\in \R$ and $B\in\mathcal{B}_b(I)$, where $\Pi=\pi\times\lambda_{1+d}$ is the product measure of the probability measure $\pi$ on $H$ and the Lebesgue measure $\lambda_{1+d}$ on $\R \times \R^d$.  Note that when using a L\'evy basis defined on $I=\R \times \R^d$, $\Pi= \lambda_{1+d}$. Furthermore, 
\begin{align}\label{equation:phi}
	\Phi(u)=\text{i}\gamma \,u -\frac{1}{2}  \sigma^2 u^2 +\int_{\R} \left(e^{\text{i} u x}-1-\text{i} u x \bb{1}_{[0,1]}(|x|)\right)\nu(dx)
\end{align}
is the cumulant transform of an infinitely divisible distribution with characteristic triplet $(\gamma,\sigma^2,\nu)$, where $\gamma \in \R$, $\sigma^2 \geq 0$ and $\nu$ is a L\'evy-measure on $\R$, i.e.,
\begin{align*}	
	\nu(\{0\})=0 \quad \text{and} 
	\int_{\R}\left(1\wedge x^2\right)\nu(dx)<\infty.
\end{align*}

The quadruplet $(\gamma, \sigma^2,\nu,\pi)$ determines the distribution of the L\'evy basis, and therefore it is called its \emph{characteristic quadruplet}.
An important random variable associated with the L\'evy basis, is the so-called \emph{L\'evy seed}, which we define as the random variable $\Lambda^{\prime}$ having as cumulant transform (\ref{equation:phi}), that is 
\begin{equation}
	\label{equation:fact2}
	\E\left[e^{\text{i} u\Lambda^{\prime}}\right]=e^{\Phi(u)}.
\end{equation}

By selecting different L\'evy seeds, it is easy to compute the distribution of $\Lambda(B)$ for $B \in \mathcal{B}_b(I)$, for example, when $I= \R \times \R^d$. In the following two examples, we compute the L\'evy bases used in generating the data sets in Section \ref{sec5.1}. 
\begin{Example}[Gaussian L\'evy basis]
	\label{gau}
	Let $\Lambda^ {\prime} \sim \mathcal{N}(\gamma,\sigma^2)$, then its characteristic function is equal to $\exp(\text{i}\gamma u-\frac{1}{2}\sigma^2 u^2) $. Because of  (\ref{equation:fact}), we have, in turn, that the characteristic function of $\Lambda(B)$ is equal to 
	$\exp(\text{i}\gamma u \lambda_{1+d}(B)-\frac{1}{2}\sigma^2 \lambda_{1+d}(B) u^2) $. In conclusion,
	$\Lambda(B)\sim \mathcal{N}(\gamma \lambda_{1+d}(B), \sigma^2 \lambda_{1+d}(B))$ for any $B \in \mathcal{B}_b(I)$.  
\end{Example}

\begin{Example}[Normal Inverse Gaussian L\'evy basis]
	\label{nig}
	Let $K_1$ denote the modified Bessel function of the third order and index $1$. Then, for $x \in \R$, the NIG distribution is defined as
	\[
	f(x:\alpha,\beta,\mu,\delta)=\alpha \delta (\pi^2(\delta^2+(x-\mu)^2))^{-\frac{1}{2}} \exp(\delta \sqrt{\alpha^2-\beta^2}+\beta(x-\mu)) K_1 (\alpha\sqrt{\delta^2+(x-\mu)^2}),
	\] 
	where $\alpha, \beta,\mu$ and $\delta$ are parameters such that $\mu \in \R$, $\delta >0$ and $0\leq |\beta|<\alpha.$
	Let $\Lambda^{\prime}\sim NIG(\alpha,\beta,\mu,\delta)$, then by (\ref{equation:fact}) we have that $\Lambda(B) \sim NIG ( \alpha, \beta, \mu \lambda_{1+d}(B), \delta\lambda_{1+d}(B))$ for all $B \in \mathcal{B}_b(I)$.
\end{Example}

We now follow \cite{Ambit}, and \cite{BNS04} to formally define ambit sets. 

\begin{Definition}
	A family of ambit sets $(A_t(x))_{(t,x) \in \R \times \R^d}\subset \R\times \R^d$
	satisfies the following properties:
	\begin{gather}\label{equation:ambitset}
		\begin{cases}
			A_t(x)=A_0(0)+(t,x), \text{ (Translation invariant)}\\
			A_s(x)\subset A_t(x), \,\, \text{for $s<t$}\\
			A_t(x)\cap (t,\infty)\times\R^d=\emptyset \text{ (Non-anticipative)}.
		\end{cases}
	\end{gather} 
\end{Definition} 
We further assume that the random fields $\bb{Z}:=(\bb{Z}_t(x))_{(t,x)\in \R \times \R^d}$ in the paper are \emph{influenced}. By this name we mean random fields defined on a given complete probability space $(\Omega, \mathcal{F}^{\prime},P)$, equipped with the \emph{filtration of influence} (in the sense of Definition 3.8 in \cite{CSS20}) $\mathbb{F}=(\mathcal{F}_{(t,x)})_{ (t,x) \in \R \times \R^d}$ generated by $\Lambda$ and the family of ambit sets $(A_t(x))_{(t,x) \in \R \times \R^d}\subset \R\times \R^d$, i.e., each $\mathcal{F}_{(t,x)}$ is the $\sigma$-algebra generated by the set of random variables $\{ \Lambda(B) : B \in \mathcal{B}_b(H\times A_t(x))\}$, which are adapted to $\mathbb{F}$. We call our field \emph{adapted} to the filtration of influence $\mathbb{F}$ if it is measurable with respect to the $\sigma$-algebra $\mathbb{F}$ for each $(t,x) \in \R \times \R^d$. 

Moreover, we work with spatio-temporal stationary random fields in the following. 

\begin{Definition}[Spatio-temporal stationarity]
	\label{statio}
	We say that $\bb{Z}$ is spatio-temporal stationary if for every $n \in \N$, $\tau \in \R$, $u \in \R^d$, $t_1,\ldots,t_n \in \R$ and $x_1,\ldots,x_n \in \R^d$, the joint distribution of $(\bb{Z}_{t_1}(x_1),\ldots,\bb{Z}_{t_n}(x_n))$ is the same as that of $(\bb{Z}_{t_1+\tau}(x_1+u),\ldots,\bb{Z}_{t_n+\tau}(x_n+u))$. 
\end{Definition}

We use simply the term \emph{stationary} throughout when referring to processes satisfying Definition \ref{statio}.
We can now formally define the stochastic model underlying our learning methodology.

\begin{Definition}[MMAF]
	\label{mmaf}
	Let $\Lambda=\{\Lambda(B), B\in \mathcal{B}_b(I)\}$ a L\'evy basis, $f:H\times\R \times \R^d \rightarrow \R$ a $\mathcal{B}(I)$-measurable function and $A_t(x)$ an ambit set. Then, the stochastic integral
	\begin{equation}
		\label{mmaf_i}
		\textbf{Z}_t(x)=\int_{H} \int_{A_t(x)} \, f(A,x-\xi,t-s) \,\,  \Lambda(dA,d\xi,ds),~ (t,x)\in \mathbb{R} \times \mathbb{R}^{d},
	\end{equation}
	is adapted to the filtration $\mathbb{F}$, stationary, and its distribution is infinitely divisible. We call the  $\R$-valued random field $\bb{Z}$ an (influenced) mixed moving average field and $f$ its kernel function. 
\end{Definition}

\begin{Remark}
	On a technical level, we assume all stochastic integrals in this paper to be well defined in the sense of Rajput and Rosinski \cite{RR1989}. For more details, including sufficient conditions on the existence of the integral as well as the explicit representation of the characteristic triplet of the MMAF's infinitely divisible distribution (which can be directly determined from the characteristic quadruplet of $\Lambda$), we refer to \cite[Section 3.1]{CSS20}. In the latter, there can also be found a multivariate definition of a L\'evy basis and an MMAF.
\end{Remark}

Important examples of MMAFs are the spatio-temporal Ornstein-Uhlenbeck field (STOU) and the mixed spatio-temporal Ornstein-Uhlenbeck field (MSTOU), whose properties have been thoroughly analyzed in \cite{STOU} and \cite{MSTOU}. There are also interesting time series models in the MMAF framework, which we present in section \ref{timeseries_app}.

\begin{Example}[STOU process]
	\label{prop_stou}
	Let $\Lambda=\{\Lambda(B), B\in \mathcal{B}_b(I)\}$ be a L\'evy basis, $f:\R \times \R^d \rightarrow \R$ a $\mathcal{B}(I)$-measurable function defined as $f(s,\xi)=\exp(-A s)$ for $A>0$, and $A_t(x)$ be defined as in (\ref{lightcone}). Then, the STOU is defined as 
	\begin{equation}
		\label{stou}
		\bb{Z}_t(x):=\int_{A_t(x)} \exp(-A (t-s)) \Lambda(ds,d\xi).
	\end{equation}
	The STOU is a stationary and Markovian random field. Moreover, an STOU exhibits exponential temporal autocorrelation (just like the temporal Ornstein-Uhlenbeck process) and has a spatial autocorrelation structure determined by the shape of the ambit set. In addition, this class of fields admits non-separable autocovariances, which are desirable in practice, see Example \ref{corr_stou1}. 
\end{Example}

\begin{Example}[MSTOU process]
	\label{prop_mstou}
	An MSTOU process is defined by \emph{mixing} the parameter $A$ in the definition of an STOU process; that is, we assume that $\pi$ has support in $H=(0,\infty)$. This modification allows the determination of random fields with power-decaying autocovariance functions, see example \ref{ex_mstou}. Let $\Lambda=\{\Lambda(B), B\in \mathcal{B}_b(I)\}$ be a L\'evy basis, $f:(0,\infty) \times\R \times \R^d \rightarrow \R$ a $\mathcal{B}(I)$-measurable function defined as $f(A,s,\xi)=\exp(-A s)$, and $A_t(x)$ be defined as in (\ref{lightcone}). Moreover, let $l(A)$ be the density of $\pi$ with respect to the Lebesgue measure such that
	\[
	\int_{0}^{\infty} \frac{1}{A^{d+1}} l(A) dA < \infty.
	\]
	Then, the MSTOU is defined as 
	\begin{equation}
		\label{mstou}
		\bb{Z}_t(x):=\int_0^{\infty}\int_{A_t(x)} \exp(-A (t-s)) \Lambda(dA, ds,d\xi).
	\end{equation}
\end{Example}

The MMAF framework has a \emph{causal interpretation} under the following assumption.
\begin{Assumption}
	\label{ass_light}
	For a $c>0$, we consider
	\begin{equation}
		\label{lightcone}
		A_t(x) := \big\{(s,\xi)\in\mathbb{R}\times\mathbb{R}^d: s\leq t \,\, \textit{and}\,\, \|x-\xi\| \leq c |t-s|  \big\}.
	\end{equation}
\end{Assumption}	
To explain why using cone-shaped ambit sets allows to have such interpretation, we borrow the concept of \emph{lightcone} from special relativity. 

A lightcone  describes the possible paths that the light can make in space-time leading to a space-time point $(t,x)$ and the ones that lie in its future. In the context of our paper, we use their geometry to identify the space-time points having a causal relationship. For a point $(t,x)$, $c >0$ and by using the Euclidean norm to assess the distance between different space-time points, we define a \emph{lightcone} as the set
\begin{equation*}
	\mathcal{A}^{light}_t(x)=\big\{(s,\xi)\in\mathbb{R}\times\mathbb{R}^d: \|x-\xi\| \leq c |t-s| \big\}.
\end{equation*}
The set $\mathcal{A}^{light}_t(x)$ can be split into two disjoint sets, namely, $A_t(x)$ and $A_t(x)^+$. The set $A_t(x)$ is called \emph{past lightcone}, and its definition corresponds to the one of  a cone-shaped ambit set (\ref{lightcone}).

The set
\begin{equation}
	\label{future_lightcone}
	A_t(x)^+=\{(s,\xi)\in \R\times \R^d: s >t \, \textrm{and} \, \,  \|x-\xi\| \leq c |t-s| \},
\end{equation}
is called instead the \emph{future lightcone}. By using an influenced MMAF on a cone-shaped ambit set as the underlying model, we implicitly assume that the following sets
\begin{equation}
	\label{lightcone2}
	l^{-}(t,x)=\{\bb{Z}_s(\xi): (s,\xi) \in A_t(x) \setminus (t,x) \} \,\,\, \textrm{and} \,\,\, l^{+}(t,x)=\{\bb{Z}_s(\xi): (s,\xi) \in A_t(x)^+ \}
\end{equation}
are respectively describing the values of the field that have a direct influence on the determination of $\bb{Z}_t(x)$ and the future field values influenced by $\bb{Z}_t(x)$. We can then uncover the causal relationship between space-time points described above by estimating the constant $c$ from observed data, which we call \emph{the speed of information propagation} in the physical system under analysis. A similar approach to the modeling of causal relationships can be found in several machine learning frameworks, as in \cite{MS16}, \cite{Disco}, and \cite{AImink}. In \cite{MS16} and \cite{Disco}, the sets (\ref{lightcone2}) are considered and employed to discover coherent structures, as defined in \cite{Coherent}, in spatio-temporal physical systems and to perform video frame prediction, respectively. In \cite{AImink}, forecasts are performed by embedding spatio-temporal information on a Minkowski space-time. Hence, the concept of lightcones enters into play in the definition of their algorithm. In statistical modeling, we typically have two equivalent approaches towards causality: structural causal models, which rely on the use of directed acyclical graphs (DAG) \cite{Pearl}, and Rubin causal models, which rely upon the potential outcomes framework \cite{Rubin}. The concept of causality employed in this paper can be inscribed into the latter. In fact, by using MMAFs on cone-shaped ambit sets, the set $l^{+}(t,x)$ describes the possible future outcomes that can be observed starting from the spatial position $(t,x)$. 

Finally, we consider the following definitions of temporal and spatial short and long-range dependence in the paper.

\begin{Definition}[Short and long range dependence]
	A random field $(\bb{Z}_t(x))_{(t,x)\in \R\times \R^d}$ is said to have \emph{temporal short-range dependence} if 
	\[
	\int_{0}^{\infty} Cov(\bb{Z}_t(x),\bb{Z}_{t+\tau}(x)) \,\, d\tau < \infty,
	\]
	and \emph{temporal long-range dependence} if the integral above is infinite.
	Similarly, an isotropic random field, see Definition \ref{isotropy}, has \emph{spatial short-range dependence} if 
	\[
	\int_{0}^{\infty} C(r) \,\, dr <\infty,
	\]
	where $Cov(\bb{Z}_t(x),\bb{Z}_t(x+u))=C(|u|)$ and $r=|u|$. It is said to have \emph{spatial long-range dependence} if the integral is infinite.
\end{Definition}

Under Assumption \ref{ass_light}, an STOU process admits temporal and spatial short-range dependence, whereas an MSTOU process can admit temporal and spatial short and long-range dependence by carefully modeling the parameter $A$.

\begin{Example}
	\label{memory}
	Let $\bb{Z}$ be an MSTOU process as defined in Example \ref{prop_mstou} for $d=1$, Assumption \ref{ass_light} hold, and $l(A)=\frac{\beta^{\alpha}}{\Gamma(\alpha)} A^{\alpha-1} \exp(-\beta A)$ be the Gamma density with shape and rate parameters $\alpha >d+1$ and $\beta>0$. From the calculations in Example \ref{ex_mstou} and by setting $u=0$, then $\bb{Z}$ has temporal short-range dependence because for $\alpha >3$, 
	\begin{align*}
		\int_0^{\infty} Cov(\bb{Z}_t(x),\bb{Z}_{t+\tau}(x)) \, d\tau&= \frac{c\beta^{\alpha}Var(\Lambda^{\prime})}{2(\alpha-2)(\alpha-1)} \int_0^{\infty} (\beta+\tau)^{-(\alpha-2)} \, d\tau\\
		&= \frac{c\beta^3 Var(\Lambda^{\prime})}{2(\alpha-1)(\alpha-2)(\alpha-3)}.
	\end{align*}
	is finite.
	This integral is infinite for $2<\alpha\leq 3$, and we then say that the MSTOU process has temporal long-range dependence. We obtain spatial short or long-range dependence for the same choice of parameters. In fact, for $r=|u|$ and $\tau=0$, and $\alpha>3$
	\begin{align*}
		\int_0^{\infty} C(r) \, dr &=\frac{c\beta^{\alpha}Var(\Lambda^{\prime})}{2(\alpha-2)(\alpha-1)} \int_0^{\infty} (\beta+r/c)^{-(\alpha-2)} \, dr\\
		&= \frac{c\beta^3 Var(\Lambda^{\prime})}{2(\alpha-1)(\alpha-2)(\alpha-3)},
	\end{align*}
	converges, whereas the integral diverges for $2<\alpha\leq 3$.
\end{Example}

\subsection{Weak dependence coefficients in MMAF-guided learning}
\label{sec2.3}

We start by giving the definitions of the dependence notions involved in our learning methodology.

\begin{Definition}\label{thetaweaklydependent}
	Let $\bb{Z}$ be an $\R$-valued random field. Then, $\bb{Z}$ is called $\theta$-lex-weakly dependent if
	\begin{gather*}
		\theta_{lex}(r)=\sup_{u,v\in\N}\theta_{u,v}(r) \underset{r\ra\infty}{\longrightarrow} 0,
	\end{gather*} 
	where
	\begin{equation*}
		\theta_{u,v}(r)=\sup\bigg\{\frac{|Cov(F(\bb{Z}_{\Gamma}),G(\bb{Z}_{\Gamma^{\prime}}))|}{\norm{F}_{\infty}v Lip(G)},\, F\in\mathcal{G}^*_u,\, G\in\mathcal{G}_v,\, \Gamma,\, \Gamma^{\prime},\, |\Gamma|= u,\, |\Gamma^{\prime}|=v \bigg\}
	\end{equation*}
	for $\Gamma=\{t_{i_1},\ldots,t_{i_u}\} \in \R^{1+d}$, $\Gamma^{\prime}=\{t_{j_1},\ldots,t_{j_{v}}\} \in \R^{1+d}$ such that $\Gamma \in V_{\Gamma^{\prime}}^{r}= \bigcap_{l=1}^v V_{t_{j_l}}^r$ for $t_{j_l} \in \Gamma^{\prime}$. We call $(\theta_{lex}(r))_{r\in\R^+}$ the $\theta$-lex-coefficients. 
\end{Definition}

MMAFs are $\theta$-lex weakly dependent random fields, as proven in Proposition \ref{proposition:mmathetaweaklydep}. Moreover, Definition \ref{thetaweaklydependent} is an extension to the random field case of a dependence notion developed for \emph{causal processes} called $\theta$-weak dependence. 

\begin{Definition}\label{thetadependent}
	Let $\bb{Z}$ be an $\R$-valued stochastic process. Then, $\bb{Z}$ is called $\theta$-weakly dependent if
	\begin{gather*}
		\theta(k)=\sup_{u\in\N}\theta_{u}(k) \underset{k\ra\infty}{\longrightarrow} 0,
	\end{gather*} 
	where
	\begin{equation*}
		\theta_{u}(k) =\sup\bigg\{\frac{|Cov(F(\bb{Z}_{\Gamma} ),G(\bb{Z}_{j_1}))|}{\norm{F}_{\infty}Lip(G)},\, F\in\mathcal{G}_u^*, \, G\in\mathcal{G}_1, \, \Gamma, |\Gamma|=u \bigg\}.
	\end{equation*}
	for $\Gamma=\{t_{i_1},\ldots,t_{i_u}\} \in \R$ such that $i_1 \leq i_2 \leq \ldots \leq i_u \leq i_u +k \leq j_1$. We call $(\theta(k))_{k\in\R^+}$ the $\theta$-coefficients. 
\end{Definition}

We use extensively in the proofs of Section \ref{sec4}, the following result.

\begin{Remark}[Projective-type representation of $\theta$-weak dependence] 
	\label{mix_rule1}
	Let $(X_t)_{t \in \Z}$ be a real-valued $\theta$-lex weakly dependent  process, $\mathcal{L}_1=\{g:\R\to\R, \, g \in \mathcal{G}_1,\,\, Lip(g)\leq 1 \}$, $s \in \N$, $j_1 \in \Z$, and $\mathcal{M}=\sigma\{ \bb{X}_t: t \leq j_1 \,\, \text{and} \,\, |t-j_1|\geq s \}$, then it is showed in \cite[Proposition 2.3]{D07} that
	\begin{equation}
		\label{mix2}
		\theta(s)=\sup_{j_1 \in \Z} \sup_{g \in \mathcal{L}_1} \|\E[g(\bb{X}_{j_1})|\mathcal{M}]-\E[g(\bb{X}_{j_1})]\|_1.
	\end{equation}
	An alternative proof of this result can be also found in \cite[Lemma 5.1]{CSS20}.
\end{Remark}

\begin{Remark}[About the parameter $k$ and $r$ in Definition \ref{thetadependent} and \ref{thetaweaklydependent}]
	\label{aboutk}
	Let us start by assuming that $\bb{Z}=(\bb{Z}_i)_{i \in \R}$ is a sequence of independent and identically distributed (in short, i.i.d.) random variables, then for any $F \in \mathcal{G}_u^*$ with $u \in \N$, $G \in \mathcal{G}_1$, and selecting a set $\Gamma=\{i_1,\ldots,i_u\} \in \R$ such that $i_1\leq \ldots\leq i_u\leq i_u+k\leq j$, we have that
	\[
	Cov(F(\bb{Z}_{\Gamma}), G(\bb{Z}_{j}))=0,
	\]
	and $\theta(k)=0$ for all $k \in \R^+$. The process $\bb{Z}$ is $\theta$-weakly dependent and the parameter $k$ is encoding the distance between the marginals
	\[
	\bb{P}:=(\bb{Z}_{i_1},\ldots, \bb{Z}_{i_u}),\,\,\, \textrm{and} \,\,\,
	\bb{F}:=(\bb{Z}_{j}).
	\]
	In terms of the process $\bb{Z}$, the $\sigma$-algebras generated by $\bb{P}$ and $\bb{F}$ represent \emph{past} and \emph{future} events. Obviously, in the case of a sequence of independent random variables, the past plays no role in the unfolding of the future. However, if we consider a $\theta$-weakly dependent process $\bb{Z}$, then it satisfies the inequality
	\[
	Cov(F(\bb{Z}_{i_1},\ldots,\bb{Z}_{i_u}), G(\bb{Z}_j))\leq 2 \|F\|_{\infty} Lip(G) \theta(k).
	\]
	The result above has been proven, for example, in \cite{CS19} and gives a criteria to \emph{measure explicitly} the dependence between past and future. Here, the past is progressively forgotten for $k \to \infty$.
	The parameter $k$ expresses the distance at which we are evaluating the influence of the past on how the future unfolds, and the coefficients $\theta(k)$ is a measure of how fast the past is forgotten.
	
	Similar considerations can be done in the case in which we consider $\bb{Z}=(\bb{Z}_i)_{i \in \R^{1+d}}$ a $\theta$-lex weakly dependent random field for the constant $r \in \R^+$. However, $\bb{P}$ and $\bb{F}$ represent the marginals of lexicographically ordered elements in this case. The lexicographic order in $\R^{1+d}$ substitutes the natural temporal order for stochastic processes defined on $\R$.
	
\end{Remark}

In the MMAF modeling framework, we can show general formulas for the computation of upper bounds of the $\theta$-lex coefficients. The latter is given as a function of the characteristic quadruplet of the driving L\'evy basis $\Lambda$ and the kernel function $f$ in (\ref{mmaf_i}), see Proposition \ref{proposition:mmathetaweaklydep}.

For $d=1,2$, when the MMAF has a kernel function with no spatial component, we can compute a bound for the $\theta$-lex coefficients expressed in terms of the covariances of the field $\bb{Z}$, which can be computed as shown in Section \ref{cov}. These bounds have an expression that allows us to use standard statistical inference tools to infer the decay rate parameter of the coefficients; see, for example, the estimators (\ref{plugin}) and (\ref{plugin2}). In the general framework described in Proposition \ref{proposition:mmathetaweaklydep}, however, similar estimation methodologies are not yet available and remain an interesting open problem.

\begin{Proposition}
	\label{example:spatiotemporaldata}
	Let $\Lambda$ be an $\R$-valued L\'evy basis with characteristic quadruplet $(\gamma, \sigma^2,\nu,\pi)$ and $f:H\times\R\rightarrow\R$ a $\mathcal{B}(H\times\R)$-measurable function not depending on the spatial dimension, i.e.,
	\begin{align}\label{eq:MMAFnospatialdependence}
		\bb{Z}_t(x)=\int_H\int_{A_t(x)} f(A,t-s) \Lambda(dA,ds,d\xi),\quad (t,x)\in\R^{1+d}.
	\end{align}
	\begin{enumerate}
		\item[(i)] For $d=1$, if $\int_{|x|>1} x^2 \nu(dx)<\infty$ and $\gamma+\int_{|x|>1}x\nu(dx)=0$, then $\bb{Z}$ is $\theta$-lex weakly dependent and 
		\begin{align*}
			\theta_{lex}(r) \!&\leq\! 2 \left(\!2 Var(\Lambda^{\prime}) \int_0^\infty \int_{A_0(0)\cap A_0(r\min(2,c))}\!\!f(A,-s)^2 dsd\xi\pi(dA) \right)^{1/2}\\
			&=\! 2 \sqrt{2Cov(\bb{Z}_0(0),\bb{Z}_0(r\min(2,c)))},
		\end{align*}
		where $Var(\Lambda^{\prime}) = \sigma^2 + \int_{\R} x^2 \,\nu(dx)$.
		\item[(ii)] For $d=2$, if $\int_{|x|>1} x^2 \nu(dx)<\infty$ and $\gamma+\int_{|x|>1}x\nu(dx)=0$, then $\bb{Z}$ is $\theta$-lex weakly dependent and 
		\begin{align*}
			\theta_{lex}(r) &\leq 2 \Bigg( 2Cov\left(\bb{Z}_0(0,0),\bb{Z}_0\left(r\min\left(1,\frac{c}{\sqrt{2}}\right),r\min\left(1,\frac{c}{\sqrt{2}}\right)\right)\right)\\
			&\qquad +2Cov\left(\bb{Z}_0(0,0),\bb{Z}_0\left(r\min\left(1,\frac{c}{\sqrt{2}}\right),-r\min\left(1,\frac{c}{\sqrt{2}}\right)\right)\right)\Bigg)^{1/2}.
		\end{align*}
	\end{enumerate}
\end{Proposition}

The proof of the results above is given in Appendix \ref{appb1}. 

\begin{Notation}
	\label{tildetheta}
	In general, we indicate the bounds of the $\theta$-lex coefficients determined in Proposition \ref{proposition:mmathetaweaklydep}, Corollary \ref{gen_coeff} or Proposition \ref{example:spatiotemporaldata} using the sequence $(\tilde{\theta}_{lex}(r))_{r \in \R^+}$ where
	\[
	\theta_{lex}(r)\leq 2\tilde{\theta}_{lex}(r).
	\]
\end{Notation}

\begin{Definition}
	\label{decay}
	For $r \in \R^+$, $\bar{\alpha} >0$,  if $ \tilde{\theta}_{lex}(r) \leq \bar{\alpha} \, \exp(-\lambda r)$ we say that $\bb{Z}$ admits exponentially decaying $\theta$-lex coefficients, whereas if $ \tilde{\theta}_{lex}(r) \leq \bar{\alpha} r^{-\lambda}$, we say that $\bb{Z}$ admits power decaying $\theta$-lex coefficients.
\end{Definition}	

We give below examples of MMAF with exponentially and power-decaying $\theta$-lex coefficients.
\begin{Example}
	\label{coeff_stou1}
	Let $d=1$ and $\bb{Z}$ be an STOU as in Definition \ref{prop_stou}. If $\int_{|x| >1} x^2 \, \nu(dx) < \infty$, $\gamma+\int_{|x| >1} x \, \nu(dx)=0$, then $\bb{Z}$ is $\theta$-lex weakly dependent with 
	\begin{align*}
		\tilde{\theta}_{lex}(r)&= \Big( Var(\Lambda^{\prime}) \int_{A_0(0) \cap (A_0(\psi) \cup A_0(-\psi))} \exp(2As) \, ds\, d\xi \Big)^{\frac{1}{2}} \\
		&\leq \Big( 2 Var(\Lambda^{\prime}) \int_{A_0(0) \cap A_0(\psi)}  \exp(2As) \, ds\, d\xi \Big)^{\frac{1}{2}}\\
		&=\Big( 2 Var(\Lambda^{\prime}) \int_{-\infty}^{-\frac{\psi}{2c}} \int_{\psi+cs}^{-cs} \exp(2As) \, ds\, d\xi \Big)^{\frac{1}{2}}=  \Big( \frac{c}{A^2} Var(\Lambda^{\prime}) \, \exp\Big( \frac{-A \psi}{c} \Big) \Big)^{\frac{1}{2}} \\
		&= \Big( \,\frac{c}{A^2} Var(\Lambda^{\prime})  \,\exp\Big( -\,\underset{2\lambda}{\underbrace{\frac{A \min(2,c)}{c}}}\,r \Big) \Big)^{\frac{1}{2}}\\
		&= \sqrt{2Cov(\bb{Z}_0(0),\bb{Z}_0(r\min(2,c)))}:= \bar{\alpha} \exp(-\lambda r),
	\end{align*}
	
	where $\lambda >0$ and $\bar{\alpha}>0$. Because the temporal and spatial autocovariance functions of an STOU are exponential, see (\ref{cov_stou}), the model admits temporal and spatial short-range dependence.

	
	By estimating the parameter vector $\theta_0=\{A,c, Var(\Lambda^{\prime})\}$ using the methodologies revised in Appendix \ref{sec_stou}, we can estimate the parameter $\lambda$  using the following plug-in estimator
	\begin{equation}
		\label{plugin}
		\bb{\lambda^{*}}=\frac{\min(2,\bb{c^*})}{2\bb{c^*}},
	\end{equation}
	where the estimators $\bb{A^*}$ and  $\bb{c^*}$ are defined in (\ref{est_stou}).
	This estimator is consistent because of \cite[Theorem 12]{STOU} and the continuous mapping theorem. Furthermore, by using an estimator of the parameter $Var(\Lambda^{\prime})$, we can also obtain a consistent estimator for the parameter $\bar{\alpha}$.
	
\end{Example}

\begin{Example}
	\label{coeff_mstou1}
	Let $d=1$ and $\bb{Z}$ be an MSTOU as defined in Example \ref{memory}. 
	If $\int_{|x| >1} x^2 \, \nu(dx) < \infty$, $\gamma+\int_{|x| >1} x \, \nu(dx)=0$, then $\bb{Z}$ is $\theta$-lex weakly dependent with 
	\begin{align*}
		\tilde{\theta}_{lex}(r)&\leq  \Big( \frac{c}{A^2} Var(\Lambda^{\prime}) \, \int_0^{\infty} \exp\Big( \frac{-A \psi}{c} \Big) \pi(dA) \Big)^{\frac{1}{2}}\\
		&=  \Big( \frac{Var(\Lambda^{\prime}) c \beta^{\alpha}}{(\beta+\psi/c)^{\alpha-2} (\alpha-2)(\alpha-1)} \Big)^{\frac{1}{2}}\\
		&=  \Big( \frac{Var(\Lambda^{\prime}) c \beta^{\alpha}}{(\alpha-2)(\alpha-1)}  \Big(\beta+\frac{r\min(2,c)}{c} \Big)^{-(\alpha-2)}  \Big)^{\frac{1}{2}}\\
		&=  \sqrt{2Cov(\bb{Z}_0(0),\bb{Z}_0(r\min(2,c)))} :=  \bar{\alpha}  r^{-\lambda},
	\end{align*}
	where $\lambda=\frac{\alpha-2}{2}$ and $\bar{\alpha}>0$. As already addressed in Example \ref{memory}, for $2 < \alpha \leq 3$, that is $0<\lambda \leq \frac{1}{2}$, the model admits temporal and spatial long-range dependence. Instead, for $\alpha >3$, that is $\lambda > \frac{1}{2}$, the model admits temporal and spatial short range dependence.


	For the model used in Example \ref{coeff_mstou1}, an estimator for 
	\begin{equation}
		\label{plugin2}
		\bb{\lambda}^*=\frac{\bb{\alpha}^*-2}{2}
	\end{equation}
	where the vector of parameters $\theta_1=\{\alpha,\beta,c,Var(\Lambda^{\prime})\}$ is estimated using a GMM estimator $\bb{\theta}_1^*=\{\bb{\alpha}^*,\bb{\beta}^*,\bb{c}^*,\bb{Var(\Lambda^{\prime})}^*\}$.
	
\end{Example}

The estimator of $\lambda$ depends on the chosen data generating process and the spatial dimension of the data. For further details on the parametric estimation methodologies nowadays available for other fields belonging to the MMAF class, we refer the reader to the Appendices \ref{sec_stou}, \ref{sec_mstou} and \ref{timeseries_app}.

\section{Mixed moving average field guided learning}
\label{sec4}

\subsection{Pre-processing $N$ frames}
\label{sec4.2}

\begin{figure}[h!]
	\centering \scalebox{1.0}
	{\includegraphics[width=5.8cm]{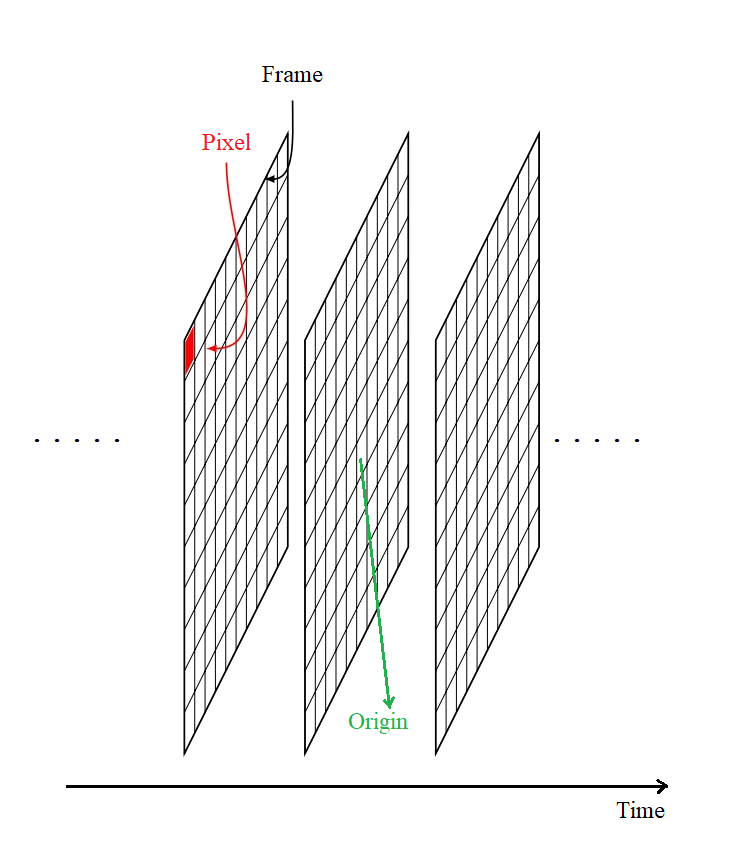}}
	\caption{Raster data cube's spatio-temporal index set with origin in $(t_0,x_0)$.} \label{frame}
\end{figure}

In this section, we describe MMAF-guided learning for a spatial dimension $d=2$. Let  $(\tilde{Z}_t(x))_{(t,x)\in  \mathbb{T} \times \mathbb{L}}$ be an observed data set on a regular lattice $\mathbb{L} \subset \mathbb{R}^2$ across times $\mathbb{T}=\{t_0+h_t,\ldots,t_0+h_tN\}$ for $h_t \in \R$, such that 
\begin{equation}
	\label{decomposition}
	\tilde{Z}_t(x)= \mu_t(x)+ Z_t(x)
\end{equation}
holds and no measurement errors are present in the observations.
Here, $\mu_t(x)$ is a deterministic function, and $Z_t(x)$ are considered realizations from a zero mean stationary (influenced) MMAF. 

We represent graphically the regular spatial lattice $\mathbb{L}$ as a \emph{frame made of a finite amount of pixels}, i.e., squared-cells representing each of them a unique spatial position $x \in \mathbb{L}$, see Figure \ref{frame}. In several applications, such as satellite imagery, a pixel refers to a spatial cell of several square meters.
In the paper, we assume that a pixel represents the spatial point $x \in \R^2$ corresponding to the center of the pictured squared cell. We then use the name \emph{pixel} and spatial position throughout interchangeably. In total, $N$ frames represent the spatio-temporal index set of the observed data set. This terminology is often used to describe \textit{raster data cubes} \cite{Raster}. We define in Section \ref{st_embedding} a spatio-temporal embedding for data with such structure.
For dimension $d=1$, we consider that the pixel collapses in the point $x \in \R$ that describes, see Figure \ref{causal}. MMAF-guided learning also applies to spatial index set of dimension $d >2$. However, we do not represent the spatial positions using pixels in such cases.

We call $(t_0,x_0)$ the origin of the space-time grid, see Figure \ref{frame}, and $h_t$ and $h_s$ the time and space discretization step in the observed data set, i.e. the distance between two pixels along the temporal and spatial dimensions.

MMAF-guided learning has the target to determine one-time ahead ensemble forecasts of the field $\bb{Z}$ in a pixel $x^*$. We do not consider further the problem of estimating the deterministic function $\mu_t(x)$ when performing forecasting tasks, i.e., we assume our data set to be generated by a zero mean MMAF from now on. We refer the reader to \cite{CW11} for a review of how to estimate the function $\mu_t(x)$.

We need to define a \emph{training data set} to employ in our learning methodology in the following sections. Therefore, we pre-process the set of indices represented by the $N$ frames to select a set of different examples, i.e., input-output pairs, to include in the training data set $S_m$.

\subsubsection{Spatio-temporal embedding}
\label{st_embedding}
Let us consider a stationary random field $(\bb{Z}_t(x))_{(t,x)\in \Z \times \mathbb{L}}$, and select a pixel position $x^*$ in $\mathbb{L}$. We define the input-output vectors 
\begin{equation}
	\label{st_emb}
	\bb{X}_i=\bb{L^{-}_p}(t_0+ia,x^*), \, \,\,\, \textrm{and} \,\,\,\,\, \bb{Y}_i=\bb{Z}_{t_0+ia}(x^*), \,\,\,\textrm{for $i \in \Z$},
\end{equation}
where
\begin{equation}
	\label{feature}
	\bb{L^{-}_p}(t,x^*)=(\bb{Z}_{i_1}(\xi_{1}), \ldots, \bb{Z}_{i_{a(p,c)}}(\xi_{a(p,c)}))^{\top}, 
\end{equation}
with indices selected in the set
\begin{align}
	\label{index}
	\mathcal{I}(t,x^*):=\{ (i_s,\xi_s): \|x^*-\xi_s \| \leq c \, (t-i_s)\,\,&\textrm{for}\,\, 0< t-i_s\leq p, \nonumber \\
	&\textrm{and}\,\, (i_s,\xi_s)<_{lex} (i_{s+1},\xi_{s+1})  \},
\end{align} 
for $t=t_0+i a$ with $i \in \Z$, and  $c>0$. We call $a(p,c):=|\mathcal{I}(t,x^*)|$ and assume that is constant for all $t=t_0+ia$ and $i \in \Z$. The parameters $a >0$ and $p>0$ are multiples of $h_t$ such that $a=a_t h_t$, $p=p_th_t$, $a_t,p_t \in \N$ and $a_t \geq p_t+1$. We note that each element of $\bb{L^-_p}(t,x^*) \subset l^-(t,x^*)$ for $t=t_0+ia$ and $i \in \Z$, where $l^-(t,x^*)$ is defined in (\ref{lightcone2}) and identifies the set of all points in $\R \times \R^d$ that could possibly influence the realization $\bb{Z}_t(x^*)$. The sampling leading to (\ref{st_emb}) can be performed starting by a pixel $x^*$ for which the index set $\mathcal{I}(t_0+ia,x^*) \subseteq \mathbb{Z} \times \mathbb{L} $ for all $i$. Further details on why the examples need to be structured following the index sets $\mathcal{I}(t,x^*)$ are given in the next remark.

\begin{Remark}[Geometry and lexicographic order of the examples]
	\label{geolex}
	
	The sets $\mathcal{I}(t,x^*)$ are chosen with a geometry that is inherited by the definition of the cone-shaped ambit set in (\ref{lightcone}). Such a geometry allows us to give a causal interpretation of the one-time ahead ensemble forecast, as shown in Section \ref{sec5}. Moreover, we store in the input vectors $\bb{X}_i$ values of the fields $\bb{Z}$ with indices in lexicographic order. This choice implies that $((\bb{X}_{i_1},\bb{Y}_{i_1}),\ldots,(\bb{X}_{i_u},\bb{Y}_{i_u}))$ and $(\bb{X}_{j},\bb{Y}_{j})$ for $u \in \N$  and $j \in \Z$ are lexicographically ordered marginals of the field $\bb{Z}$. This allows, in turn, to precisely assess how the $\theta$-lex weakly dependence of the data generating process $\bb{Z}$ (defined w.r.t. lexicographically ordered marginals, see Definition \ref{thetaweaklydependent}) is inherited by the process $\bb{S}$ and to understand how the $\theta$-lex weakly dependence plays a role in the definition of the randomized estimators in Section \ref{sec4.4}. Modifications of this representation may be needed when using convolutional or transformer architectures as predictors. This latter issue is outside the scope of the present paper but constitutes an important future research direction of MMAF-guided learning.
	
	Last but not least, it is important to notice that the considerations above rule out the choice of \emph{overlapping} examples, i.e., the possible choice of examples that maintain the cone-shaped geometry but have indices not following the lexicographic order. If we were to make this choice, we would then work with a training data set that does not inherit the dependence structure of the data generating process $\bb{Z}$; see Section \ref{inheritance} for more details.
\end{Remark}

The sequence $((\bb{X}_i,\bb{Y}_i)^{\top})_{i \in \Z}$ is composed of identically distributed random vectors for all $i \in \Z$. We call $\bb{S}:=((\bb{X}_i,\bb{Y}_i)^{\top})_{i \in \Z}$ \emph{a cone-shaped sampling process}. An example of a realization $S$ of the sampling scheme, together with its related spatio-temporal embedding's description can be found in Figure \ref{cone}. The distribution of $\bb{S}$ is indicated throughout by $\mathbb{P}$. 

\begin{figure}[h!]
	\centering 
	\subfigure[]{\includegraphics[width=.37\textwidth]{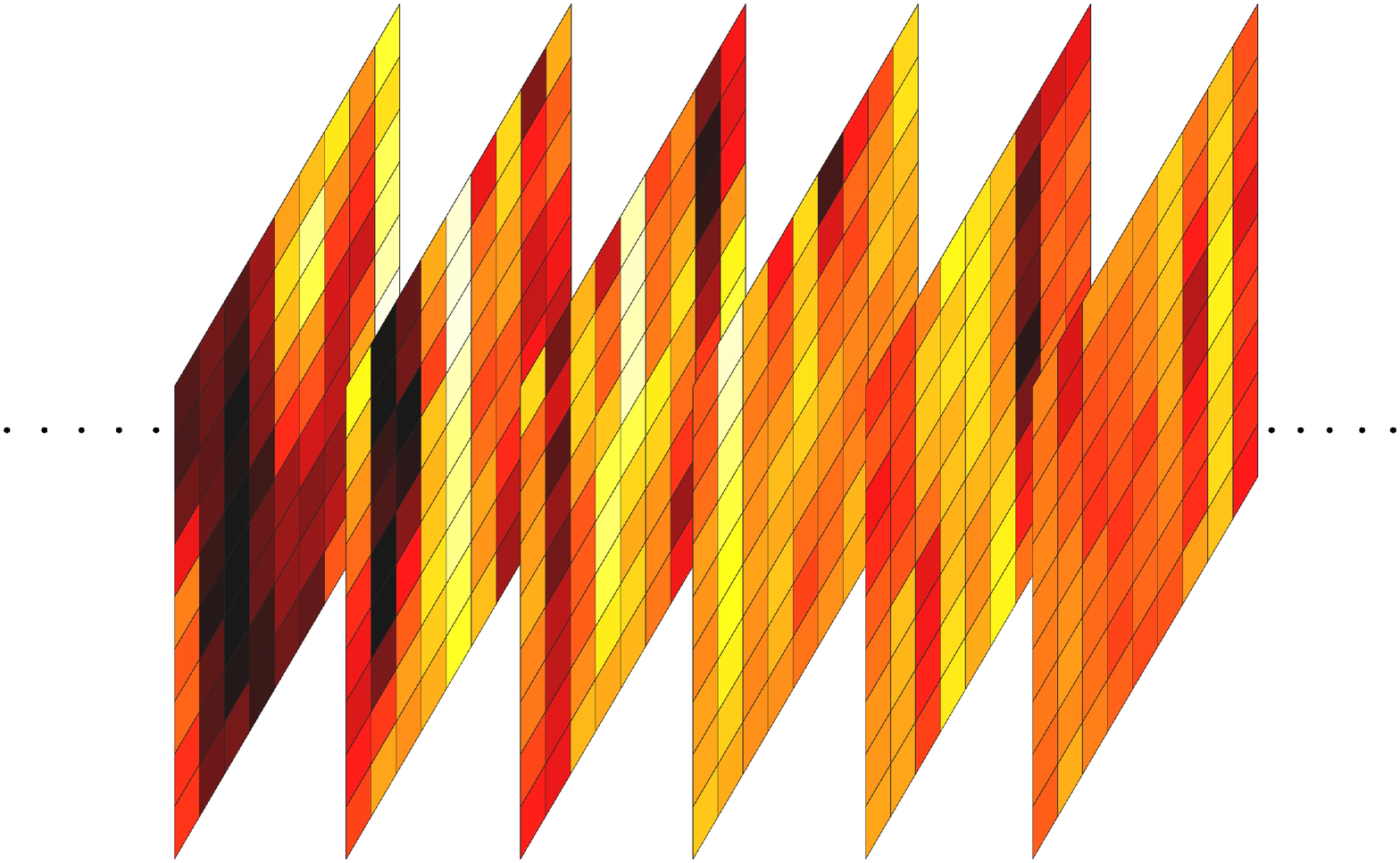}} \hspace{.02\textwidth}
	\subfigure[]{\includegraphics[width=.40\textwidth]{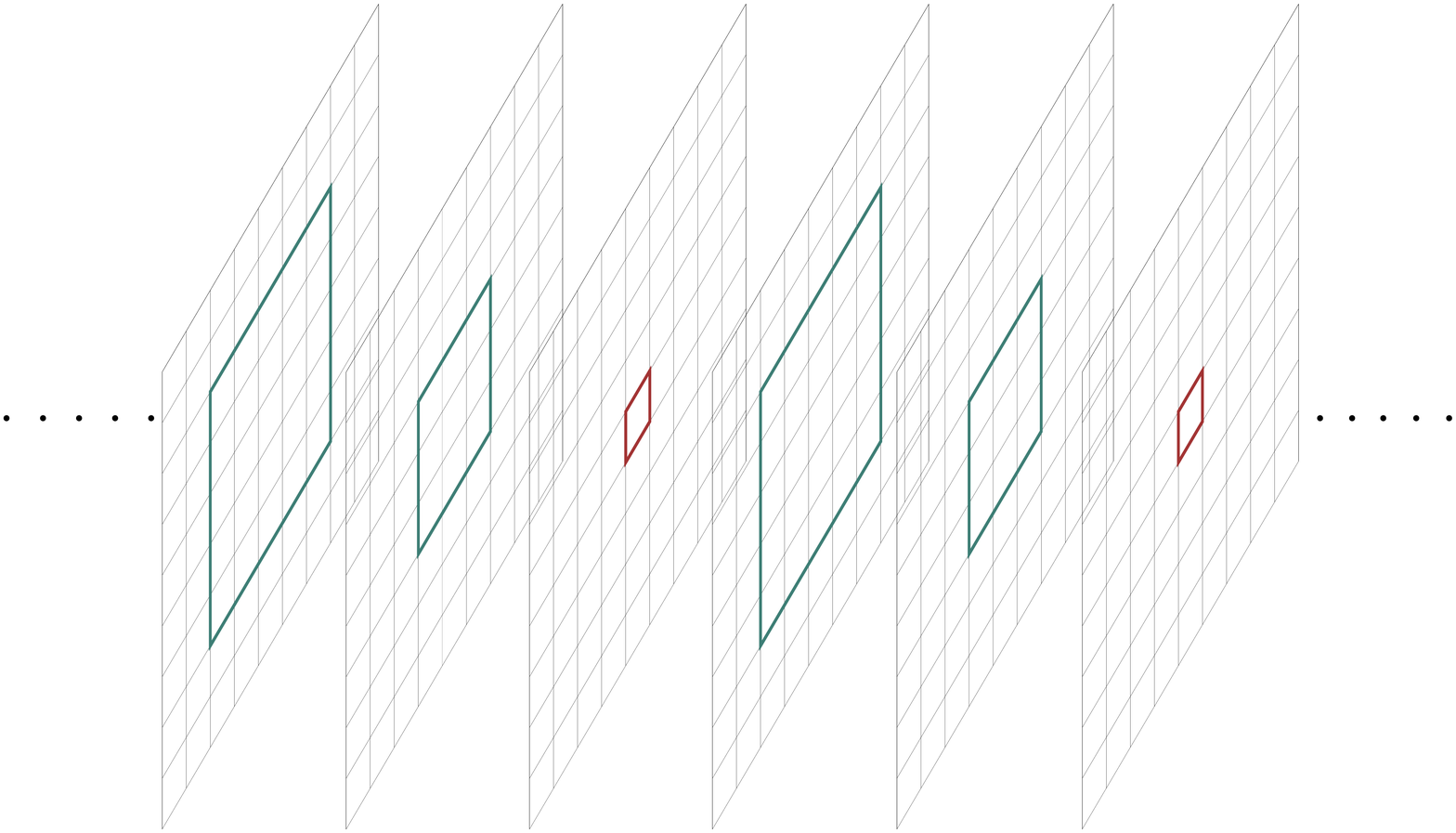}}\hspace{.02\textwidth}
	\subfigure[]{\includegraphics[width=.40\textwidth]{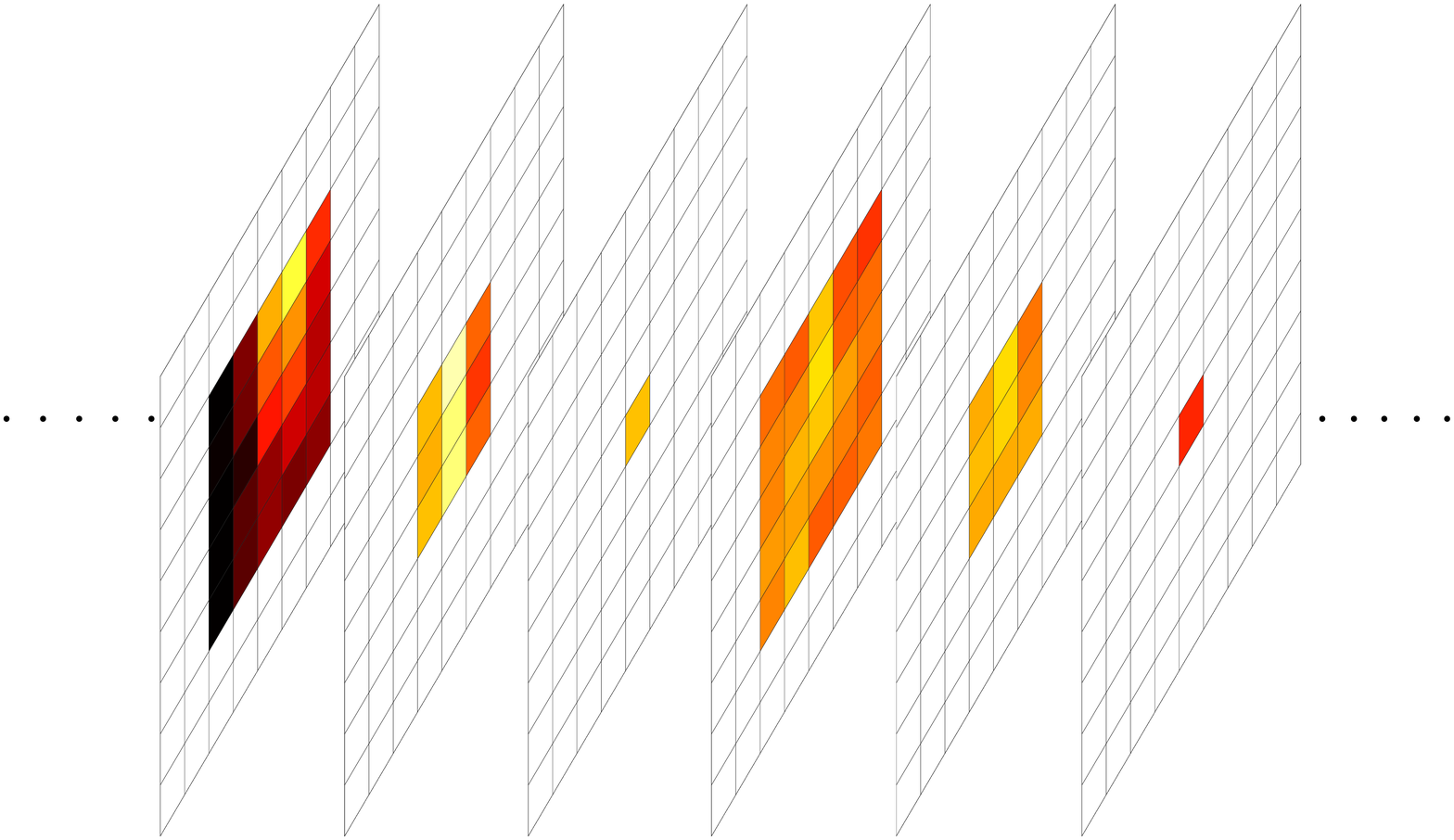}}  \vspace{-0.5cm}
	\caption{(a) Observed raster data cube. (b) Spatio-temporal embedding: the red pixel identifies the spatial point $x^*$, and for $i \in \Z$ the set $\mathcal{I}(t_0+ia,x^*)$ defined in (\ref{index}) is represented by the pixels in the green boxes. The parameters in use in this sampling are $c=\sqrt{2}$, $p_t=2$, $a_t=3$, $h_t=h_s=1$. (c) A realization $S$ from the cone-shaped sampling process.}\label{cone}
\end{figure}

Next, let us assume to observe a data set $(Z_t(x))_{(t,x) \in \mathbb{T} \times \mathbb{L}}$, and that we want to determine a one-time ahead ensemble forecast in the pixel $x^*$. We define $m:=\Big \lfloor \frac{N}{a_t} \Big \rfloor$ and a training data set $S_m =((X_i,Y_i)^{\top}))_{i=1}^m$ as a realization from the \emph{cone-shaped sampling process} $\bb{S}$ of fixed length $m$. In particular, 

\begin{equation}
	\label{sampling}
	X_i=L^{-}_p(t_0+ia,x^*), \, \,\,\, \textrm{and} \,\,\,\,\, Y_i=Z_{t_0+ia}(x^*)\,\,\,\,\textrm{for $i=1,\ldots,m$},
\end{equation}
where 
\begin{equation}
	\label{to work}
	L^{-}_p(t,x^*)=(Z_{i_1}(\xi_{1}), \ldots, Z_{i_{a(p,c)}}(\xi_{a(p,c)}))^{\top}, \,\,\,\textrm{and $(i_s,\xi_s) \in \mathcal{I}(t,x^*)$}
\end{equation}
for $s=1,\ldots,a(p,c)$ and $t=t_0+ia$ with $i=1,\ldots,m$.
We assume that the parameters $a$ and $p$ follow the constraints in Table \ref{Tab1}. The index set used to define the training data set (\ref{sampling}) is a spatio-temporal embedding in the set $ \R \times \R^2$. A similar interpretation can be given for the index set defining the cone-shaped sampling process (\ref{st_emb}).

\begin{table}[h!]
	\small
	\centering
	\begin{tabular}{|c|c|c|}
		\hline
		\textbf{ Parameters} & \textbf{Constraints} & \textbf{Interpretation}\\
		\hline
		$a:= a_t h_t$  & $p_t+1 \leq a_t \leq \Big \lfloor \frac{N}{2} \Big \rfloor$ & \emph{translation vector} \\
		\hline
		$p:= p_t h_t$  &  $1 \leq p_t < \Big \lfloor \frac{N}{2} \Big \rfloor -1$ &\emph{past time horizon}\\
		\hline
		$m:= \Big\lfloor \frac{N}{a_t} \Big\rfloor$  &  $2< m < N$  &  \emph{number of examples in $S_m$}\\
		\hline
	\end{tabular}
	\caption{Parameters involved in pre-processing $N$ observed frames.}
	\label{Tab1}
\end{table}

For an observed raster data cube, we know the value of the constants $N$, $h_t$, and $h_s$, and it remains to select the parameters $a_t$ and $p_t$. We discuss the selection of the parameter $a_t$ in Section \ref{sec4.4} and of $p_t$ in  Section \ref{sec5.1}, respectively. 

\subsubsection{Study of the dependence structure of $\bb{L}$ and $\bb{L}^{\epsilon}$}
\label{inheritance}

We analyze in this section the dependence structure of the processes $\bb{L}:=(L(h(\bb{X}_i),\bb{Y}_i))_{i \in \Z}$ and $\bb{L^{\epsilon}}:=(L^{\epsilon}(h(\bb{X}_i),\bb{Y}_i))_{i \in \Z}$, where  $L$ and $L^{\epsilon}$ are the loss functions defined in Section \ref{setting}, and $h$ is a Lipschitz predictor.
\begin{Proposition}
	\label{heredithary}
	Let $\bb{S}$ be the cone-shaped sampling process defined by (\ref{st_emb}), then  $\bb{L}$ is a $\theta$-weakly dependent process for all $h \in \mathcal{H}$. Moreover, for $a, p>0$, $k \in \N$, and $r = ka-p > 0$, it has coefficients  
	\begin{equation}
		\label{bound_gen}
		\theta(k) \leq \tilde{d} (Lip(h)a(p,c) + 1)\Big( \frac{2}{\tilde{d}} \E[|\bb{Z}_t(x)-\bb{Z}_t^{(r)}(x)|] + \theta_{lex}(r) \Big),
	\end{equation}
	where $\tilde{d} >0$ is a constant independent of $r$, and $\bb{Z}_t^{(r)}(x):=\bb{Z}_t(x) \wedge r $.
\end{Proposition}

\begin{Remark}[Locally Lipschitz predictor]
	\label{loclip}
	Let the predictor $h$ be a locally Lipschitz function such that
	$h(0)=0$ and 
	\[
	|h(x)-h(y)|\leq \tilde{c} \, \|x-y\|_1 (1+\|x\|_1+\|y\|_1) \,\, \textrm{for $x,y \in \R^{a(p,c)}$},
	\]
	for $\tilde{c}>0$. Moreover, let $\bb{Z}$ be a stationary and $\theta$-lex weakly dependent random field such that $|\bb{Z}|\leq C$ almost surely. An easy generalization of Proposition \ref{heredithary}, leads to show that $\bb{L}$ is $\theta$-weakly dependent with coefficients
	\begin{equation}
		\label{bound_locally}
		\theta(k) \leq \tilde{d} (\tilde{c}(1+2 C)a(p,c)+1) \Big( \frac{2}{\tilde{d}} \E[|\bb{Z}_t(x)-\bb{Z}_t^{(r)}(x)|] + \theta_{lex}(r) \Big),
	\end{equation}	
	where  $r = ka-p > 0$ for $a, p>0$ and $k \in \N$, $\tilde{d}>0$ is a constant independent of $r$, and $\bb{Z}_t^{(r)}(x):=\bb{Z}_t(x) \wedge r $.
\end{Remark}

\begin{Remark}
	\label{timeseries}
	The spatio-temporal embedding discussed in Section \ref{st_embedding} also apply to $\theta$-weakly dependent time series models $\bb{Z}$. In such case, the parameter $c=0$,  $\bb{S}$ is \emph{a flat cone-shaped sampling process}, and (straightforwardly) a $\theta$-weakly dependent process with coefficients satisfying the bound (\ref{bound_gen}).
\end{Remark}

In the case of MMAFs, we have obtained explicit bounds for the $\theta$-lex coefficients in Propositions \ref{proposition:mmathetaweaklydep} and \ref{example:spatiotemporaldata}. We now prove that a more refined bound than (\ref{bound_gen}) for the $\theta$-coefficients of the process $\bb{L}$ can be given in this setting.

We consider the following assumption.

\begin{Assumption}
	\label{ass1}
	Let $\bb{Z}$ be an MMAF under the Assumption \ref{ass_light} and such that $\E[|\bb{Z}_t(x)|^2] <\infty$.
	Let $N \in \N$, $\mathbb{T}=\{t_0+h_t,\ldots,t_0+h_tN\}$ and $\mathbb{L} \subset \R^2$. $S_m$ and $S$ are realization from $(\bb{Z}_t(x))_{(t,x)\in \Z \times \mathbb{L}}$ following the spatio-temporal embedding defined in Section \ref{st_embedding}.
\end{Assumption}

\begin{Proposition}
	\label{mmaf_absolute}
	Let Assumption \ref{ass1} hold. Then $\bb{L}$ is a $\theta$-weakly dependent process for all $h \in \mathcal{H}$ with coefficients  
	\begin{equation}
		\label{bound3}
		\theta(k) \leq 2 (Lip(h) a(p,c) + 1)  \widetilde{\theta}_{lex}(r),
	\end{equation}
	where $r = ka-p > 0$ for $a, p>0$ and $k \in \N$.
	In particular, for linear predictors, i.e., $h_{\beta}(X)=\beta_0+ \beta_1^T X,\, \textrm{for}\,\, \beta:=(\beta_0,\beta_1)^{\top} \in  B $ and $B= \R^{a(p,c)+1}$, we have that $\bb{L}$ is a $\theta$-weakly dependent process for all $\beta \in B$ with coefficients  
	\begin{equation}
		\label{bound_linear}
		\theta(k) \leq 2 (\|\beta_1\|_1 + 1)  \widetilde{\theta}_{lex}(r).
	\end{equation}	
	
\end{Proposition}

Lemma \ref{truncated} straightforwardly implies that the process $\bb{L}^{\epsilon}$ is $\theta$-weakly dependent with the same $\theta$-coefficients as the process $\bb{L}$ under the assumptions of Proposition \ref{heredithary}, Remark \ref{loclip} and Proposition \ref{mmaf_absolute}.

%

\subsection{PAC Bayesian bounds for MMAF generated data}
\label{sec4.4}

Three essential results, namely, the change of measure theorem of Donsker and Varadhan \cite{DV76}, the Markov's inequality, and an exponential inequality, are typically employed to prove a \emph{fixed-time PAC Bayesian bound}, which holds for a given choice of $m$ (which in our framework is related to a given number $N$ of frames). We find in \cite{BG16} a \emph{scheme of proof} for PAC Bayesian bounds that summarizes the above. 
There is, however, another scheme of proof described in \cite{Uniform}, which allows us to obtain any-time PAC Bayesian bounds, which hold simultaneously for all $m$. Such methodology avoids using the Markov's and the exponential inequalities by substituting them with the Ville's inequality for non-negative supermartingale \cite{Ville}. To the best of our knowledge, in the dependent case, the existing proofs of such bounds are holding in the so-called \emph{online framework} and make use of martingale properties as detailed in \cite{Uniform} and \cite{SMartingale1}.

In this section, we prove a novel exponential inequality for $\theta$-weakly dependent processes (the dependence property of $\bb{L}^{\epsilon}$) such to be capable of applying the scheme of proof of \cite{BG16} and obtain a fixed-time bound. Moreover, we also prove an any-time bound for $\theta$-weakly dependent processes in the \emph{batch learning framework}. We then define a novel \emph{scheme of proof} for fixed time bounds that combines an any-time bound with the Markov's inequality and the projective-type property of the $\theta$-weak dependence (\ref{mix2}). The last part of this \emph{scheme of proof} is inspired from the analysis made in \cite{Hostile}. Our examples and simulation studies show that the latter bound allows us to \emph{guide} a randomized estimator employing realistic spatio-temporal embeddings, which means training data sets that can best preserve the serial correlation observed in a given \emph{raster data cube}. From the definition of the spatio-temporal embedding in Section \ref{st_embedding}, because the MMAFs are $\theta$-lex weakly dependent fields, the more the parameter $a_t$ becomes bigger, the more the examples in $S_m$ are less correlated. This phenomenon is a consequence of the dependence structure of the data generating process and significantly impacts the generalization performance of the employed estimators.

For a measurable space $(\mathcal{H}, \mathcal{T})$ and for any $(\rho,\pi) \in \mathcal{M}_+^1(\mathcal{H})^2$, where $\rho <<\pi$ means that $\rho$ is absolutely continuous respect to $\pi$ with Radon-Nikodym derivative $\frac{d\rho}{d\pi}$,
the PAC Bayesian bounds introduced in this section employ the \emph{Kullback-Leibler divergence} 
\[
KL(\rho,\pi) =\left\{ \begin{array}{ll}
	\rho\Big[ \log \frac{d\rho}{d\pi} \Big] \, \,\, &\textrm{if} \, \,\, \rho <<\pi	\\
	+\infty \, \, &\textrm{otherwise}
\end{array} \right.,
\]
and for $\phi_p(x)=x^p$, the \emph{f-divergences} defined as
\[
D_{\phi_p-1}(\rho,\pi)=\left\{ \begin{array}{ll}
	\pi\Big[ f\Big( \frac{d\rho}{d\pi}\Big) \Big] \, \,\, &\textrm{if} \, \,\, \rho <<\pi	\\
	+\infty \, \, &\textrm{otherwise}
\end{array} \right..
\]
For $p=2$, we have that $D_{\phi_2-1}(\rho,\pi)$ corresponds to the chi-square divergence.

We prove next an exponential inequality for $\theta$-weakly dependent processes using the notations introduced in Section \ref{sec4.2}.

\begin{Theorem}
	\label{new}
	Let $\bb{Z}$ be an $\R$-valued stationary $\theta$-weakly dependent process, and $f:\R \to [a,b]$, for $a,b \in \R$ such that $(f(\bb{Z}_i))_{i \in \Z}$ is itself $\theta$-weakly dependent. Let $l=\lfloor \frac{m}{k} \rfloor$, for $m,k \in \N$, such that $l\geq2$ and $0<s<\frac{3l}{|b-a|}$, then 
	\begin{align}
		&\E\Big[ \exp\Big( \frac{s}{m} \sum_{i=1}^m (f(\bb{Z}_i)-\E[f(\bb{Z}_i)])  \Big) \Big] \leq \exp \Bigg(\frac{s^2 Var(f(\bb{Z}_1))}{2 l \Big(1-\frac{s|b-a|}{3l} \Big)}  \Bigg) +  \exp( s \, |b-a|) \theta(k) s, 	\label{laplace}\\
		&\E\Big[ \exp\Big( \frac{s}{m} \sum_{i=1}^m (\E[f(\bb{Z}_i)]-f(\bb{Z}_i))  \Big) \Big] \leq \exp \Bigg(\frac{s^2 Var(f(\bb{Z}_i))}{2 l \Big(1-\frac{s|b-a|}{3l} \Big)}  \Bigg) + \exp( s \, |b-a|) \theta(k) s \label{laplace2}. 
	\end{align}
\end{Theorem}

Note that increasing the parameter $k$ makes the bounds in (\ref{laplace}) and (\ref{laplace2}) become tighter because $\theta(k)$ goes to zero when $k$ goes to infinity. 

We have shown in the last section that the process $\bb{L}^{\epsilon}$ is $\theta$-weakly dependent. We then apply Theorem \ref{new} in the scheme of proof described in \cite{BG16} and determine a \emph{fixed-time PAC Bayesian bound}.

\begin{Theorem}[Fixed-time PAC Bayesian bound- \emph{Type I}]
	\label{prop_PAC1}
	Let $0<\epsilon<3$, $l=\Big\lfloor \frac{m}{k} \Big \rfloor$ such that $l \geq 2$, and Assumption \ref{ass1} holds. If $\pi \in \mathcal{M}_+^1(\mathcal{H})$ such that $\pi [\theta(k)] < \infty$, $\hat{\rho} \in \mathcal{M}_+^1(\mathcal{H})$ in the sense of Definition (\ref{datadep}), $\hat{\rho} <<\pi$, and $\delta \in (0,1)$, then
	\begin{align}
		\mathbb{P}\Big\{ \forall \hat{\rho}: |\hat{\rho}[R^{\epsilon}(h)] &- \hat{\rho}[r^{\epsilon}(h)] |\leq  \Big(KL(\hat{\rho},\pi)+\log \Big(\frac{1}{\delta}\Big) \Big) \frac{1}{\sqrt{l}} \nonumber\\ 
		&+\frac{1}{\sqrt{l}} \log \Big(\exp\Big(\frac{3\epsilon^2}{2(3-\epsilon)}\Big) +\pi \Big[ 3\sqrt{l} \exp(3\sqrt{l}) \theta(k) \Big]  \Big)  \Big\} \geq 1-2\delta, \label{PAC1}
	\end{align}
	where $\theta(k)$ is a $\theta$-coefficient of the process $\bb{L}^{\epsilon}$.
\end{Theorem}

For any possible choice of the parameters defining the spatio-temporal embedding leading to $\bb{S}_m$, Theorem \ref{prop_PAC1} gives us a PAC Bayesian bound for a randomized estimator $\hat{\rho}$. The fastest convergence rate that can be obtained in this framework is $\mathcal{O}(m^{-1/2})$ when choosing the parameter $k=1$. 

Given an accuracy level $\epsilon$, if the right-hand side of the PAC Bayesian bound is less than $\epsilon$ with a high probability (i.e., the bound is not vacuous), we say that a randomized estimator $\hat{\rho}$ has \emph{good generalization performance}. The smaller the right-hand side becomes, the better the performance of a randomized estimator in forecasting is. We then aim to select a spatio-temporal embedding such as the right-hand side of the inequality (\ref{PAC1}) is small as possible. For the reader's convenience, we have summarized all the parameters appearing in MMAF-guided learning in Table \ref{para}.

\begin{table}[h!]
	\footnotesize
	\centering
	\begin{tabular}{|c|c|c|}
		\hline
		\textbf{Parameters} & \textbf{Type} & \textbf{Interpretation}\\
		\hline
		$\epsilon$ & Hyperparameter &  Accuracy level \\
		\hline
		$N$ & Given Parameter & Number of frames \\
		\hline
		$h_t$  & Given Parameter & time step \\
		\hline
		$h_s$ & Given Parameter & space step \\
		\hline
		$\lambda$ & Unknown Parameter &  Decay rate of the $\theta$-lex coefficients \\
		\hline
		$c$ & Unknown Parameter & Speed of information propagation \\
		\hline 
		$p_t$ & Hyperparameter &  Length of the past included in each input $X_i$ \\
		\hline
		$a_t$ & User Choice Parameter  &  Translation vector \\
		\hline
		$k$  & User Choice Parameter & $k$-th $\theta$-coefficient of the process $\bb{L}^\epsilon$ \\
		\hline
	\end{tabular}
	\caption{Overview of parameters appearing in MMAF-guided learning.}
	\label{para}
\end{table}

In the following, we show in our examples and theorems how to \emph{guide} the design of two type of randomized estimators, see Definition \ref{datadep}: 

\begin{itemize}
	\item \emph{a Dirac delta mass concentrated on the empirical risk minimizer}, defined for all $E \in \mathcal{T}$ and $\omega \in \Omega$ as
	\begin{equation}
		\label{delta}
		\hat{\rho}(E,\omega):= \delta_{\bb{\hat{\beta}}(\omega)},
	\end{equation}
	where $\bb{\hat{\beta}}: \Omega \to \mathcal{H}$ defines the \emph{empirical risk minimizer}
	\begin{equation}
		\label{ols}
		\bb{\hat{\beta}}:=arg\inf_{\beta} r^{\epsilon}(\beta);
	\end{equation}
	\item and \emph{the randomized Gibbs estimator} $\bar{\rho}$ defined for all $E \in \mathcal{T}$ and $\omega \in \Omega$ as
	\begin{equation}
		\label{Gibbs}
		\bar{\rho}(E,\omega):=\frac{\int_E \exp(-\sqrt{m}r^{\epsilon}(h)) \pi(dh)}{\int_{\mathcal{H}} \exp(-\sqrt{m}r^{\epsilon}(h)) \pi(dh)},
	\end{equation}
	where $r^{\epsilon}(h):=r(h) \wedge \epsilon $, and $r(h)$ is defined in (\ref{emp_err}). 
\end{itemize}
Note that the definition of a randomized estimator depends on $\bb{S}_m$. Therefore, different values of the parameters $\epsilon, a_t, p_t, \lambda$ and $c$  give us different distributions to be used in forecasting tasks.

\begin{Remark}[Possible selection rule for $\sqrt{l}$ and the parameter $a_t$ in relation to the bound (\ref{PAC1})]
	\label{selection1}
	Differently from the classical PAC Bayesian for independently and identically distributed data, see reviews \cite{User} and \cite{Primer}, the parameter $l$ is tuned in the bounds (\ref{PAC1}) through the choice of the spatio-temporal embedding. We now give several examples of the latter that can be chosen to tighten the right-hand side of the bound (\ref{PAC1}). In Proposition \ref{mmaf_absolute}, we have determined for a Lipschitz predictor a bound from above of the $\theta$-coefficient $\theta(k)$ related to the process $\bb{L^{\epsilon}}$. 
	If we employ this result, we get a PAC Bayesian bound where we are capable of dividing in the last term the contribution of a predictor $h$ from the dependence structure of the underlying MMAF field $\bb{Z}$ expressed through the $\theta$-lex coefficients. By applying Proposition \ref{mmaf_absolute}, the $Lip(h)$ appears in the bound.
	
	Therefore for $l \geq2$, and $r=ka-p$, we obtain the bound
	\begin{small}
		\begin{align}
			\mathbb{P}\Big\{ \forall &\hat{\rho}: |\hat{\rho}[R^{\epsilon}(h)] - \hat{\rho}[r^{\epsilon}(h)] |\leq  \Big(KL(\hat{\rho},\pi)+\log \Big(\frac{1}{\delta}\Big) \Big) \frac{1}{\sqrt{l}} \nonumber\\ 
			&+\frac{1}{\sqrt{l}} \log \Big(\exp\Big(\frac{3\epsilon^2}{2(3-\epsilon)}\Big) +\pi \Big[ 3\sqrt{l} \exp(3\sqrt{l}) 2(Lip(h)a(p,c)+1) \bar{\alpha} \, \tilde{\theta}_{lex}(r) \Big]  \Big)  \Big\} \geq 1-2\delta. \label{PAC2}
		\end{align}
	\end{small}
	
	We now consider that the parameters $\epsilon, p_t, \lambda$, $c$, and $\bar{\alpha}$ (which determines how tight $\tilde{\theta}(r)$ is as a bound of $\theta(k)$) have been selected and we concentrate on the selection of the parameter $a_t$. We refer the reader to Section \ref{sec5} for a discussion on the selection of the other parameters.

	For $k=1$, we obtain a convergence rate of $\mathcal{O}(m^{-\frac{1}{2}})$, and we can pre-process the data choosing the smallest value of $a_t$ satisfying the following inequalities
	\begin{equation*}
		\label{sample1}
		\left\{ \begin{array}{ll}
			-\lambda h_t(a_t-p_t)+3\sqrt{\frac{N}{a_t}}<0  &\textrm{for exp. decaying $\theta$-lex coef., see Def. \ref{decay}}, 	\\
			-\lambda \log(h_t(a_t-p_t))+3\sqrt{\frac{N}{a_t}} <0 &\textrm{for power decaying $\theta$-lex coef., see Def, \ref{decay}}.
		\end{array} \right.
	\end{equation*}
	Under these choices, we obtain that $\exp(3\sqrt{l}) \tilde{\theta}_{lex}(r) \leq 1$, and the bound (\ref{PAC2}) tightens. 
	For $k >1$, we obtain a convergence rate of $\mathcal{O}((m/k)^{-\frac{1}{2}})$, and we can pre-process the data choosing the smallest value of $a_t$ satisfying the following inequalities
	\begin{equation*}
		\label{sample2}
		\left\{ \begin{array}{ll}
			-\lambda h_t(ka_t-p_t)+3\sqrt{\frac{N}{a_t}}<0 &\textrm{for exp. decaying $\theta$-lex coef.}, 	\\
			-\lambda \log(h_t(ka_t-p_t))+3\sqrt{\frac{N}{a_t}} <0 &\textrm{for power decaying $\theta$-lex coef.}.
		\end{array} \right.
	\end{equation*}
	By choosing $k>1$, we could however select a lower value of $a_t$.
	
	The choice of the parameter $a_t$ discussed in this section is inversely proportional to the parameter $k$ and proportional to the parameter $p_t$. This means we obtain more examples (i.e., longer training data sets $S_m$) when $k$ increases and $p_t$ decreases. Hence, a careful choice of the parameter $a_t$ must be done even to obtain training data sets with $m \geq 1$, especially when the data generating process $\bb{Z}$ admits power decaying $\theta$-lex coefficients.
\end{Remark}

\begin{Example}[\emph{Guided} randomized estimators of type (\ref{delta})]
	\label{finite}
	Let us assume to observe two data sets with $N=20000$ from an STOU and an MSTOU process as defined in Example \ref{coeff_stou1} and \ref{coeff_mstou1}, respectively. The parameters $\lambda=\frac{1}{2},p_t=1,c=1,h_t=1$ and $h_s=1$ are associated to the STOU data set and $\lambda=\frac{1}{2}, \alpha=3,p_t=1,c=1,h_t=1$, and $h_s=1$ to the MSTOU data set.
	The former is a serially correlated data set having temporal and spatial short-range dependence and exponentially decaying $\theta$-lex coefficients, whereas the latter possesses temporal and spatial long-range dependence and power decaying $\theta$-lex coefficients.
	Let us choose $a_t$ following the rule in Remark \ref{sample1} in the case of linear predictors.
	We call $\bb{S_m^1}$ the sampling from the STOU, where $a_t=92$ and $m=217$, and  $\bb{S_m^2}$ the one from the MSTOU, where $a_t=8742$ and $m=2$. We have that the dimension of the input space of our examples is $a(p,c)=3$ and we consider the parametric space $B\subset \{ \beta=(\beta_0,\beta_1): \beta_0 \in \R, \beta_1 \in \R^3 \,\text{and}\, \|\beta\|_1 \leq 1\}$ such that $card(B)=M < \infty$ for $M \in \N$. We look at the bound (\ref{PAC2}) for $k=1$. We consider as reference distribution a uniform distribution $\pi$ on $B$. For a realization $S_m^1$ or $S_m^2$, we have that
	\begin{equation}
		\label{reference}
		KL(\delta_{\hat{\beta}}||\pi)=\sum_{\beta \in B} \log \Big( \frac{\delta_{\hat{\beta}}\{\beta\}}{\pi\{\beta\}} \Big) \delta_{\hat{\beta}}\{\beta \}= \log \frac{1}{\pi\{\hat{\beta} \}}=\log(M).
	\end{equation}
	It is crucial to notice that the bigger the cardinality of the space $B$ is, the more the term $\log(M)$ and the bound increase. 
	
	Finally, for $\delta \in (0,1)$, we obtain for $\bb{S_m^1}$ and $\bb{S_m^2}$ (obviously w.r.t. different distribution $\mathbb{P}$ depending on their data generation process) the bound
	\begin{align}
		\label{erm_1}
		\mathbb{P}\Big\{ |R^{\epsilon}(\bb{\hat{\beta}}) - r^{\epsilon}(\bb{\hat{\beta}})| \leq &\log \Big( \frac{M}{\delta} \Big)  \frac{1}{\sqrt{m}} +\frac{1}{\sqrt{m}} \log \Big( \exp\Big(\frac{3\epsilon^2}{2(3-\epsilon)}\Big) \nonumber\\ &+ 3\sqrt{m}\pi [2(\|\beta_1\|_1 +1 )\bar{\alpha}] \Big) \Big\} \geq 1-2\delta.
	\end{align}
	Let us assume to work with an accuracy level $\epsilon=1$, $M=100$, $\bar{\alpha}=1$, $\delta=0.025$. Then, the generalization gap with respect to $\bb{S_m^1}$ is less than $0.98$ with at least $95 \%$ probability, whereas for $\bb{S_m^2}$ we can just prove that the generalization gap is less than $9.40$ with probability of at least $95 \%$. The latter is obviously a vacuous bound.
\end{Example}

The randomized Gibbs estimator is the \emph{minimizer} of the right hand-side of the bound (\ref{PAC1}) and gives the best generalization performance in the class $\mathcal{M}^1_+(\mathcal{H})$ and has the best possible rate of convergence for $k=1$. The result below is also called an \emph{oracle inequality} in the literature.

\begin{Theorem}[PAC Bayesian bound for the randomized Gibbs estimator- \emph{Type I}]
	\label{oracle}
	Let $0<\epsilon<3$, $m \geq 2$, and Assumption \ref{ass1} holds. If $\pi$ is a distribution on $\mathcal{M}^1_+(\mathcal{H})$ such that $\pi [\theta(1)] < \infty$, $\bar{\rho}$ is the randomized Gibbs estimator defined in (\ref{Gibbs}), and $\delta \in (0,1)$ 
	\begin{align}
		\mathbb{P}\Big\{ \bar{\rho}[R^{\epsilon}(h)] &\leq  \inf_{\hat{\rho}} \Big(  \hat{\rho}[R^{\epsilon}(h)] +\Big(KL(\hat{\rho},\pi)+\log \Big(\frac{1}{\delta}\Big)  \frac{2}{\sqrt{m}} \Big)\nonumber\\ 
		&+\frac{2}{\sqrt{m}} \log \Big(\exp\Big(\frac{3\epsilon^2}{2(3-\epsilon)}\Big) +\pi \Big[ 3\sqrt{m} \exp(3\sqrt{m}) \theta(1) \Big]  \Big)  \Big) \Big\} \geq 1-2\delta. \label{PAC_G1}
	\end{align}
\end{Theorem}

We now focus on determining the any-time PAC Bayesian bound in our framework. Let us define the filtration $\mathbb{F}=(\mathcal{F}_m)_{m \in \N_0}$ where $\mathcal{F}_m=\sigma(\bb{S}_m)$ for all $m \geq 1$, i.e., the filtration generated by the cone-shaped sampling process, and $\mathcal{F}_0$ is equal to the trivial sigma-algebra. For $\epsilon >0$, we then define the process $(f_i(\bb{S},h))_{i \in \N_0}$ as

\[
\left\{ \begin{array}{ll}
	f_0(\bb{S},h)  =&0 \\
	f_m(\bb{S},h) =& \eta \sum_{i=1}^m (L^{\epsilon}(h(X_i),Y_i)-\E[L^{\epsilon}(h(\bb{X}_i),\bb{Y}_i)]) \\&- \eta \sum_{i=1}^m  (\E[L^{\epsilon}(h(X_i),Y_i)| \mathcal{F}_{i-1}] -\E[L^{\epsilon}(h(\bb{X}_i),\bb{Y}_i)]) \\
	&-\frac{\eta^2}{2} m \epsilon^2 , \,\, \textrm{for $m \in \N$}.
\end{array}\right.
\]
We can then prove an any-time PAC-Bayesian bound which allows a choice of the accuracy level $\epsilon>3$ and, therefore, it complements the results given in Theorem \ref{prop_PAC1} for a bounded loss.

\begin{Theorem}
	\label{anytimePAC}
	Let $\epsilon >0$ and Assumption \ref{ass1} holds. If $\pi \in \mathcal{M}_+^1(\mathcal{H})$ such that $\pi [\theta(1)] < \infty$, $\hat{\rho} \in \mathcal{M}_+^1(\mathcal{H})$ in the sense of Definition (\ref{datadep}), $\hat{\rho} <<\pi$, and $\delta \in (0,1)$, then 
	\begin{align}
		\mathbb{P} \Big \{ &\forall \hat{\rho}, \forall m \geq 1: \hat{\rho}[r^{\epsilon}(h)] - \hat{\rho}[R^{\epsilon}(h)]  \leq   \frac{KL(\hat{\rho},\pi)+\log \Big(\frac{1}{\delta} \Big) }{m \eta} +\frac{\eta}{2} \epsilon^2 \nonumber\\
		&+\hat{\rho} \Big[ \frac{1}{m} \sum_{i=1}^m \E[L^{\epsilon}(h(\bb{X}_i),\bb{Y}_i)|\mathcal{F}_{i-1}] -\E[L^{\epsilon}(h(\bb{X}_i),\bb{Y}_i)] \Big]  \Big \} \geq 1-\delta. \label{anytimePAC1}
	\end{align}
	
\end{Theorem}

We call the process $ \frac{1}{m} \sum_{i=1}^m (\E[L^{\epsilon}(h(\bb{X}_i),\bb{Y}_i)|\mathcal{F}_{i-1}] -\E[L^{\epsilon}(h(\bb{X}_i),\bb{Y}_i)])_{i \in \N}$ the \emph{residual process}.

Following the same line of proof, it can also be proven that an any-time PAC Bayesian bound holds for the average generalization gap $\hat{\rho}[R^{\epsilon}(h)]-\hat{\rho}[r^{\epsilon}(h)]$, by using the following definition of $f_m(\bb{S},h)$, namely,
\[
\left\{ \begin{array}{ll}
	f_0(\bb{S},h)  =&0 \\
	f_m(\bb{S},h) =& \eta \sum_{i=1}^m (\E[L^{\epsilon}(h(\bb{X}_i),\bb{Y}_i)]-L^{\epsilon}(h(X_i),Y_i)) \\ & -\eta \sum_{i=1}^m  (\E[L^{\epsilon}(h(\bb{X}_i),\bb{Y}_i)]-\E[L^{\epsilon}(h(X_i),Y_i)| \mathcal{F}_{i-1}]) \\
	&- \frac{\eta^2}{2} m \epsilon^2 , \,\, \textrm{for $m \in \N$}.
\end{array}\right.
\] 

The following result follows straightforwardly by applying a union bound.

\begin{Corollary}
	\label{any_time_full}
	Let $\epsilon >0$ and Assumption \ref{ass1} holds. If $\pi \in \mathcal{M}_+^1(\mathcal{H})$ such that $\pi [\theta(1)] < \infty$, $\hat{\rho} \in \mathcal{M}_+^1(\mathcal{H})$ in the sense of Definition (\ref{datadep}), $\hat{\rho} <<\pi$, and $\delta \in (0,1)$, then 
	\begin{align}
		\mathbb{P} \Big\{ \forall & \hat{\rho}, \forall m \geq 1: | \hat{\rho}[R^{\epsilon}(h)] - \hat{\rho}[r^{\epsilon}(h)]| \leq  \frac{KL(\hat{\rho},\pi)+\log \Big(\frac{1}{\delta}\Big)}{\eta m} +\frac{\eta \epsilon^2}{2} \nonumber\\
		&+ \hat{\rho} \Big[ \Big| \frac{1}{m} \sum_{i=1}^m \E[L^{\epsilon}(h(\bb{X}_i),\bb{Y}_i)|\mathcal{F}_{i-1}] -\E[L^{\epsilon}(h(\bb{X}_i),\bb{Y}_i)] \Big |\Big]   \Big \} \geq 1-2\delta. \label{PAC3} 
	\end{align}
\end{Corollary}

We can now give the proof of a fixed-time PAC Bayesian bound obtained by using the bound (\ref{PAC3}), a Markov's inequality and the projective-type property (\ref{mix2}). In this new scheme of proof, it is determined a fixed-time PAC Bayesian bound of the \emph{residual process} following \cite[Theorem 1]{Hostile}.

\begin{Theorem}[Fixed-time PAC Bayesian bound- \emph{Type II}]
	\label{new_fixedtime}
	Let $\epsilon >0$ and Assumption \ref{ass1} holds. If $\pi \in \mathcal{M}_+^1(\mathcal{H})$ such that $\pi [\theta(1)] < \infty$, $\hat{\rho} \in \mathcal{M}_+^1(\mathcal{H})$ in the sense of Definition (\ref{datadep}), $\hat{\rho} <<\pi$, and $\delta \in (0,1)$, then  
	\begin{align}
		\mathbb{P} \Big\{ \forall \hat{\rho}: |\hat{\rho}[r^{\epsilon}(h)] -\hat{\rho}[R^{\epsilon}(h)]| & \leq  \frac{KL(\hat{\rho},\pi)+\log \Big(\frac{1}{\delta}\Big)}{\eta m} \nonumber \\
		&+\frac{\eta \epsilon^2}{2} + \Big( \epsilon \pi\Big[\frac{\theta(1)}{\delta} \Big] (D_{\phi_2-1}(\hat{\rho},\pi)+1) \Big)^{\frac{1}{2}}  \Big \} \geq 1-3\delta, \label{the_best} 
	\end{align}
	where $\theta(1)$ is a $\theta$-coefficient of the process $\bb{L}^{\epsilon}$.
\end{Theorem}

\begin{Remark}[Selection rule for the parameter $a_t$ in relation to the bound (\ref{the_best})]
	\label{selection2}
	We now discuss an exemplary spatio-temporal embedding which can be chosen such to tighten the right hand side of the bound (\ref{the_best}).
	We consider that the parameters $\epsilon, p_t, \lambda$ and $c$ have been selected (the parameter $k$ does not enter in the determination of the spatio-temporal embedding in this case) and we concentrate on the selection of the parameter $a_t$. We refer the reader to Section \ref{sec5} for a discussion on the selection of the other parameters. Let us choose $\eta=\frac{1}{\sqrt{m}}$, and select $a_t$ as the smallest constant such that $\theta(1) \leq \frac{1}{2m}$, which implies
	\begin{equation*}
		\label{sample3}
		\left\{ \begin{array}{ll}
			-\lambda h_t(a_t-p_t)-\log\Big( \frac{a_t}{2N} \Big)\leq 0	 &\textrm{for exp. decaying $\theta$-lex coef., see Def. \ref{decay}}, 	\\
			-\lambda	\log(h_t(a_t-p_t)) -\log\Big( \frac{a_t}{2N} \Big)\leq 0 &\textrm{for power decaying $\theta$-lex coef., see Def. \ref{decay}}
		\end{array} \right.
	\end{equation*}
	So doing we have that the third addend in the right hand side of the bound has the same order of magnitude of the other terms, which gives us the convergence rate of $\mathcal{O}(m^{-\frac{1}{2}})$. We can notice that respect to the selection rule in Remark \ref{selection1}, we can guide the design of a randomized estimator using a lower value of the parameter $a_t$. This ultimately means that we are working with a training data set which admits a stronger serial correlation along the temporal and spatial dimension than the ones obtained using the selection rule in Remark \ref{selection1}.
	
	The bound (\ref{the_best}) can also be used in different ways to guide the selection of a randomized estimator.
	For example, when working with the estimator (\ref{delta}), we can choose a value of $\eta$ that minimizes the right-hand side of the bound and then selecting a spatio-temporal embedding that makes the bound not vacuous. Examples of such a spatio-temporal embedding are used in Example \ref{finite2}. 
\end{Remark}

\begin{Remark}[Hostile Framework]
	\label{whynot}
	Theorem \ref{new_fixedtime} employs a well-known result in the PAC Bayesian literature for dependent and stationary data, analyzed in \cite{Hostile}. The authors determine a general fixed-time bound holding for unbounded losses; see Theorem 1. In particular, let $\delta \in (0,1)$, $p>1$, and $q=\frac{p}{p-1}$, with probability at least $1-\delta$, they prove that for any $\hat{\rho}$ 
	\[	
	\hat{\rho}[R(h)] \leq \hat{\rho}[r(h)] + \Big(\frac{\mathcal{M}_{\phi_q,n}}{\delta} \Big)^{\frac{1}{q}} (D_{\phi_p-1}(\hat{\rho},\pi)+1)^{\frac{1}{p}}.
	\] 
	This bound obviously applies to our set-up, but it cannot be used in its general shape to \emph{guide} the selection of a randomized estimator, which is the primal target of our paper. The reason for this is that we cannot write the term $\mathcal{M}_{\phi_q,n}$ as a function of the $\theta$-coefficients of the process $\bb{L^\epsilon}$, similarly, as done in \cite{Hostile} in the case of $\alpha$-mixing processes. When working with such a dependence notion, the term $\mathcal{M}_{\phi_q,n}$ can be controlled using Lemma 3 (for bounded losses) and Theorem 3 (for unbounded losses) in \cite{D94}. To the best of our knowledge, no proofs in the literature extend such results for $\theta$-weakly dependent processes. Moreover, even if we could determine such a proof, following the methodology described in \cite{Hostile}, we will end up estimating the term $\mathcal{M}_{\phi_q,n}$ from above using  $\sum_{j \in \mathcal{Z}} \theta(j)$. This estimate diverges for specific power decaying $\theta$-coefficient sequences, such as those discussed in Example \ref{coeff_mstou1}. 
	
	Therefore, we developed novel results where the PAC Bayesian bounds depend just on \emph{one single} $\theta$-coefficient and are also not equal to infinity for a given data set generated by a temporal and spatial long-range dependent MMAF. 
\end{Remark}

\begin{Example}[\emph{Guided} randomized estimators of type (\ref{delta}) and comparisons between the bounds (\ref{PAC2}) and (\ref{the_best})]
	\label{finite2}
	We use the data described in Example \ref{finite}. We also maintain the same assumptions on the set $B$ and the reference distribution $\pi$. Moreover, we select parameters $\epsilon=1, M=100, \delta=\frac{0.05}{3}$, and $c=p_t=h_t=1$. From \cite[Proposition 1]{Hostile}, we have that $(D_{\phi_2-1}(\hat{\rho},\pi)+1)=M$. We compute first the bound (\ref{the_best}) for the estimator (\ref{delta}) and the data set observed from the STOU process. We use the bound (\ref{bound_linear}) for the $\theta$-coefficients of the process $\bb{L}^{\epsilon}$, and assume that $\bar{\alpha}=1$. Moreover, we select the parameter as $\eta=\sqrt{\frac{2\log(M/\delta)}{\epsilon^2 m}} $. Such choice minimizes the right-hand side of the bound  (\ref{the_best}) in this framework, we then obtain for $\delta \in (0,1)$ that
	\begin{equation}
		\label{erm_2}
		\mathbb{P}\Big\{ |R^{\epsilon}(\bb{\hat{\beta}}) -  r^{\epsilon}(\bb{\hat{\beta}})| \leq \sqrt{\frac{2\log(M/\delta)}{m}} +\Big(\epsilon \frac{M 2(\pi[\|\beta_1\|_1] +1 )\bar{\alpha} \tilde{\theta}_{lex}(a-p)}{\delta} \Big)^{\frac{1}{2}} \Big\} \geq 1-3\delta.
	\end{equation}
	We then select the smallest parameter $a_t$ such that $\tilde{\theta}_{lex}(a-p) \leq \frac{\delta}{4Mm}$. We obtain an $\bb{S_m^1}$, where $a_t=34$ and $m=588$, and a generalization gap less than or equal to $0.21$ with at least $95 \%$ probability. With the same choice of spatio-temporal embedding and computing directly the bound (\ref{erm_1}) for $\delta=0.025$, we obtain that the generalization gap is less than or equal to $2.99$ with at least $95\%$ probability. Using instead the spatio-temporal embedding defined in Example \ref{finite} for the STOU data set, we obtain that the generalization gap computed with (\ref{erm_2}) is less than $0.28$ versus a $0.98$ obtained from (\ref{erm_1}). 
	
	For the data set observed from the MSTOU process, we now compute the bound (\ref{erm_2}) by employing the selection rule in Remark \ref{selection2}. We then obtain a spatio-temporal embedding $\bb{S_m^2}$ for $a_t=1170$ and $m=17$ and a generalization error less or equal to $27.51$ with probability at least $95\%$. With the same choice of spatio-temporal embedding and computing directly the bound (\ref{erm_1}) for $\delta=0.025$, we obtain a generalization gap less than or equal to $5.66$ with at least $95\%$ probability instead. Using the spatio-temporal embedding defined in Example \ref{finite} for the MSTOU data set, we obtain that the generalization gap computed with (\ref{the_best}) is equal to $18.97$ versus a $9.40$ obtained from (\ref{erm_1}).
	
	From this simple example, we can notice how, for the temporal and spatial short- range data set, the bound obtained in (\ref{erm_2}) is tighter than the one presented in (\ref{erm_1}). In the temporal and spatial long-range case,  it seems however that with all spatio-temporal embeddings so far considered we can just obtain a vacuous bound. However, there is an aspect of the bound (\ref{erm_2}) that we have not truly used so far in our evaluations. All the results above hold under the choice of the accuracy level $\epsilon=1$. If we choose $\epsilon=1000$ (which goes outside the range of validity of the bound (\ref{erm_1})) and compute again the bound in (\ref{erm_2}) for the choice of $a_t=1170$, we obtain a not-vacuous bound. In fact, under this choice of the accuracy level, we obtain a generalization error less or equal than $838.84$ with probability at least of $95\%$. Therefore, for different reasons that in the case of temporal and spatial short-range data, the bound obtained in \ref{erm_2} is the one to employ when analyzing long-range data.
\end{Example}

From the description of possible selection rules for the parameter $a_t$ and the Examples \ref{finite} and \ref{finite2}, different questions may arise. The first regards the vacuousness of the bounds (\ref{PAC1}) and (\ref{the_best}) observed in some of our examples when employing temporal and spatial long-range data. Secondly, it is important to question the dependence of the bound from the constant $Lip(h)$, and the accuracy level parameter $\epsilon$. The three remarks below tackle these issues and present interesting future research directions of our work.

\begin{Remark}[Spatio-temporal embeddings for temporal and spatial long-range dependent data sets]
	In general, the bounds (\ref{PAC1}) and (\ref{the_best}) depend non-linearly on the set of parameters listed in Table \ref{para}, and assessing their magnitude in the temporal and spatial long-range dependence framework for MMAF is still an open problem. 
	
	It is important to highlight that the selection of the spatio-temporal embedding depends on the discretization step in time $h_t$. Such constant is related, in practical applications, to the frequency of the observed data and has a significant impact on the tightness of the bound. From the selection rules in Remarks \ref{selection1} and \ref{selection2}, we can see how the parameter value $a_t$ decreases when $h_t >1$. 
	
	Moreover, the bound (\ref{the_best}) holds for $\epsilon >0$ and this gives us better generalization performance as long as we consider an $\epsilon \geq 3$ (which goes outside the range of validity of the bound (\ref{PAC1})), see Example \ref{finite2}.
\end{Remark}

\begin{Remark}[Are the fixed-time bounds (\ref{PAC1}) and (\ref{the_best}) depending on the value of the $Lip(h)$?]
	There is no explicit dependence on the $Lip(h)$ in the bounds (\ref{PAC1}) and (\ref{the_best}). Such coefficient appears if we employ Proposition \ref{mmaf_absolute} to estimate the $\theta$-coefficients in the bound. However, we do not have a formal proof of the tightness of this estimation.
	
	When working with \emph{deep neural network predictors} computing their Lipschitz constants is a complex numerical task; see \cite{DeepLipschitz2} and \cite{DeepLipschitz}. In this framework, we could use the bound (\ref{bound3}) for the $\theta$-coefficients of the process $\bb{L}^{\epsilon}$ and then work with a numerical approximation for the $Lip(h)$. However, future research should focus on determining tight estimates for the $\theta$-coefficients and assessing their dependence on the $Lip(h)$ in detail. Such analysis could greatly help the empirical computation  of PAC Bayesian bounds for (random) deep learning architectures and extend the range of applicability of MMAF-guided learning in practical applications. 
\end{Remark}

\begin{Remark}[The importance of the accuracy level $\epsilon$]
	\label{accuracy}
	For the time being, the proofs of Theorems \ref{prop_PAC1} and \ref{new_fixedtime} work just if $\bb{L}^{\epsilon}$ is bounded. Such an assumption allows us to apply the projective-type representation of the $\theta$-coefficients discussed in Remark \ref{mix_rule1}. This, in turn, allows us to be capable of \emph{guiding} a randomized estimator, i.e., choosing the parameters defining the spatio-temporal embedding by controlling the magnitude of the $\theta(1)$-coefficient of the process $\bb{L}^{\epsilon}$. It is important to highlight that as observed at the end of Example \ref{finite2} for a temporal and spatial long-range data set, or in the sensitivity analysis conducted for a temporal and spatial short range data set in Section \ref{sec5.1}, the performances of the employed randomized estimators improve when $\epsilon \geq 3$. 
\end{Remark}
In the next theorem, we give an alternative fixed-time bound for the randomized Gibbs estimator, obtained using Theorem (\ref{new_fixedtime}). In this case, the estimator does not minimize anymore the right hand side of the inequality. 

\begin{Theorem}[PAC Bayesian bound for the randomized Gibbs estimator- \emph{Type II}]
	Let $\epsilon>0$ and Assumption \ref{ass1} hold. If $m \geq 1$, $\eta=\frac{1}{\sqrt{m}}$, $\pi$ is a distribution on $\mathcal{H}$ such that $\pi [\theta(1)] < \infty$, $\bar{\rho}$ is the randomized Gibbs estimator defined in (\ref{Gibbs}), and $\delta \in (0,1)$  
	\begin{align}
		\mathbb{P} \Big\{\bar{\rho}[R^{\epsilon}(h)] \leq \inf_{\hat{\rho}}\Big( \hat{\rho}[R^{\epsilon}(h)] + &\Big(KL(\hat{\rho},\pi)+\log \Big(\frac{1}{\delta}\Big)\Big)\frac{2}{\sqrt{m}} \Big) +\frac{ \epsilon^2}{\sqrt{m}} \nonumber\\
		&+ 2\Big( \epsilon \pi\Big[\frac{\theta(1)}{\delta} \Big] (D_{\phi_2-1}(\bar{\rho},\pi)+1) \Big)^{\frac{1}{2}}  \Big \} \geq 1-4\delta, \label{PAC_G2}
	\end{align}
	where $\theta(1)$ is a $\theta$-coefficient of the process $\bb{L}^{\epsilon}$.
\end{Theorem}

Several results on fixed-time  PAC Bayesian bounds in a dependent framework can be found in the batch setting, but only for time series models. We review them in the remark below.

\begin{Remark}[Are there PAC Bayesian bounds for dependent data with faster convergence rates?] 
	\label{rate_literature}
	In \cite{Hostile}, the authors determine an oracle inequality with a rate of $\mathcal{O}(m^{-\frac{1}{2}})$ under the assumption that $((\bb{X}_i,\bb{Y}_i)^{\top})_{i \in \Z}$ is generated by a stationary and $\alpha$-mixing process, see \cite{Bradley} for a detailed explanation of the properties of this dependence notion,  with coefficients $(\alpha_j)_{j \in \Z}$ such that $\sum_{j \in \Z} \alpha_j < \infty$. Such bound employs the chi-squared divergence and holds for unbounded losses. It is important to highlight that the randomized estimator obtained by minimization of the PAC Bayesian bound is not a Gibbs estimator in this framework. An explicit bound for linear predictors can be found in their \cite[Corollary 2]{Hostile}. This result holds under the assumption that $\pi[\|\beta\|^6] < \infty$. 
	
	In \cite{AW12}, the authors prove oracle inequalities for a Gibbs estimator and data generated by a stationary and bounded $\theta_{1,\infty}$-weakly dependent process-- such dependence notion extends the concept of $\phi$-mixing discussed in \cite{Rio96}-- or a causal Bernoulli shift process. Models with bounded $\theta_{1,\infty}$-weak coefficients are causal Bernoulli shifts with bounded innovations, uniform $\phi$-mixing sequences, and dynamical systems; see \cite{AW12} for more details. The oracle inequality is here obtained for an absolute loss function and has a rate of $\mathcal{O}(m^{-\frac{1}{2}})$. An extension of this work for Lipschitz loss functions under $\phi$-mixing \cite{ibra} can be found in \cite{fast}. Here, the authors show an oracle inequality for a Gibbs estimator with the optimal rate $\mathcal{O}(m^{-1})$. This rate is considered optimal in the i.i.d literature, and for a squared loss function, \cite{optimal}.  
	
	Interesting results in the literature of PAC Bayesian bounds for heavy-tailed data (albeit, identically distributed) can also be found in \cite{Expo} and \cite{H19}. 
\end{Remark}

MMAF-guided learning has the potential to be extended to general $\theta$-lex weakly dependent models because of the results in Proposition \ref{heredithary}. Another possible extension of the methodology is related to using bounded and locally Lipschitz losses; see Remark \ref{loclip}. 	
Considering unbounded losses is beyond the scope of the present paper and the content of future research treated in \cite{CS24}.

\section{Ensemble Forecasts using MMAF-guided learning}
\label{sec5}

\subsection{Practical implementation and casual forecast}
\label{sec5.0}
The knowledge of all the parameters in Table \ref{para} allows us to have a precise definition of the spatio-temporal embedding defined in Section \ref{sec2.1}. So far, we have only discussed the selection of the parameter $a_t$. The parameter $p_t$ is an \emph{hyperparameter} in our learning methodology. By establishing a finite grid of values $I \subset \N$ for the parameter $p_t$, we could introduce in (\ref{PAC1}) and (\ref{the_best}) an explicit dependence on its possible values. So doing, we could approach the general question of minimizing the right-hand-side of the bound (also known as, \emph{exact minimization}) in the function of every possible randomized estimator $\hat{\rho}$ and the grid of values $I$, see Section 2.1 in \cite{User} for exemplary calculations. Similarly, we could approach the selection of the parameter $\epsilon$ and obtain a methodology for the \emph{selection of  all the hyperparameters} involved in MMAF-guided learning. Such an issue is outside the scope of the present paper. However, we provide in our numerical experiments in Section \ref{sec5.1} a sensitivity analysis about the hyperparameters $\epsilon$ and $p_t$.

If we assume that our data are generated by an STOU or an MSTOU process, i.e., we are assuming that the data admits exponential or power-decaying $\theta$-lex coefficients, as discussed in Section \ref{sec_stou} and \ref{sec_mstou}, there are several methodologies available for the estimation of the parameters $c$ and $\lambda$. There are also other feasible model set-ups for time series models, that is when $c=0$. In this framework, several estimation methodologies can be employed to estimate the model's parameters of an MMAF, see \ref{timeseries_app}. 


\begin{figure}[h!]
	\centering \scalebox{1}
	{\includegraphics[width=.35\textwidth]{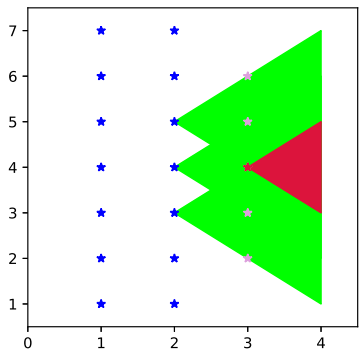}}
	\caption{The x- and y-axes represent the time and spatial dimension, respectively. We picture the last $3$ frames of a data set with spatial dimension $d=1$ where the blue stars represent the pixels used in the definition of the training data set, and the violet stars represent the space-time points where it is possible to provide forecasts with MMAF-guided learning for $p_t=c=h_t=1$. Note that the forecast in the time-spatial position $(4,3)$ lies in the intersection (red area) of the future lightcones $A_2(5)^+$, $A_2(4)^+$ and $A_2(3)^+$ as defined in (\ref{future_lightcone}) and represented with green cones.}\label{causal}
\end{figure}

\begin{Remark}
	When using the spatio-temporal embeddings described in Remark \ref{selection1} and \ref{selection2}, similarly to the kriging literature, we need an inference step before being capable of delivering one-time ahead ensemble forecasts. In this literature, it is often assumed that the estimated parameters used in the calculation of kriging weights and kriging variances are the \emph{true one}, see \cite[Chapter 3]{C93} and \cite[Chapter 6]{CW11} for a discussion on the range of applicability of such estimates. We implicitly make the same assumptions if, for example, we use estimated values for the parameters $\lambda$ and $c$.
	
	It remains an interesting open problem to understand the interplay of the estimates' bias of the parameters involved in the computation of the PAC Bayesian bounds (\ref{PAC1}) and (\ref{the_best}). One of the biggest problems of this analysis relies upon disentangling the effect of the bias of the constant $c$ introduced in the pre-processing step, which changes the length of the input-features vector $X_i$. 
\end{Remark}

We detail now how to \emph{guide} the design of a randomized Gibbs estimation $\bar{\rho}$ and make one-time ahead ensemble forecasts. Differently from the notations so far employed, we indicate the training data set by $S_m^{x^*}$ to remark the dependence of the training data set on the pixel position $x^*$ where we perform our forecasts. In general, the learning methodology applies to any pixel  $x^*$ for which $\mathcal{I}(t_0+ia,x^*) \subseteq \mathbb{T} \times \mathbb{L} $ for all $i=1,\ldots,N$.

\textbf{Application Steps:}
\begin{itemize}
	\item[(i)] We observe a raster data cube whose spatio-temporal index set is described by $N$ frames and we estimate the parameters $\lambda$ and $c$ (and $\bar{\alpha}$ if necessary).
	\item[(ii)]  
	We fix a pixel position $x^*$ and choose a value for the accuracy level $\epsilon$ and the hyperparameter $p_t$. We then select the parameter $a_t$ as suggested in  Remark \ref{selection1} or \ref{selection2} and determine the spatio-temporal embedding.
	\item[(iii)] We determine the training data set $S_m^{x^*}$ using all available $N$ frames, which correspond to a specific realization $\omega \in \Omega$, i.e., $\bb{S}_m^{x^*}(\omega)=S_m^{x^*}$.
	\item[(iv)] We then draw $\beta$ from the distribution $\bar{\rho}(\cdot,\omega)$, following the definition of the randomized Gibbs estimator in (\ref{Gibbs}). 
	\item[(v)] We perform a so-called \emph{ensemble forecast} by repeating point (iv) several times. 
\end{itemize}

A one-time ahead forecast corresponds to the space-time point $((t_0+Na)+h_t,x^*)$ and it is given by $\langle L_p^-((t_0+Na)+h_t), \beta \rangle$, where $\beta$ is a draw from the randomized Gibbs distribution. Therefore, we can make a forecast in a future time point $t=(t_0+Na)+h_t$ as long as the set $\mathcal{I}((t_0+Na)+h_t,x^*)$, as defined in (\ref{index}), has cardinality $a(p,c)$. 

For each $\beta$, the forecast of the field we obtain in the space-time point $((t_0+Na)+h_t,x^*)$ lies in the intersections of the future light cones of the space-time points belonging to $L^-_p((t_0+Na)+h_t,x^*)$. In Figure \ref{causal}, we give an example of a one-time ahead forecasts performed for an MMAF for $d=1$. As we can see, MMAF-guided learning enables us to make forecasts in space-time points that are plausible (under the causality concept induced by the ambit sets, see Section \ref{sec1.2}) starting from the set of inputs we observe. 

\subsection{Linear predictors: an example with simulated data}
\label{sec5.1}

We work in the hypothesis space $\mathcal{H}^{\prime}=\{h_{\beta}(X)= \beta_0+ \beta_1^T X,\, \textrm{for}\,\, \beta:=(\beta_0,\beta_1)^{\top} \in  B\}$, where $B \subset \R^{(a(p,c)+1)}$. We use simulated observations from an STOU, i.e., a temporal and spatial short-range dependent field with exponentially decaying $\theta$-lex weakly dependent coefficients. We simulate four data sets $(Z_t(x))_{(t,x)\in \mathbb{T}\times \mathbb{L}}$ from a zero mean STOU process by employing the diamond grid algorithm introduced in \cite{STOU} for $d=1$. The time and spatial discretization steps are chosen as $h_t=h_s=0.05$ on the spatio-temporal interval $[0,100]\times[0,10]$. Therefore $\mathbb{T}=\{h_t,\ldots,2,000 h_t\}$ and $\mathbb{L}=\{0,\ldots,200\}$. We use the unidimensional frames related to the time indices $\mathbb{T}^{train} =\{0,h_t,\ldots,1,999 h_t\}$ to determine the training data sets and the one corresponding to $\mathbb{T}^{test}=\{2000 h_t\}$ as a test set.
We choose as distribution for the L\'evy seed $\Lambda^{\prime}$ a normal distribution with mean $\mu=0$ and standard deviation $\sigma=0.5$, and an $NIG(\alpha,\beta,\mu,\delta)$ distribution with $\alpha=5,\beta=0, \delta=0.2$ and $\mu=0$; see Exercises \ref{gau} and \ref{nig}. We use the latter distribution to test the behavior of MMAF-guided learning for different sets of heavy-tailed data. We  generate data with different seeds for the L\'evy basis realizations. Moreover, the speed of information propagation $c$ is equal to one for all generated data sets. Finally, we choose different mean reverting parameters, namely $A=1$ or $4$. Therefore,  $\lambda$ is equal to $\frac{1}{2}$ if $A=1$ or $2$ if $A=4$. We call these data sets GAU10, GAU1A4, NIG1A4, and NIG10. For the NIG L\'evy seed described above, we generate two further spatio-temporal data sets called $(Z^i_t(x))_{(t,x)\in \mathbb{T}_i\times \mathbb{L}}$ for $i=1,2$ on the spatio-temporal interval $[0,1000]\times[0,10]$. We choose $h_t=h_s=0.05$, $\mathbb{L}=\{0,\ldots,200\}$ and time indices $\mathbb{T}_1=\{18,000 h_t,\ldots,20,000 h_t\}$ and $\mathbb{T}_2=\{0, h_t,\ldots,20,000 h_t\}$. Such data sets are called NIG1, in the following. A summary scheme of the data's characteristics is given in Table \ref{data}. For the NIG1 data sets, we use the unidimensional frames corresponding to $\mathbb{T}_1^{train} =\{18,000 h_t,\ldots,19,999 h_t\}$ and  $\mathbb{T}_2^{train} =\{0, h_t,\ldots,19,999 h_t\}$ to determine the training data sets, respectively, and $\mathbb{T}_1^{test} =\{20,000 h_t\}$ and $\mathbb{T}_2^{test} =\{20,000 h_t\}$ as test sets. For all the generated data sets, we perform a one-time ahead ensemble forecast for the pixels corresponding to $\mathbb{L}^{\prime}=\{1,\ldots,199\}$.

We follow the steps detailed in Section \ref{sec5.0} to use MMAF-guided learning in practice. We start by estimating the parameters $c$ and $\lambda$. We use the  estimators $(\ref{est_stou})$ presented in Section \ref{sec_stou} and the plug-in estimator $(\ref{plugin})$. Table \ref{result} gives the results for each data set. We then use such estimates and the frames corresponding to $\mathbb{L}^{\prime} \times \mathbb{T}^{train}$, $\mathbb{L}^{\prime} \times \mathbb{T}_1^{train}$, or $\mathbb{L}^{\prime} \times \mathbb{T}_2^{train}$ in the selection of the parameter $a_t$ by following the rules in Remark \ref{selection1} for $k=1$, and Remark  \ref{selection2}.
We then analyze the performance of the different randomized Gibbs estimators obtained from a particular choice of the parameter $a_t$ when the reference distribution is assumed to be multivariate standard Gaussian. We do not give an empirical evaluation of the right-hand side of the bounds (\ref{PAC_G1}) and (\ref{PAC_G2}) and base our assessment of the performance of the different estimators on how narrow their inter-quartile range on a $50$ member ensemble forecast is. Note that each forecast we make has a casual interpretation as described in Figure \ref{causal}.

\begin{table}[h!]
	\small
	\centering
	\begin{tabular}{|c|c|c|c|}
		\hline
		\textbf{Data Set} & \textbf{Mean Reverting Parameter} & \textbf{L\'evy seed} & \textbf{Random generator seed}\\
		\hline
		GAU1A4 & $A=4$ &  Gaussian & $1$\\
		\hline
		GAU10 & $A=1$ & Gaussian & $10$\\
		\hline
		NIG1  & $A=1$ & NIG & 1 \\
		\hline
		NIG1A4  & $A=4$ & NIG & $1$ \\
		\hline
		NIG10 & $A=1$ & NIG & $10$ \\
		\hline
	\end{tabular}
	\caption{Overview on simulated data sets with $c=1$ and spatial dimension $d=1$.}
	\label{data}
\end{table}

We start by conducting two different experiments to showcase the performance of our methodology for $\epsilon=2.99$ and $p_t=1$. We use as baseline model a linear model where the estimation of the parameter vector $\beta$ is performed using the empirical risk minimizer defined in (\ref{ols}).

\begin{table}[h!]
	\small
	\centering
	\begin{tabular}{|c|c|c|c|c|}
		\hline
		\textbf{Data Set} & $\bb{A^*}$ & $\bb{c^*}$ & $\bb{\lambda^*}$ & Frames Used\\
		\hline
		GAU1A4 & $3.9684$ &  $0.9958$ & $1.9715$ & $\mathbb{T} \times \mathbb{L}$\\
		\hline
		GAU10 & $0.8429$ & $0.9978$ & $0.4196$ & $\mathbb{T} \times \mathbb{L}$\\
		\hline
		NIG1 & $1.0186$ & $1.0018$ &  $0.5111$ & $\mathbb{T}_1 \times \mathbb{L}$\\
		\hline
		NIG1 & $1.0186$ &  $1.0019$ & $0.5112$ &  $\mathbb{T}_2 \times \mathbb{L}$\\
		\hline
		NIG1A4  & $4.0308$ & $1.0076$ & $2.0461$ & $\mathbb{T} \times \mathbb{L}$\\
		\hline
		NIG10 & $0.9728$ & $1.0021$ & $0.4884$ & $\mathbb{T} \times \mathbb{L}$ \\
		\hline
	\end{tabular}
	\caption{Estimations of parameters $A$, $c$ and $\lambda$.}
	\label{result}
\end{table}

In the first experiment, we use the data sets GAU1A4, GAU10, NIG1A4, NIG10 and work with the training data sets $S^{x^*}_m$ described by the Tables \ref{akmi} and \ref{akmii}.

\begin{table}[h!]
	\small
	\centering
	\begin{tabular}{|c|c|c|c|c|c|c|c|}	
		\hline
		\multicolumn{8}{|c|}{Selection of parameters as in Remark \ref{selection1}}\\
		\hline
		\multicolumn{2}{|c|}{GAU1A4} &
		\multicolumn{2}{c|}{GAU10} &
		\multicolumn{2}{c|}{NIG1A4} &
		\multicolumn{2}{c|}{NIG10}\\
		\hline
		$a_t$ & 124 & $a_t$ & 346 &$a_t$ & 121 & $a_t$ & 313\\
		m & 16 & m & 6 & m & 17 & m & 7\\
		k & 1 & k & 1  & k & 1 & k & 1 \\
		\hline
	\end{tabular}
	\caption{Parameters defining the training data sets $S^{x^*}_m$ used in the first experiment for each pixel.}
	\label{akmi}
\end{table}

\begin{table}[h!]
	\small
	\centering
	\begin{tabular}{|c|c|c|c|c|c|c|c|}	
		\hline
		\multicolumn{8}{|c|}{Selection of parameters as in Remark \ref{selection2}}\\
		\hline
		\multicolumn{2}{|c|}{GAU1A4} &
		\multicolumn{2}{c|}{GAU10} &
		\multicolumn{2}{c|}{NIG1A4} &
		\multicolumn{2}{c|}{NIG10} \\
		\hline
		$a_t$ & 47 & $a_t$ & 156 &$a_t$ & 45 & $a_t$ & 139\\
		m & 41 & m & 13 & m & 43 & m & 15\\
		k & 1 & k & 1  & k & 1 & k & 1 \\
		\hline
	\end{tabular}
	\caption{Parameters defining the training data sets $S^{x^*}_m$ used in the first experiment for each pixel.}
	\label{akmii}
\end{table}

\begin{figure}[h!]
	\centering
	\subfigure[GAU1A4]{\includegraphics[width=.4\textwidth]{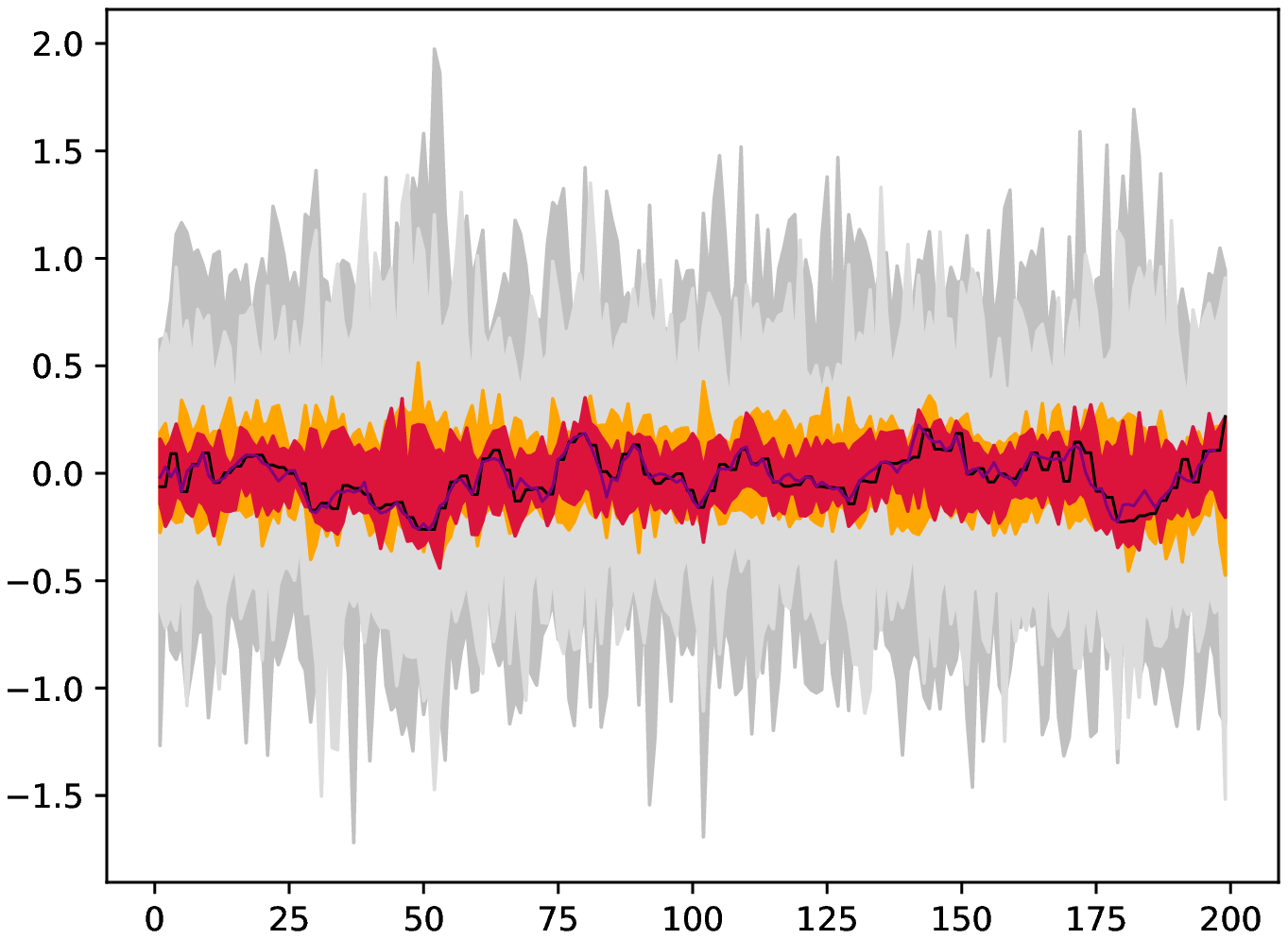}} \hspace{.02\textwidth}
	\subfigure[NIG1A4]{\includegraphics[width=0.4\textwidth]{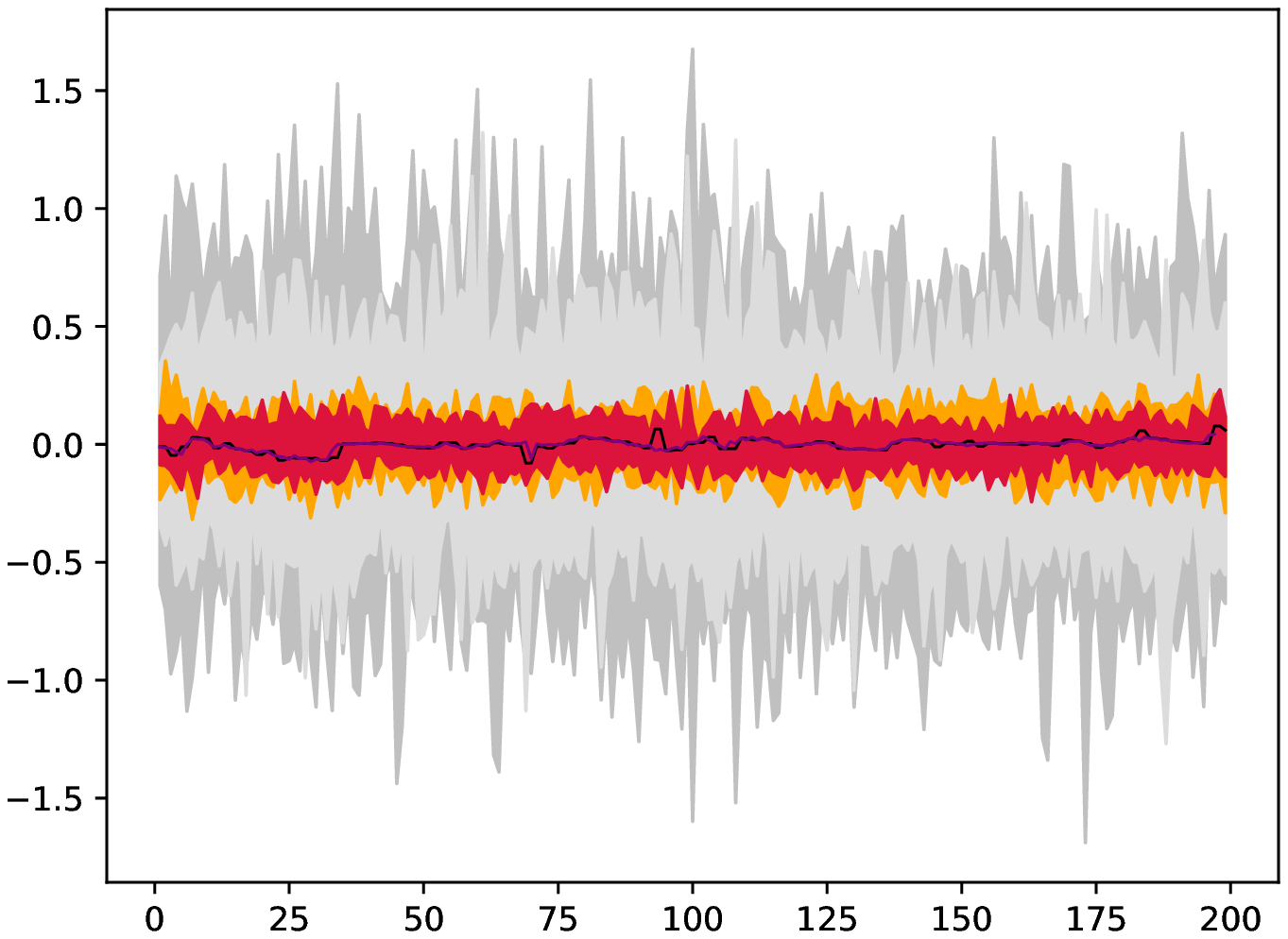   }}\hspace{.02\textwidth}
	\subfigure[GAU10]{\includegraphics[width=0.4\textwidth]{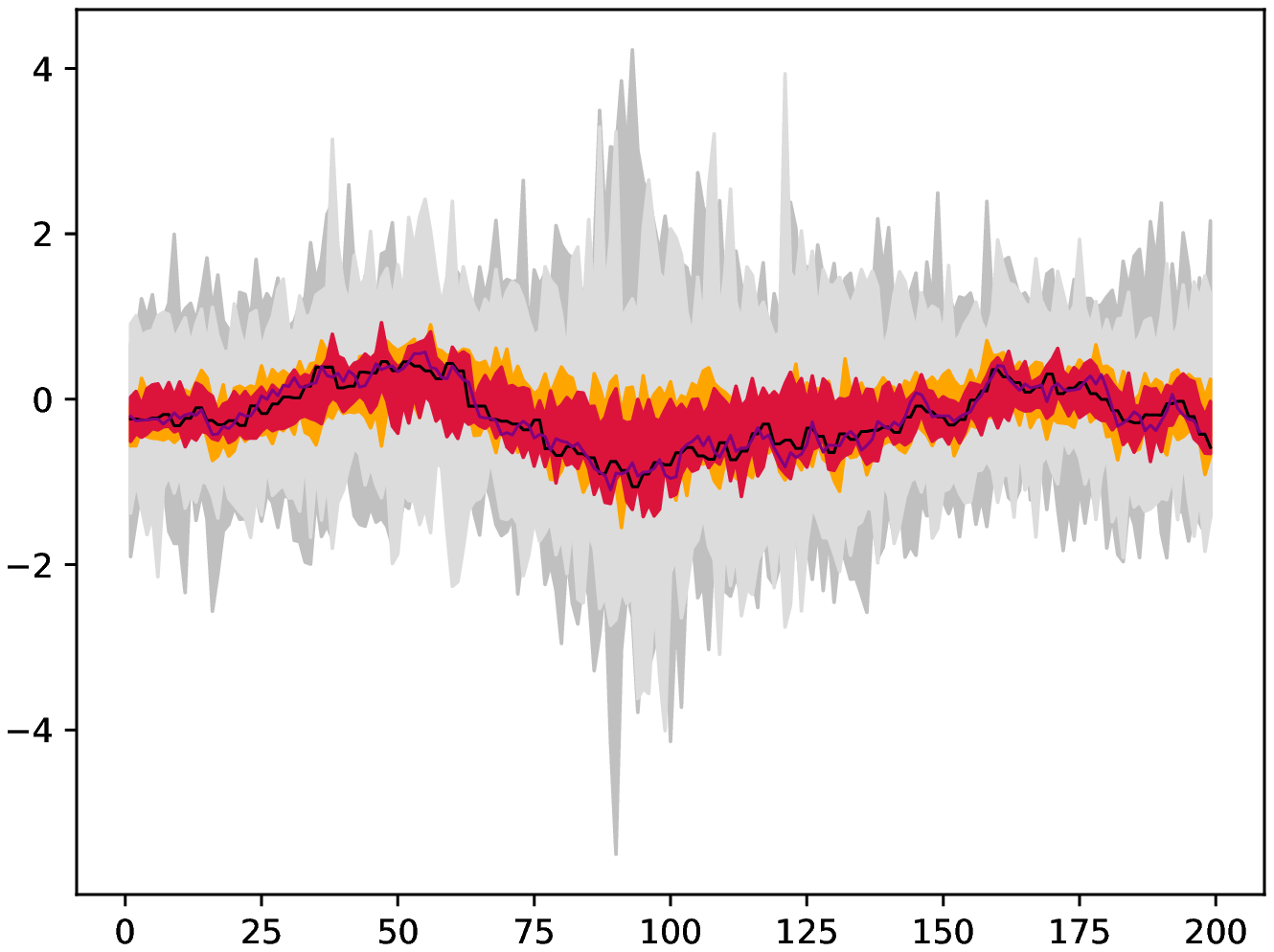}}\hspace{.02\textwidth}
	\subfigure[NIG10]{\includegraphics[width=0.4\textwidth]{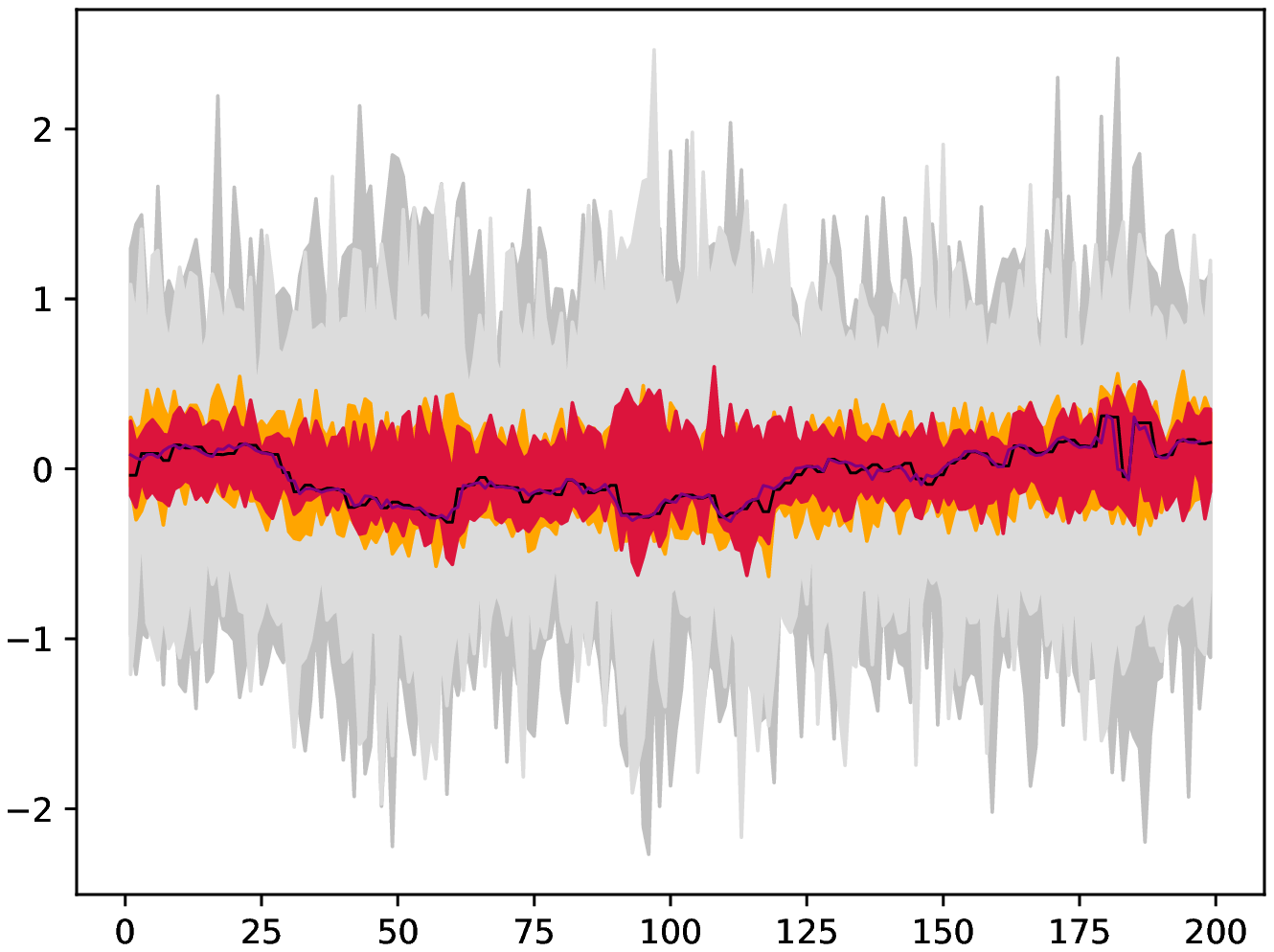}}\hspace{.02\textwidth}
	\caption{Min-max and inter-quartile range of a $50$-member ensemble forecast for the training data sets $S^{x^*}_m$ described in Tables \ref{akmi} and \ref{akmii}.
		Dark grey and orange color represent the ranges related to the use of Table \ref{akmi}, whereas the light grey and red color represent the ranges related to the use of Table \ref{akmii}. The test set is depicted with a black thick line, while the forecasts obtained using the baseline estimator (\ref{ols}) are represented with a violet thick line.  The training data set $S^{x^*}_m$ described in Table \ref{akmii} are used for computing the baseline estimates with respect to the GAU1A4, NIG1A4, GAU10, and NIG10 data sets in (a), (b), (c) and (d), respectively. The $x$-axis represents the pixels in $\mathbb{L}^{\prime}$, whereas the forecast values are along the $y$-axis.}
	\label{prediction}
\end{figure}

An acceptance-rejection algorithm with a Gaussian proposal determines a draw $\beta$ from the randomized Gibbs estimator. We show in Figure \ref{prediction} the min-max range of the two ensemble forecasts (dark grey for the parameters in Table \ref{akmi}, and light grey for the one in Table \ref{akmii}) as well as the inter-quartile ranges of a $50$-member ensemble forecast for each pixel in $\mathbb{L}^{\prime}$ compared with the test set. As the plots clearly show, using a randomized Gibbs estimator, we obtain an inter-quartile range that contains the test set for each spatial position $x^* \in \mathbb{L}^{\prime}$.
Moreover, the ranges seem to have a similar behavior independently of the test set and the L\'evy seeds. The randomized Gibbs estimators guided by the choices of parameters in Table \ref{akmii} have narrower inter-quartile ranges (in red) with respect to the ones guided by the choices of parameters in Table \ref{akmi} (in orange).

Let us define the average Relative  Mean Absolute Error (averRMAE) to compare our forecasts with the baseline model. 
Let $P=|\mathbb{L}^{\prime}|$, we define the average Relative Mean Absolute Error as
\[
\textit{aver}RMAE= \frac{1}{P} \sum_{i=1}^P \frac{  \lvert Z_{t}(x_i)-\hat{Z}_{t}(x_i)\rvert}{\lvert Z_{t}(x_i)\rvert},
\]

where $\hat{Z}_t(x_i)$ is the one-time ahead forecast obtained with the linear model for all $i$. The observations in our simulated data sets have an order of magnitude (on average) of $10^{-3}$. Tables \ref{rmaei} and \ref{rmaeii} show that the empirical risk minimizer cannot capture any significant digit, as also seen in Figure \ref{prediction}. Our ensemble forecasts give, at least, an interval where the one-time ahead forecasts can lie.

\begin{table}[h]
	\small
	\centering
	\begin{tabular}{|l| c |c|c|c|c|}	
		\hline
		&
		GAU1A4 &
		GAU10 &	
		NIG1A4&
		NIG10 &
		NIG1\\
		\hline
		linear & $0.0304$ & $0.0176$ &  $0.0168$ &  $0.0127$ &  $0.0142$\\ 
		\hline
	\end{tabular}
	\caption{averRMAE for the baseline estimator following the selection rule in Remark \ref{selection1}. The NIG1 data set's averRMAE has been computed with respect to $S^{2,x^*}_m$ described in Table \ref{akmCompi}.}	
	\label{rmaei}
	\begin{tabular}{|l| c |c|c|c|c|}	
		\hline
		&
		GAU1A4 &
		GAU10 &	
		NIG1A4&
		NIG10 &
		NIG1\\
		\hline
		linear & $0.0303$  & $0.0176$ &  $0.0156$ &  $0.0128$ &  $0.0138$\\ 
		\hline
	\end{tabular}
	\caption{averRMAE for the baseline estimator following the selection rule in Remark \ref{selection2}. The NIG1 data set's averRMAE has been computed with respect to $S^{2,x^*}_m$ described in Table \ref{akmCompii}.}
	\label{rmaeii}	
\end{table}
In the second experiment, we analyze the performance of randomized Gibbs estimators for the data sets NIG1. For each pixel, we work with the data sets $S^{1,x^*}_m$ and $S^{2,x^*}_m$ described in Tables \ref{akmCompi} and \ref{akmCompii}, obtaining the ensemble forecasts in Figure \ref{NigComp}.

\begin{table}[h!]
	\small
	\centering
	\begin{tabular}{|p{.5cm}|p{.5cm}|p{.5cm}|p{.5cm}|}
		\hline
		\multicolumn{4}{|c|}{Selection of parameters}\\
		\multicolumn{4}{|c|}{ as in Remark \ref{selection1}}\\
		\hline
		\multicolumn{2}{|c|}{$S^{1,x^*}_m$} &
		\multicolumn{2}{c|}{$S^{2,x^*}_m$}\\
		\hline
		\centering $a_t$ &\centering 303 & \centering $a_t$ & 652\\
		\centering m &\centering 7 &\centering m & \;32\\
		\centering k &\centering 1 &\centering k & \;\;1 \\
		\hline
	\end{tabular}
	\caption{Parameters describing the training data sets used in the second experiment at each pixel $x^*$.}
	\label{akmCompi}
\end{table}
\begin{table}
	\small
	\centering
	\begin{tabular}{|p{.5cm}|p{.5cm}|p{.5cm}|p{.5cm}|}
		\hline
		\multicolumn{4}{|c|}{Selection of parameters}\\
		\multicolumn{4}{|c|}{ as in Remark \ref{selection2}}\\
		\hline
		\multicolumn{2}{|c|}{$S^{1,x^*}_m$} &
		\multicolumn{2}{c|}{$S^{2,x^*}_m$}\\
		\hline
		\centering $a_t$ & \centering 134 & \centering$a_t$ & 207\\
		\centering 	m & \centering 15 &\centering m & \;96\\
		\centering k & \centering 1 & k\centering & \;\;1 \\
		\hline
	\end{tabular}
	\caption{Parameters describing the training data sets used in the second experiment at each pixel $x^*$.}
	\label{akmCompii}
\end{table}

By comparing the inter-quartile range of the ensemble forecasts, we see that in both cases, the green range, representing the forecasts related to $S^{2,x^*}_m$ is contained in the purple range, which represents the inter-quartile range of the forecasts made using the data set $S^{1,x^*}_m$. Both of them include the test set. The amplitude of the inter-quartile range reduces when the number of observations in the data set increases. The forecasts of the baseline model are performed using the data set $S^{2,x^*}_m$ and have a high averRMAE as reported in Tables \ref{rmaei} and \ref{rmaeii}.

\begin{figure}[h!]
	\centering 
	\subfigure[NIG1's ensemble forecast based on Table \ref{akmCompi}]{\includegraphics[width=.4\textwidth]{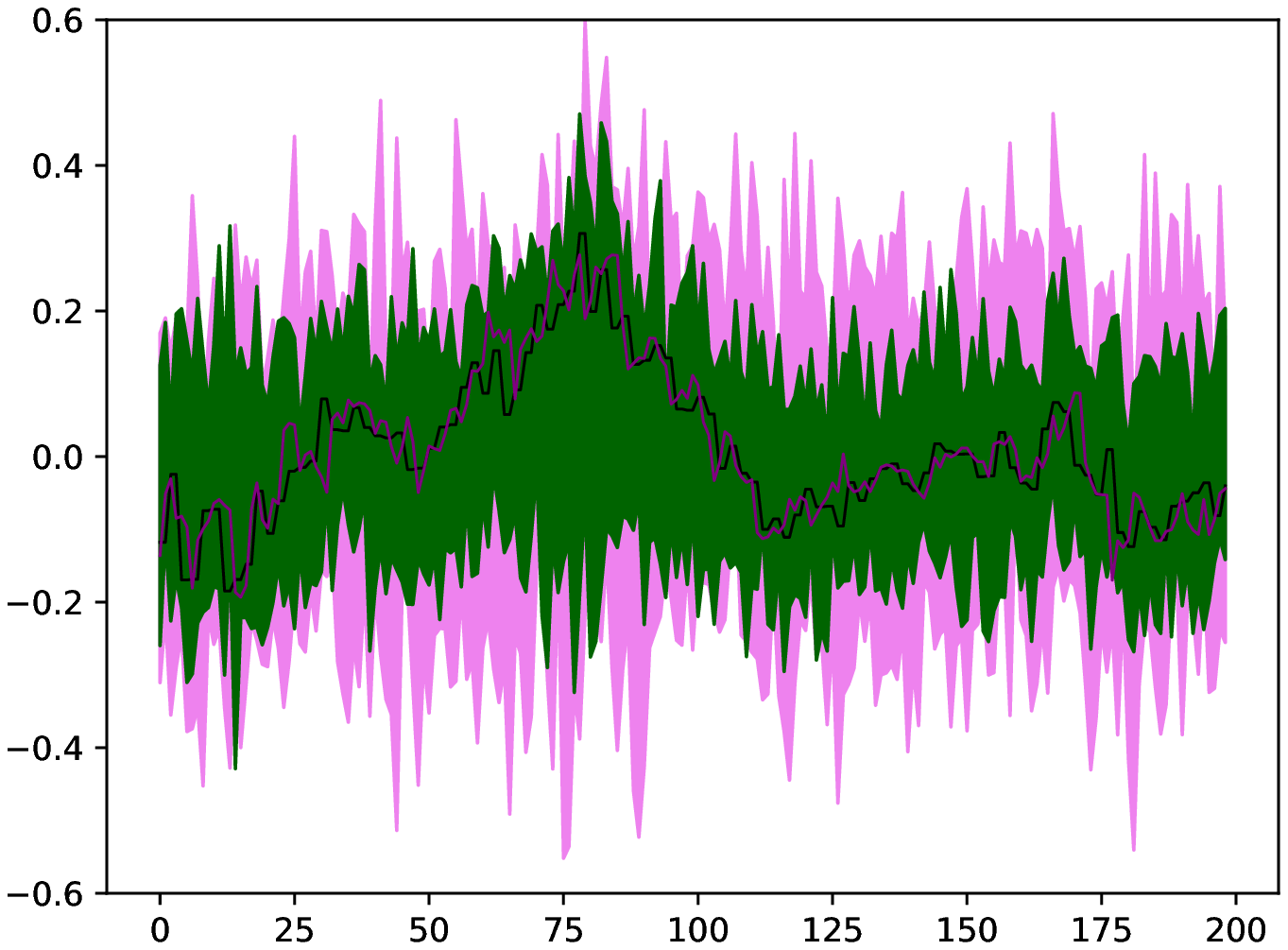}} \hspace{.02\textwidth}
	\subfigure[NIG1's ensembe forecast based on Table \ref{akmCompii}]{\includegraphics[width=.4\textwidth]{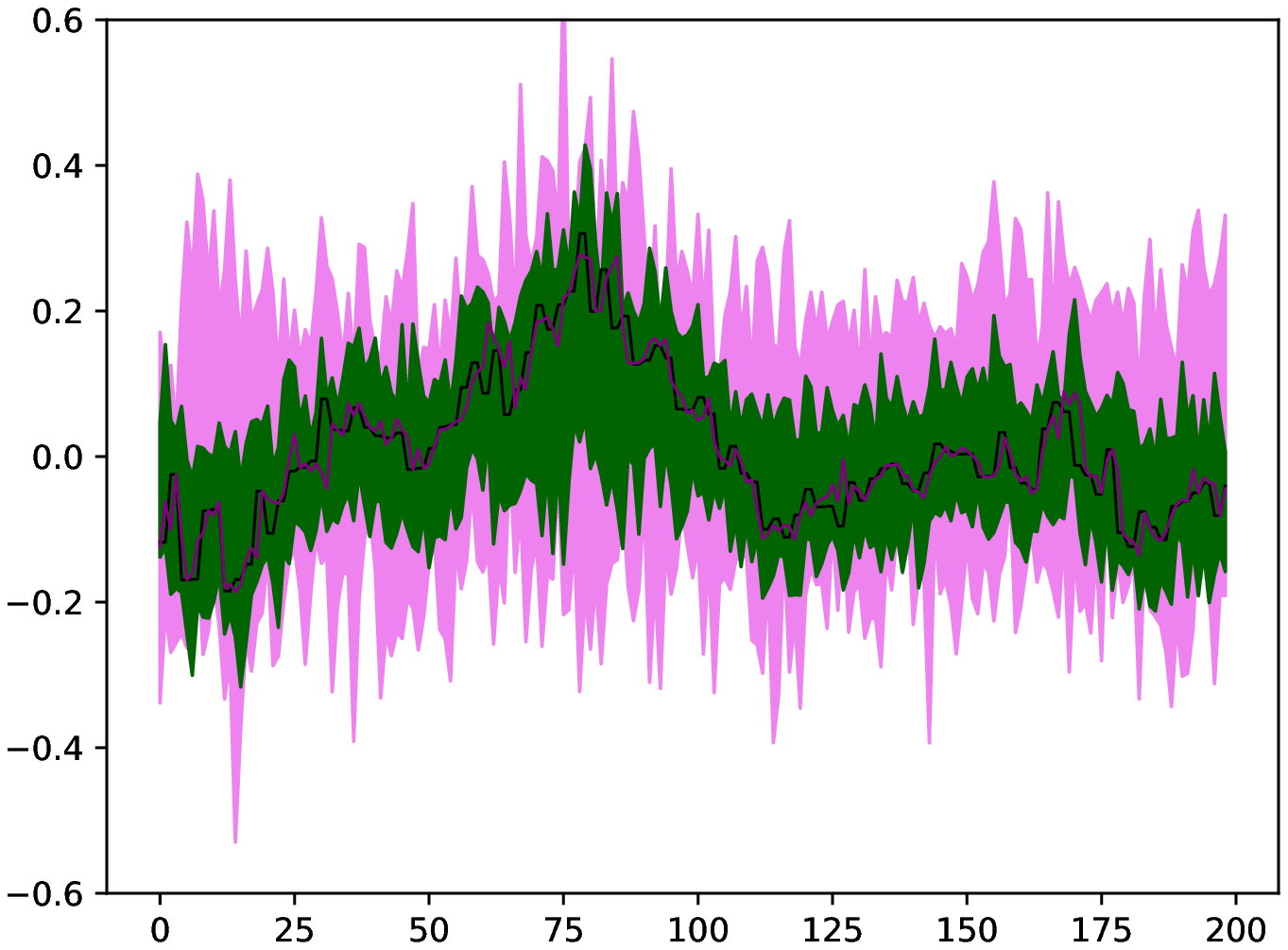}} \hspace{.02\textwidth}
	\caption{Inter-quartile range of a 50-member ensemble forecast using  $S^{1,x^*}_m$ (in violet) and $S^{2,x^*}_m$ (in green) as defined in Table \ref{akmCompi} (a) and Table \ref{akmCompii} (b). The test set and the baseline forecasts are depicted with a black and a violet thick line, respectively. The training data sets $S^{2,x^*}_m$ in Table \ref{akmCompi} and \ref{akmCompii} are respectively used as baseline estimates in (a) and (b). The $x$-axis represents the pixels in $\mathbb{L}^{\prime}$, whereas the forecast values are along the $y$-axis.}
	\label{NigComp}
\end{figure}

We want now to analyze our methodology's sensitivity to the hyperparameter $p_t$ when $a_t$ is selected following the rules in Remarks \ref{selection1} and \ref{selection2}, respectively. The selection of the parameter $a_t$ is proportional to the values of $p_t$. Therefore, the smaller we choose this parameter, the more examples we obtain in $S^{x^*}_m$. We work in this experiment with the data sets GAU1A4 and NIG1A4. However, we obtained the same conclusions for all the other data sets employed in our study. We generate ensemble forecasts (and their respective inter-quartile ranges) for $p_t=1,8,15$. All inter-quartile ranges in Figure \ref{forp} contain the test set, and a significant increase of $p_t$ has a negative impact on the inter-quartile range amplitude. For this reason, we have chosen the parameter $p_t=1$ in our previous experiments.

\begin{table}[h!]
	\small
	\centering
	
	\begin{tabular}{|c c | c c | c c|c c|c c|c c|}
		\hline
		\multicolumn{12}{|c|}{Selection of parameters as in Remark \ref{selection1}}\\	
		\hline
		\multicolumn{6}{|c|}{GAU1A4} &
		\multicolumn{6}{c|}{NIG1A4}\\
		\hline
		\multicolumn{1}{|c}{$p_t$} &
		\multicolumn{1}{c|}{1} &
		\multicolumn{1}{c}{$p_t$} &
		\multicolumn{1}{c|}{8} &
		\multicolumn{1}{c}{$p_t$} &
		\multicolumn{1}{c|}{15} &
		\multicolumn{1}{c}{$p_t$} &
		\multicolumn{1}{c|}{1} &
		\multicolumn{1}{c}{$p_t$}&
		\multicolumn{1}{c|}{8} &
		\multicolumn{1}{c}{$p_t$}&
		\multicolumn{1}{c|}{15} \\
		\hline$a_t$ & 124 & $a_t$ & 129& $a_t$ & 134 & $a_t$ & 121&$a_t$ & 127 & $a_t$ & 132\\
		m & 16 & m & 15 & m & 14 & m & 17 &	m & 15 & m & 14 \\
		k & 1 & k & 1 & k & 1 & k & 1 & k & 1 & k & 1 \\
		\hline
	\end{tabular}
	\caption{Parameters describing the training data sets $S^{x^*}_m$ used in the third experiment for each pixel.}
	\label{pCompi}
	\begin{tabular}{|c c | c c | c c|c c|c c|c c|}
		\hline
		\multicolumn{12}{|c|}{Selection of parameters as in Remark \ref{selection2}}\\	
		\hline
		\multicolumn{6}{|c|}{GAU1A4} &
		\multicolumn{6}{c|}{NIG1A4}\\
		\hline
		\multicolumn{1}{|c}{$p_t$} &
		\multicolumn{1}{c|}{1} &
		\multicolumn{1}{c}{$p_t$} &
		\multicolumn{1}{c|}{8} &
		\multicolumn{1}{c}{$p_t$} &
		\multicolumn{1}{c|}{15} &
		\multicolumn{1}{c}{$p_t$} &
		\multicolumn{1}{c|}{1} &
		\multicolumn{1}{c}{$p_t$}&
		\multicolumn{1}{c|}{8} &
		\multicolumn{1}{c}{$p_t$}&
		\multicolumn{1}{c|}{15} \\
		\hline
		$a_t$ & 47 & $a_t$ & 53 & $a_t$ & 58 & $a_t$ & 45 & $a_t$ & 52 & $a_t$ &  58\\
		m & 41 & m & 37 & m & 33 & m & 43 &	m & 39 & m & 35 \\
		k & 1 & k & 1 & k & 1 & k & 1 & k & 1 & k & 1 \\
		\hline
	\end{tabular}
	\caption{Parameters describing the training data sets $S^{x^*}_m$ used in the third experiment for each pixel.}
	\label{pmCompii}
\end{table}
Finally, we analyze the sensitivity of our methodology to the choice of the accuracy level $\epsilon$. We work in this experiment with the data set GAU10 and the related $S^{x^*}_m$ described in Tables \ref{akmi} and \ref{akmii}. However, we obtained the exact same conclusions for all the other data sets employed in our study. We choose $\epsilon=1,2,2.99$ for the $S^{x^*}_m$ in Table \ref{akmi} and $\epsilon=1,2.99,5$ for the training data set in Table \ref{akmii}. We remind that the bounds (\ref{PAC_G1}) and (\ref{PAC_G2}) work for $0<\epsilon <3 $ and $\epsilon>0$, respectively. Also, in this experiment, we plot the inter-quartile ranges obtained for the different randomized Gibbs estimators in Figure \ref{fore} and observe that the bigger the parameter $\epsilon$, the narrower the amplitude of the ranges. We choose in our experiment $\epsilon=2.99$ because it is the bigger $\epsilon$ for which both the bound (\ref{PAC_G1}) and (\ref{PAC_G2}) are defined.
\begin{figure}[h!]
	\centering
	\subfigure[GAU1A4's ensemble forecast based on Table \ref{pCompi}]{\includegraphics[width=.4\textwidth]{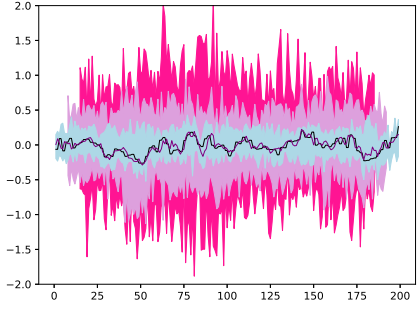}} \hspace{.02\textwidth}
	\subfigure[GAU1A4's ensemble forecast based on Table \ref{pmCompii}]{\includegraphics[width=.4\textwidth]{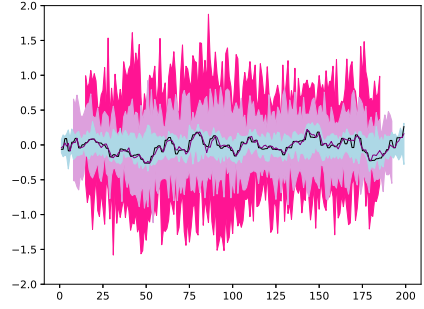}}\hspace{.02\textwidth}
	\subfigure[NIG1A4's ensemble forecast based on Table \ref{pCompi}
	]{\includegraphics[width=.4\textwidth]{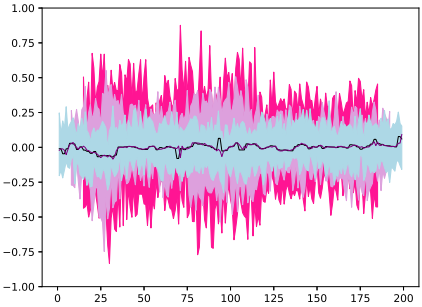}} \hspace{.02\textwidth}
	\subfigure[NIG1A4's ensemble forecast based on Table \ref{pmCompii}]{\includegraphics[width=.4\textwidth]{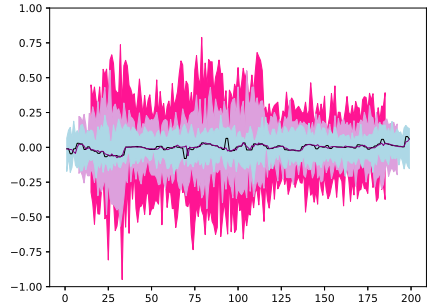}} \hspace{.02\textwidth}
	\caption{Inter-quantile range of a $50$-member ensemble forecast for $p_t=1$ (blue), $p_t=8$ (lilac), $p_t=15$ (magenta) using the training data sets described in Table \ref{pCompi} and Table \ref{pmCompii}. The test set and the baseline forecasts are depicted in black and violet thick lines, respectively. The training data sets $S^{x^*}_m$ in Table \ref{pmCompii} are respectively used as baseline estimates in (a), (b), (c) and (d). The $x$-axis represents the pixels in $\mathbb{L}^{\prime}$, whereas the forecast values are along the $y$-axis.}
	\label{forp}
\end{figure}

\begin{figure}[h!]
	\centering
	\subfigure[GAU10's ensemble forecast based on Table \ref{akmi}]{\includegraphics[width=.4\textwidth]{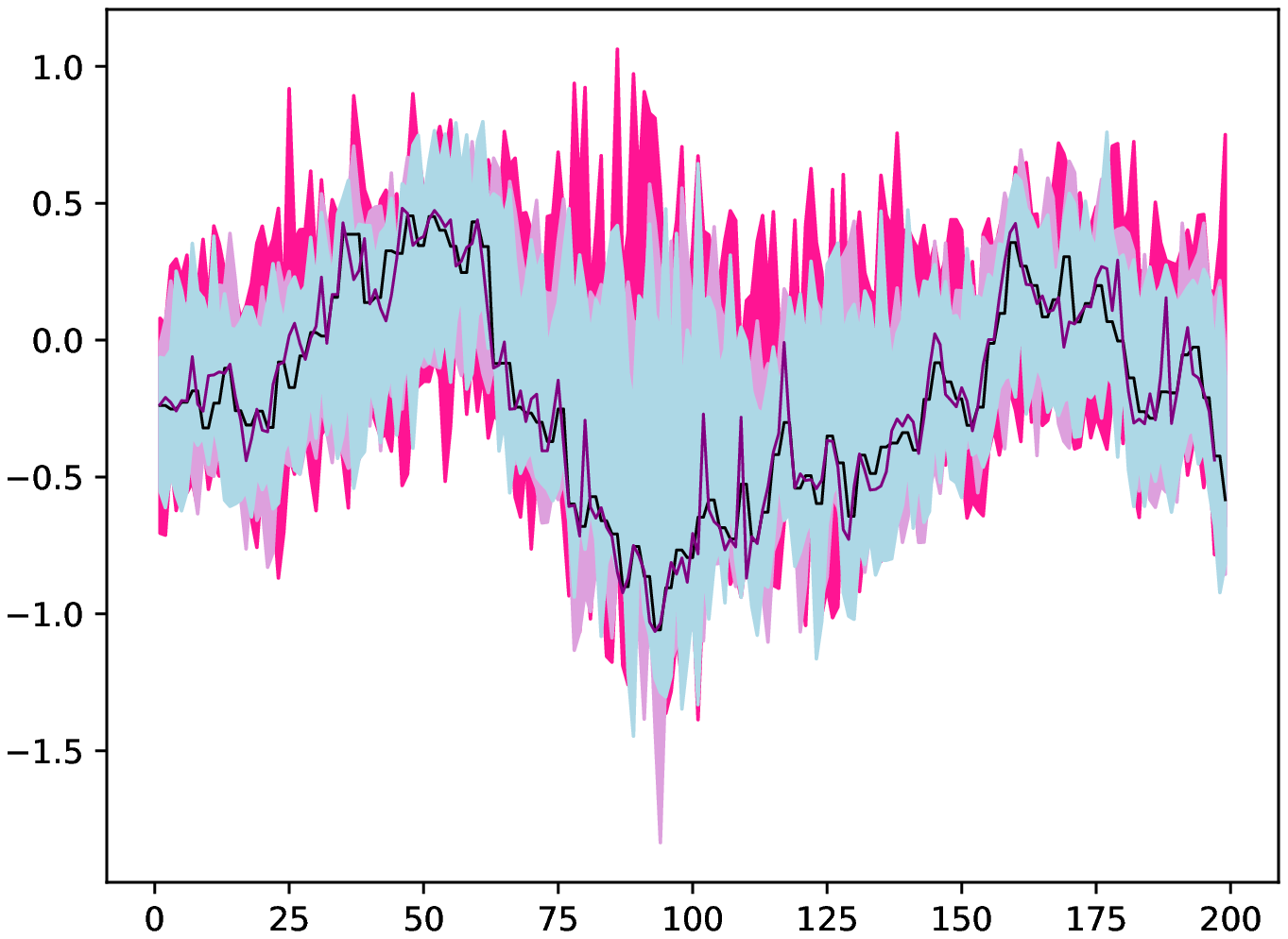}} \hspace{.02\textwidth}
	\subfigure[GAU10's ensemble forecast based on Table \ref{akmii}]{\includegraphics[width=.4\textwidth]{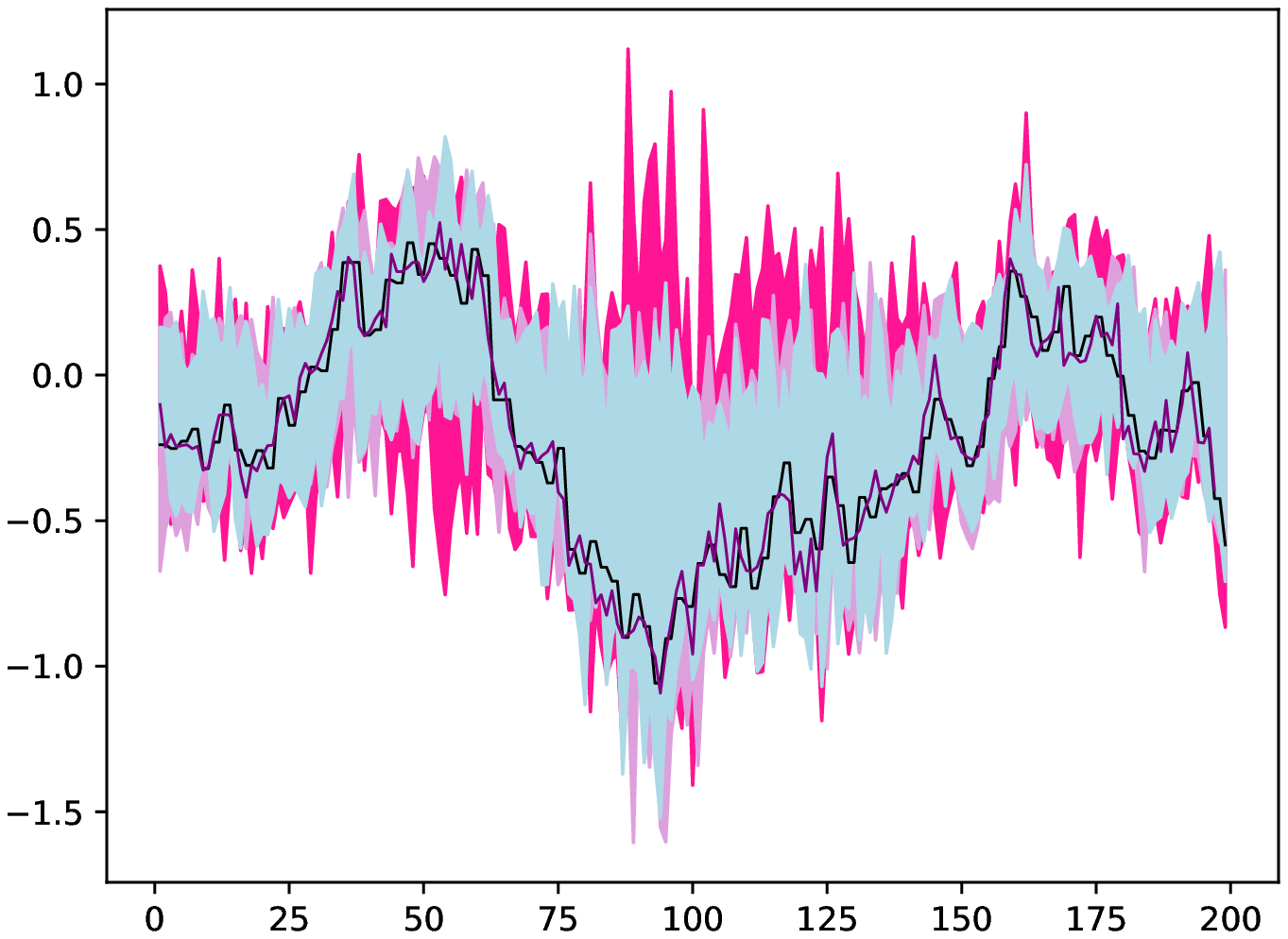}}\hspace{.02\textwidth}
	\caption{(a) Inter-quantile range of a $50$-member ensemble forecast for $\epsilon=1$ (magenta), $\epsilon=2$ (lilac), $\epsilon=2.99$ (blue) using the training data set described in Table \ref{akmi} for the GAU10. The test set and the baseline forecasts are depicted in black and violet thick lines, respectively. The training data sets $S^{x^*}_m$ in Table \ref{akmi} is used as baseline estimates for $\epsilon=2.99$.
		(b) Inter-quantile range of a $50$-member ensemble forecast for $\epsilon=1$ (magenta), $\epsilon=2.99$ (lilac), $\epsilon=5$ (blue) using the training data sets described in Table \ref{akmii} for the GAU10. The test set and the baseline forecasts are depicted in black and violet thick lines, respectively. The training data sets $S^{x^*}_m$ in Table \ref{akmii} is used as baseline estimates for $\epsilon=5$.
		The $x$-axis represents the pixels in $\mathbb{L}^{\prime}$, whereas the forecast values are along the $y$-axis.}
	\label{fore}
\end{figure}

\section{Conclusions}

We define a novel theory-guided machine learning methodology for spatio-temporal data called MMAF-guided learning, which works in the class of the Lipschitz functions, e.g., linear functions and several types of neural network modules. Our methodology applies to raster data cubes, and it works under the assumption that such data are generated by an influenced mixed moving average field (MMAF, in short) defined on a cone-shaped ambit set. Such random fields are strictly stationary, $\theta$-lex weakly dependent, and can be employed to model Gaussian and non-Gaussian distributed data. Moreover, they can be non-Markovian and admit non-separable covariance functions.

We show fixed-time and any-time PAC Bayesian bounds in this framework. All our bounds are determined for a bounded loss and depend on \emph{one single} $\theta$-lex coefficients of the underlying MMAF. In particular, our bounds hold for temporal and spatial short and long-range dependent data.

To enable one-time ahead ensemble forecasts, we need an estimate of the decay rate of the $\theta$-lex coefficients of the underlying MMAF. Such an estimation is feasible, for example, in the case of STOU and MSTOU processes. We can then define spatio-temporal embeddings such that they make the right-hand side of the fixed-time PAC Bayesian bounds, proven in the paper, not vacuous. The analyzed embeddings give us possible training data sets for learning \emph{randomized estimators} that have good generalization performance with a high probability. We can then determine one-time ahead ensemble forecasts.  

In conclusion, we test the learning procedure for a \emph{guided} randomized Gibbs estimator and a Gaussian reference distribution on the class of linear models. Our estimator has a convergence rate of $\mathcal{O}(m^{-\frac{1}{2}})$ . We simulate a set of six data sets from an STOU process with Gaussian and NIG L\'evy seed and determine (50 members) ensemble forecasts. We obtain that the inter-quartile ranges of our forecasts always contain the test set and are narrower when the number of observations in the training data set increases. Moreover, our forecasts have a causal interpretation induced by the ambit sets of the data-generating process known as Rubin's potential outcomes framework.

\appendix

\section{Appendix}
\label{appendix}

\subsection{Weak dependence notions for causal processes and (influenced) MMAF}
\label{sec2}

In this section, we discuss more in details the dependence notions called  $\theta$-weak dependence and $\theta$-lex weak dependence. 
The latter notion has been introduced in \cite[Definition 2.1]{CSS20} as an extension to the random field case of the notion of $\theta$-weak dependence satisfied by causal stochastic processes \cite{DD03}. This notion of dependence is presented in Definition \ref{thetadependent}. However, the notion of $\theta$-lex weak dependence given in Definition \ref{thetaweaklydependent} slightly differs from the one given in \cite[Definition 2.1]{CSS20} and represents an extension to the random field case of the $\theta$-weak dependence notion defined in \cite[Remark 2.1]{D07}. Note that the definitions of $\theta$-weak dependence in \cite{DD03} and \cite[Remark 2.1]{D07} differ because of the cardinality of the marginal distributions on which the function $G$ is computed, namely,  $G \in \mathcal{G}_1$ in the former and  $G \in \mathcal{G}_{\nu}$ for $\nu \in \N$ in the latter.

Let us now analyze the relationship between $\theta$-weak dependence, $\alpha$-mixing, and $\phi$-mixing. Most of the PAC Bayesian literature for stationary and heavy tailed data employs the following two mixing conditions, see Remark \ref{rate_literature}, namely $\alpha$-mixing and $\phi$-mixing. The results in the Lemma below give us a proof that the $\theta$-weak dependence is more general than $\alpha$-mixing and $\phi$-mixing and therefore describes the dependence structure of a bigger class of models.

Let $\mathcal{M}$ and $\mathcal{V}$ be two sub-sigma algebras of $\mathcal{F}$. 
First of all, the strong mixing coefficient \cite{R56} is defined as
\[
\alpha(\mathcal{M},\mathcal{V})=\sup\{ |P(M)P(V) - P(M\cap V) | , M \in \mathcal{M}, V \in \mathcal{V}\}.
\]
A stochastic process $\bb{X}$ is said to be $\alpha$-mixing  if
\[
\alpha(r)=\alpha(\sigma\{\bb{X}_s, s \leq 0\}, \sigma\{\bb{X}_s, s \geq r\})
\]
converges to zero as $r \to \infty$.
The $\phi$-mixing coefficient has been introduced in \cite{ibra} and defined as
\[
\phi(\mathcal{M},\mathcal{V})=\sup\{ |P(V|M) - P(V) | , M \in \mathcal{M}, V \in \mathcal{V}, P(M) >0\}.
\]
A stochastic process $\bb{X}$ is said to be $\phi$-mixing  if
\[
\phi(r)=\phi(\sigma\{\bb{X}_s, s \leq 0\}, \sigma\{\bb{X}_s, s \geq r\})
\]
converges to zero as $r \to \infty$.

\begin{Lemma}
	\label{gen}
	Let $(\bb{X}_t)_{t \in \Z}$ be a stationary real-valued stochastic process such that $\E[|\bb{X}_0|^q]< \infty$ for some $q >1$.
	Then, 
	\begin{itemize}
		\item[(a)]  $\theta(r) \leq  2^{\frac{2q-1}{q}} \alpha(r)^{\frac{q-1}{q}} \|\bb{X}_0\|_q \leq 2^{\frac{q-1}{q}} \phi(r)^{\frac{q-1}{q}} \|\bb{X}_0\|_q $, and 
		\item[(b)] $\theta$-weak dependence is a more general dependence notion than $\alpha$-mixing and $\phi$-mixing.
	\end{itemize}
\end{Lemma}

\begin{proof}
	The proof of the first inequality at point (a) is proven in \cite[Proposition 2.5]{CSS20} using the representation of the $\theta$-coefficients (\ref{mix2}). The proof of the second inequality follows from a classical result in \cite[Proposition 3.11]{Bradley}. In \cite[Proposition 2.7]{CSS20}, it is defined a stochastic process which is $\theta$-weak dependent but neither $\alpha$-mixing or $\phi$-mixing.
\end{proof}

As seen in Definition \ref{thetaweaklydependent} by using the lexicographic order in $\R^{1+d}$, an opportune extension of $\theta$-weak dependence valid for random fields can be defined.

The definition of  $\theta$-lex coefficients for $G \in \mathcal{G}_1$ is given in \cite[Definition 2.1]{CSS20}. The latter can be represented as $ \theta_{lex}^v(r):=\sup_{u \in \N} \{ \theta_{u,v}(r)\}$ for $v=1$. Therefore, an alternative way to define the $\theta$-lex coefficients in Definition \ref{thetaweaklydependent} is obviously
\begin{equation}
	\label{alt_theta}
	\theta_{lex}(r)=\sup_{v \in \N} \theta_{lex}^v(r), v \in \N \,\,\,\textrm{for all $r \in \R^+$.}
\end{equation}

The following Lemma has important applications in the following sections.

\begin{Lemma}
	\label{truncated}
	Let $\bb{Z}$ be a $\theta$-lex weakly dependent random field and $M>0$, then $\bb{Z}_i^M=\bb{Z}_{i} \vee (-M) \wedge M$ is $\theta$-lex weakly dependent.
\end{Lemma}

\begin{proof}
	Let $u,v \in \N$, $M >0$, $F \in \mathcal{G}^*_u$, $G \in  \mathcal{G}_v$, and $i_1, i_2, \ldots, i_u \in V_{\Gamma^{\prime}}^r$ where $\Gamma^{\prime}=\{j_1,\ldots,j_v\}$. Let $F^M(\bb{Z}_{i_1},\ldots, \bb{Z}_{i_u})=F(\bb{Z}_{i_1}^M,\ldots, \bb{Z}_{i_u}^M)$, and $G^M(\bb{Z}_{j_1},\ldots, \bb{Z}_{j_v})=G(\bb{Z}_{j_1}^M,\ldots, \bb{Z}_{j_v}^M)$.
	We have that                                                                                                                                                                                                                                                                                                                                                                                                                                             $F^M$ is a bounded function on $(\R^n)^u$ and $G^M$ is a bounded and Lipschitz function on $(\R^n)^{v}$ (with the same Lipschitz coefficients as the function $G$). Let $(\bb{Z}_1 ,\ldots, \bb{Z}_v )$ and $(\tilde{\bb{Z}}_1,\ldots, \tilde{\bb{Z}}_v) \in (\R^n)^{v}$, then
	\begin{align*}
		&|G^M(\bb{Z}_1 ,\ldots, \bb{Z}_v )-G^M(\tilde{\bb{Z}}_1,\ldots, \tilde{\bb{Z}}_v)| \leq Lip(G) \sum_{i=1}^v |\bb{Z}_i^M-\tilde{\bb{Z}}_i^M | \\
		&\leq Lip(G) \sum_{i=1}^v |\bb{Z}_i -\tilde{\bb{Z}}_i |. 
	\end{align*}
	Hence, it holds that 
	\[
	|Cov(F(\bb{Z}_{i_1}^M,\ldots, \bb{Z}_{i_u}^M), G(\bb{Z}_{j_1}^M,\ldots, \bb{Z}_{j_v}^M))|\leq \|F\|_{\infty} v Lip(G) \theta_{lex} (r),
	\]
	where $\theta_{lex}(r)$ are the $\theta$-coefficients of the field $\bb{Z}$. So, the field $\bb{Z}_t^M(x)$ is $\theta$-lex weakly dependent.
\end{proof}

Note that the above result also holds for $\bb{X}$ a $\theta$-weakly dependent process. Therefore, the truncated $\bb{X}_t^M=\bb{X}_t \vee (-M) \wedge M$ is a $\theta$-weakly dependent process.

The notion of $\theta$-lex weak dependence also admits a \emph{projective-type representation}.
\begin{Remark}
	\label{mix_rule2}
	Let $(\bb{Z}_t)_{t \in \Z^{1+d}}$  be a real-valued and $\theta$-lex weakly dependent random field, $\mathcal{L}_1=\{g:\R\to\R, \, g \in \mathcal{G}_v,\,\, Lip(g)\leq 1\}$ and $\Gamma^{\prime}=\{j_1,\ldots,j_{v}\} \in \Z^{1+d}$ such that $|\Gamma^{\prime}|=v$. Let $r \in \N$, and $\mathcal{M}=\sigma\{\bb{Z}_t: t \in V^r_{\Gamma^{\prime}} \subset \Z^{1+d}\}$, then it holds that 
	\begin{equation}
		\label{mix1}
		\theta_{lex}(r)= \sup_{v \in \N} \sup_{\Gamma^{\prime}} \sup_{g \in \mathcal{L}_1}  \|\E[g(\bb{Z}_{j_1},\ldots,\bb{Z}_{j_v})|\mathcal{M}]-\E[g(\bb{Z}_{j_1},\ldots,\bb{Z}_{j_v})]\|_1.
	\end{equation}
	The result above follows by readily applying \cite[Lemma 5.1]{CSS20}. 
\end{Remark}

We now use the representation of the $\theta$-lex coefficients (\ref{mix1}) to understand its relationships to  $\alpha_{\infty,v}$-mixing and $\phi_{\infty,v}$-mixing for $v \in \N \cup \{\infty\}$. These notions are defined in \cite{D98} and they are strong mixing notions used in the study of stationary random fields. 

In general, for $u,v \in \N \cup \{\infty\} $, given coefficients
\[
\alpha_{u,v}(r)=\sup \{\alpha(\sigma(\bb{Z}_{\Gamma}),\sigma(\bb{Z}_{\Gamma^{\prime}})), \Gamma,\Gamma^{\prime} \in \R^{1+d}, |\Gamma| \leq u, |\Gamma^{\prime}|\leq v, dist(\Gamma,\Gamma^{\prime}) \geq r\},
\] 
and
\[
\phi_{u,v}(r)=\sup \{\phi(\sigma(\bb{Z}_{\Gamma}),\sigma(\bb{Z}_{\Gamma^{\prime}})), \Gamma,\Gamma^{\prime} \in \R^{1+d}, |\Gamma| \leq u, |\Gamma^{\prime}|\leq v, dist(\Gamma,\Gamma^{\prime}) \geq r\}.
\] 
a random field $\bb{Z}$ is said to be $\alpha_{u,v}$-mixing or $\phi_{u,v}$-mixing if the coefficients $(\alpha_{u,v}(r))_{r \in \R^+}$ or $(\phi_{u,v}(r))_{r \in \R^+}$ converge to zero as $r \to \infty$.
We then have the following result.

\begin{Lemma}
	\label{gen2}
	Let $(\bb{Z}_t)_{t \in \Z^{1+d}}$ be a stationary real-valued random field such that $\E[|\bb{Z}_0|^q]< \infty$ for some $q >1$.  Then, for $v \in \N \cup \{ \infty\}$,
	\begin{itemize}
		\item[(a)] $ \theta_{lex}(r)\leq 2^{\frac{2q-1}{q}} \alpha_{\infty,v}(r)^{\frac{q-1}{q}} \|\bb{Z}_0\|_q \leq 2^{\frac{q-1}{q}} \phi_{\infty,v}(r)^{\frac{q-1}{q}} \|\bb{Z}_0\|_q $, and
		\item[(b)] it holds that $\theta$-lex weak dependence is more general than $\alpha_{\infty,v}$-mixing and $\alpha$-mixing in the special case of stochastic processes. Moreover, $\theta$-lex weak dependence is more general than $\phi_{\infty,v}$-mixing.
	\end{itemize} 
\end{Lemma}

\begin{proof}
	From the proof of \cite[Proposition 2.5]{CSS20}, we have that
	\[
	\theta_{lex}^1(r) \leq 2^{\frac{2q-1}{q}} \alpha_{\infty,1}(r)^{\frac{q-1}{q}} \|\bb{Z}_0\|_q.
	\]
	Because of (\ref{alt_theta}) and \cite[Proposition 3.11]{Bradley}, we have that
	\[
	\theta_{lex}(r) \leq 2^{\frac{2q-1}{q}} \alpha_{\infty,1}(r)^{\frac{q-1}{q}} \|\bb{Z}_0\|_q \leq 2^{\frac{q-1}{q}} \phi_{\infty,1}(r)^{\frac{q-1}{q}} \|\bb{Z}_0\|_q .
	\]
	Equally,
	\[
	\theta_{lex}(r) \leq 2^{\frac{2q-1}{q}} \alpha_{\infty,v}(r)^{\frac{q-1}{q}} \|\bb{Z}_0\|_q \leq 2^{\frac{q-1}{q}} \phi_{\infty,v}(r)^{\frac{q-1}{q}} \|\bb{Z}_0\|_q.
	\]
	The proof of the point (b) follows directly by \cite[Proposition 2.7]{CSS20}. In fact $\theta_{lex}(r)=\theta_{1,\infty}(r)$ following the notations of \cite[Definition 2.3]{D07} and the process used in the proof of the Proposition is $\theta$-lex weakly dependent but neither $\alpha_{\infty,v}$, $\alpha$ or $\phi_{\infty,v}$-mixing.
	
\end{proof}

%

\subsection{Autocovariance Structure of MMAF and Isotropy}
\label{cov}

Moment conditions for MMAFs are typically expressed in function of the characteristic quadruplet of its driving L\'evy basis and the kernel function $f$. 

\begin{Proposition}\label{proposition:MMAmoments}
	Let $\bb{Z}$ be an $\R$-valued MMAF driven by a L\'evy basis with characteristic quadruplet $(\gamma,\sigma^2,\nu,\pi)$ with kernel function $f:H\times\R\times \R^d\rightarrow \R$ and defined on an ambit set $A_t(x) \subset \R \times \R^d$.
	\begin{enumerate}
		\item[(i)] If $\int_{|x|>1} |x| \nu(dx)<\infty$ and $f\in L^1(H \times \R \times \R^d)\cap L^2(H \times \R \times \R^d)$ the first moment of $\bb{Z}$ is given by
		\begin{align*}
			\E[\bb{Z}_t]= \E(\Lambda^{\prime}) \int_{H}\int_{A_t(x)} f(A,-s,-\xi)  ds \, d\xi \, \pi(dA),
		\end{align*}
		where $\E(\Lambda^{\prime}) =\gamma+\int_{|x|\geq 1} x \, \nu(dx)$.
		\item[(ii)] If $\int_{\R} x^2 \, \nu(dx)<\infty$ and $f\in L^2(H \times \R \times \R^d)$, then $\bb{Z} \in L^2(\Omega)$ and
		\begin{align}
			&Var(\bb{Z}_t(x))= Var(\Lambda^{\prime}) \int_{H}\int_{\R \times \R^d }f(A,-s,-\xi)^2  ds \, d\xi \, \pi(dA), \nonumber\\
			&Cov(\bb{Z}_0(0),\bb{Z}_t(x))= Var(\Lambda^{\prime})\int_{H}\int_{A_0(0)\cap A_t(x)} f(A,-s,-\xi) f(A,t-s,x-\xi) \,ds\, d\xi \, \pi(dA) \hspace{0.5cm},  \nonumber\\
			& \text{and} \nonumber\\
			&Corr(\bb{Z}_0(0),\bb{Z}_t(x))= \frac{\int_{H}\int_{A_0(0)\cap A_t(x)} f(A,-s,-\xi) f(A,t-s,x-\xi) \,ds\, d\xi \, \pi(dA)}{ \int_{H}\int_{\R \times \R^d }f(A,-s,-\xi)^2  ds \, d\xi \, \pi(dA)} \nonumber,
		\end{align}
		where $Var(\Lambda^{\prime}) =\sigma^2+\int_{\R^d}xx'\nu(dx)$. 
		\item[(iii)] If $\sigma^2=0$, $\int_{|x|\in \R} |x| \, \nu(dx)<\infty$, and $f\in L^1(H \times \R \times \R^d)$, then the first moment of $\bb{Z}$ is given by
		\begin{align*}
			E[\bb{Z}_t(x)]= \int_{H}\int_{A_t(x)} f(A,-s,-\xi) \Big(\gamma_0+ \int_{\R}x\nu(dx)\Big) ds \, d\xi \, \pi(dA),
		\end{align*}
		where 
		\begin{equation}
			\label{gamma0}
			\gamma_0:= \gamma-\int_{|x| \leq 1} x \, \nu(dx).
		\end{equation}
	\end{enumerate}
\end{Proposition}
\begin{proof}
	Immediate from  \cite[Section 25]{S} and \cite[Theorem 3.3]{CSS20}.
\end{proof}

From Proposition \ref{proposition:MMAmoments}, we can evince that the autocovariance function of an MMAF depends on the variance of the L\'evy seed $\Lambda^{\prime}$, the kernel function $f$ and the distribution $\pi$ of the random parameter $A$.

We give below the explicit expression of the autocovariance functions for an STOU and MSTOU process.

\begin{Example}
	Let $\bb{Z}$ ad defined in Example \ref{prop_stou}, $u \in \R^d$, $\tau \in \R$, and $\E[\bb{Z}_t(x)^2] <\infty$. Then,
	\begin{align}
		&Cov(\bb{Z}_t(x),\bb{Z}_{t+\tau}(x+u))=Var(\Lambda^{\prime}) \exp(-A \tau) \int_{A_t(x)\cap A_{t+\tau}(x+u)}  \exp(-2A (t-s) )  ds\, d\xi,\,\,\,  \label{cov_stou}\\
		& \textrm{and} \nonumber\\
		&Corr(\bb{Z}_t(x),\bb{Z}_{t+\tau}(x+u))=\frac{\exp(-A \tau) \int_{A_t(x)\cap A_{t+\tau}(x+u)}  \exp(-2A (t-s) ) \,\,  ds\, d\xi}{\int_{A_t(x)}  \exp(-2A(t-s)) \,\, ds \, d\xi}. \label{cor_stou}
	\end{align}
\end{Example}

\begin{Example}
	\label{corr_stou1}
	Let $\bb{Z}$ ad defined in Example \ref{prop_stou} and $d=1$, then
	\begin{align}
		&\rho^T(\tau):=Corr(\bb{Z}_t(x),\bb{Z}_{t+\tau}(x))=\exp(-A |\tau|)  \label{cort_stou},\\
		&\rho^S(u):=Corr(\bb{Z}_t(x),\bb{Z}_{t}(x+u))=\exp\Big(-A\frac{|u|}{c}\Big) \label{cors_stou},\\
		&\rho^{ST}(\tau,u):=Corr(\bb{Z}_t(x),\bb{Z}_{t+\tau}(x+u))=\min\Big( \exp(-A|\tau|), \exp\Big(-\frac{A|u|}{c}\Big)  \Big). \label{corst_stou}
	\end{align}
\end{Example}

\begin{Example}
	Let $\bb{Z}$ be defined as in Example \ref{prop_mstou}, $u \in \R^d$ and $\tau \in \R$, and $\E[\bb{Z}_t(x)^2] <\infty$, then
	\begin{align}
		Cov(\bb{Z}_t(x),&\bb{Z}_{t+\tau}(x+u))= \nonumber\\ &= Var(\Lambda^{\prime}) \exp(-A \tau) \int_0^{\infty} \int_{A_t(x)\cap A_{t+\tau}(x+u)}  \exp(-2A (t-s) )  ds\, d\xi\, l(A) dA, \label{cov_mstou}\\
		Corr(\bb{Z}_t(x),&\bb{Z}_{t+\tau}(x+u))= \nonumber \\ &\frac{\exp(-A \tau) \int_0^{\infty} \int_{A_t(x)\cap A_{t+\tau}(x+u)}  \exp(-2A (t-s) ) \,\, ds\, d\xi\, l(A) dA}{\int_0^{\infty} \int_{A_t(x)}  \exp(-2A(t-s)) \,\, ds\, d\xi\, l(A) dA}. \label{cor_mstou}
	\end{align}
\end{Example}

\begin{Example}
	\label{ex_mstou} 
	Let $\bb{Z}$ be defined as in Example \ref{prop_mstou} for $d=1$, Assumption \ref{ass_light} hold, and $l(A)=\frac{\beta^{\alpha}}{\Gamma(\alpha)} A^{\alpha-1} \exp(-\beta A)$ be the Gamma density with shape and rate parameters $\alpha >d+1$ and $\beta>0$.  For $d=1$, $u \in \R$ and $\tau \in \R$
	\begin{align}
		&Var(\bb{Z}_t(x))= \frac{Var(\Lambda^{\prime})c \beta^2}{2(\alpha-2)(\alpha-1)}  \label{var_mstou}\\
		&Cov(\bb{Z}_t(x),\bb{Z}_{t+\tau}(x+u))= \frac{Var(\Lambda^{\prime})c\beta^{\alpha}}{2(\beta+\max\{|\tau|,|u|/c\})^{\alpha-2} (\alpha-2)(\alpha-1)}  \label{cost_mstou},\\
		&\rho^{ST}(\tau,u):=Corr(\bb{Z}_t(x),\bb{Z}_{t+\tau}(x+u))= \Bigg( \frac{\beta}{\beta+\max\{|\tau|,|u|/c\}} \Bigg )^{\alpha-2}. \label{corst_mstou}
	\end{align}
\end{Example}

It follows the definition of an \emph{isotropic} spatio-temporal random field.

\begin{Definition}[Isotropy]
	\label{isotropy}
	Let $t \in \R$ and $x \in \R^d$. A spatio-temporal random field $(\bb{Z}_t(x))_{(t,x)\in \R\times \R^d}$ is called \emph{isotropic} if its spatial covariance:
	\[Cov(\bb{Z}_t(x),\bb{Z}_t(x+u))=C(|u|),\,\, \forall u \in \R^d\]
	for some positive definite function $C$.
\end{Definition}
STOU and MSTOU processes defined on cone-shaped ambit sets are isotropic random fields.

\subsection{Inference on STOU processes}
\label{sec_stou}
Let us start by explaining the available estimation methodologies for the parameter vector $\theta_0=\{A,c, Var(\Lambda^{\prime})\}$ under the STOU modeling assumption when the spatial dimension $d=1$. Throughout, we refer to the notations used in Example \ref{coeff_stou1}. 

We have two ways of estimating the parameter vector $\theta_0$ in such a scenario. The first one is presented in \cite{STOU}. Here, the parameters $A$ and $c$ are first estimated using \emph{normalized spatial and temporal variograms} defined as
\begin{equation}
	\label{vario_s}
	\gamma^S(u):= \frac{\E((\bb{Z}_t(x)-\bb{Z}_t(x-u))^2)}{Var(\bb{Z}_t(x))}=2 (1-\rho^S(u))=2 \Big( 1-\exp\Big(-\frac{Au}{c} \Big)\Big),
\end{equation}
and
\begin{equation}
	\label{vario_t}
	\gamma^T(\tau):= \frac{\E((\bb{Z}_t(x)-\bb{Z}_{t-\tau}(x))^2)}{Var(\bb{Z}_t(x))}=2(1-\rho^T(\tau))=2(1-\exp(-A\tau)),
\end{equation}
where $\rho^S$ and $\rho^T$ are defined in Example \ref{corr_stou1}.
Note that normalized variograms are used to separate the estimation of the parameters $A$ and $c$ from the parameter $Var(\Lambda^{\prime})$. Let $N(u)$ be the set containing all the pairs of indices at mutual spatial distance $u$ for $u>0$ and the same observation time. Let $N(\tau)$ be the set containing all the pairs of indices where the observation times are at a distance $\tau >0$ and have the same spatial position. $|N(u)|$ and $|N(\tau)|$ give the number of the obtained pairs, respectively. Moreover, let $\bb{\hat{k}}_2$ be the empirical variance which is defined as
\begin{equation}
	\label{emp_var}
	\bb{\hat{k}}_2=\frac{1}{(D-1)}\sum_{i=1}^D \bb{Z^2}_{t_i}(x_i), 
\end{equation}
where $D$ denotes the sample size.
The empirical normalized spatial and temporal variograms are then defined as follows:
\begin{align}
	&\bb{\hat{\gamma}}^S(u)=\frac{1}{|N(u)|} \sum_{i,j \in N(u)} \frac{(\bb{Z}_{t_i}(x_i)-\bb{Z}_{t_j}(x_j))^2}{\bb{\hat{k}_2}} \label{emp_vario_s}\\
	&\bb{\hat{\gamma}}^T(\tau)=\frac{1}{|N(\tau)|} \sum_{i,j \in N(\tau)} \frac{(\bb{Z}_{t_i}(x_i)-\bb{Z}_{t_j}(x_j))^2}{\bb{\hat{k}_2}} \label{emp_vario_t}.
\end{align}
By matching the empirical and the theoretical forms of the normalized variograms, we can estimate $A$ and $c$ by employing the estimators
\begin{equation}
	\label{est_stou}
	\bb{A^{*}}=-\tau^{-1} \log \Big( 1-\frac{\bb{\hat{\gamma}^T}(\tau)}{2}  \Big), \,\,\, \textrm{and} \,\,\, \bb{c^{*}}=-\frac{\bb{A^*}u}{\log \Big(1-\frac{\bb{\hat{\gamma}^S}(u)}{2}\Big)}.
\end{equation}
Alternatively, we can use a least square methodology to estimate the parameters $A$ and $c$, i.e. (\ref{emp_vario_s}) and (\ref{emp_vario_t}) are computed at several lags, and a least-squares estimation is used to fit the computed values to the theoretical curves. The authors in \cite{STOU} use the methodology discussed in \cite{LS} to achieve the last target. We refer the reader also to \cite[Chapter 2]{C93} for further discussions and examples of possible variogram model fitting. The parameter $Var(\Lambda^{\prime})$ can be estimated by matching the second-order cumulant of the STOU with its empirical counterpart.
The consistency of this estimation procedure is proven in \cite[Theorem 12]{STOU}.

A second possible methodology for estimating the vector $\theta_0$ employs a generalized method of moment estimator (GMM), as in \cite{MSTOU}. It is essential to notice that by using such an estimator, we cannot separate the parameter $Var(\Lambda^{\prime})$ from the estimation of the parameters $A$ and $c$. Instead, all moment conditions must be combined into one optimization criterion, and all the estimations must be found simultaneously. Consistency and asymptotic normality of the GMM estimator are discussed in \cite{MSTOU} and \cite{CSS20}, respectively.

For $d \geq 2$, a least square methodology is still applicable for estimating the variogram's parameters. The estimator used in \cite{STOU} is a normalized version of the least-square estimator for spatial variogram's parameters discussed in \cite{LS}, which also applies for $d >1$. This method, paired with a method of moments (matching the second order cumulant of the field $\bb{Z}$ with its empirical counterparts), allows estimating the parameter $Var(\Lambda^{\prime})$. The GMM methodology discussed in \cite{MSTOU} also continues to apply for $d \geq 2$. However, when the spatial dimension is increasing, the shape of the \emph{normalized variograms} and the \emph{field's moments} become more complex, and higher computational effort is required to navigate through the high dimensional surface of the optimization criterion behind least-squares or GMM estimators.

\subsection{Inference for MSTOU processes}
\label{sec_mstou}
When estimating the parameter vector $\theta_1=\{\alpha, \beta,c, Var(\Lambda^{\prime})\}$ under an MSTOU modeling assumption-- see, for example, solely the shape of the coefficients in Example \ref{coeff_mstou1}-- it is evident that the shape of the autocorrelation function, and therefore of the normalized temporal and spatial variograms, become more complex for increasing $d$. As already addressed in the previous sections, when estimating the parameters $(\alpha, \beta,c)$ alone, we can use the least-squares type estimator discussed in \cite{LS}. Moreover, by pairing the latter with a method of moments or using a GMM estimator, we can estimate the complete vector $\theta_1$.

\subsection{Time series models}
\label{timeseries_app}
In the MMAF framework, we can also find time series models. The latter are $\theta$-weakly dependent.

\begin{Example}[Time series case]\label{ex:timeseriescase}
	The supOU process studied in \cite{BN01} and \cite{BNS11} is an example of a \emph{causal} mixed moving average process.
	Let the kernel function $f(A,s)= \mathrm{e}^{-As} 1_{[0,\infty)}(s)$, $A \in \R^{+}$, $s \in \R$ and $\Lambda$ a  L\'evy basis on $\R^{+}\times\R$ with generating quadruple $(\gamma,\sigma^2,\nu,\pi)$ such that
	\begin{equation}
		\label{supou_ass}
		\int_{|x|>1} \log(|x|) \, \nu(dx) < \infty, \,\, \textrm{and} \,\, \int_{\R^{+}} \frac{1}{A} \pi(dA) < \infty,
	\end{equation}
	then the process
	\begin{equation}
		\label{supou}
		\bb{Z}_t=\int_{\R^{+}} \int_{-\infty}^t \mathrm{e}^{-A (t-s)} \,\Lambda(dA,ds)
	\end{equation}
	is well defined for each $t \in \R$ and strictly stationary and called a supOU process where $A$ represents a random mean reversion parameter.
	
	If $\E(\Lambda^{\prime})=0$ and $\int_{|x|>1} |x|^2 \nu(dx) < \infty$, the supOU process  is $\theta$-weakly dependent
	with coefficients
	\begin{equation}
		\label{theta_sup}
		\theta_{Z}(r)\leq  \Big( \int_{\R^{+}} \int_{-\infty}^{r} \mathrm{e}^{-2As} \sigma^2 \, ds \, \pi(dA) \Big)^{\frac{1}{2}}=\Big[ Var(\Lambda^{\prime}) \int_{\R^{+}} \frac{\mathrm{e}^{-2Ar}}{2A} \, \pi(dA)\Big]^{\frac{1}{2}}
	\end{equation}
	\[
	= Cov(\bb{Z}_0,\bb{Z}_{2r})^{\frac{1}{2}},
	\]
	where $Var(\Lambda^{\prime})= \sigma^2+ \int_{\R} x^2 \nu(dx)$, by using Theorem 3.11 in \cite{BNS11}.
	
	If $\E(\Lambda^{\prime})=\mu$ and $\int_{|x|>1} |x|^2 \nu(dx) < \infty$, the supOU process is $\theta$-weakly dependent
	with coefficients
	\begin{equation}
		\label{theta_sup_1}
		\theta_{Z}(r) \leq  \Big(Cov(\bb{Z}_0,\bb{Z}_{2r})+\frac{4\mu^2}{Var(\Lambda^{\prime})^2}Cov(\bb{Z}_0,\bb{Z}_r)^2\Big)^{\frac{1}{2}}.
	\end{equation}
	If $\int_{\R} |x| \nu(dx) < \infty$, $\sigma^2=0$, $\gamma_0=\gamma-\int_{|x|\leq 1} x \,\nu(dx) >0 $ and $\nu(\R^{-})=0 $, where $\R^{-}$ identifies the set of the negative real numbers, then the supOU process admits $\theta$-coefficients
	\begin{equation}
		\label{theta_sup_2}
		\theta_{Z}(r) \leq  \mu \int_{\R^{+}} \frac{\mathrm{e}^{-Ar}}{A} \, \pi(dA),
	\end{equation}
	and when in addition $\int_{|x|>1} |x|^2 \nu(dx) < \infty$
	\begin{equation}
		\label{theta_sup_3}
		\theta_{Z}(r) \leq \frac{2\mu}{Var(\Lambda^{\prime})} Cov(\bb{Z}_0,\bb{Z}_r).
	\end{equation}
	
	Note that the necessary and sufficient condition $\int_{\R^{+}} \frac{1}{A} \,\pi(dA)$ for the supOU process to exist is satisfied by many continuous and discrete distributions $\pi$, see \cite[Section 2.4]{STW15}
	for more details. For example, a probability measure $\pi$ being absolutely continuous with density $\pi^{\prime}=x^{h} l(x)$ and regularly varying at zero from the right with $h>0$, i.e., l is slowly varying at zero, satisfies the above condition. If moreover, $l(x)$ is continuous in $(0.+\infty)$ and $\lim_{x \to 0^{+}} l(x) >0$ exists, it holds that 
	\[
	Cov(\bb{Z}_0,\bb{Z}_r)\sim \frac{C}{r^{h}}, \,\,\textrm{with a constant $C>0$ and $r \in \R^+$} 
	\]
	where for $h \in (0,1)$ the supOU process exhibits long memory and for $h > 1$ short memory.
	In this set-up, concrete examples where the covariances are calculated explicitly can be found in \cite{BN05} and \cite{CS19}. 
	
\end{Example}

Another interesting example of MMAFs is given by the class of \emph{trawl processes}. A distinctive feature of these processes is that one can model the correlation structure independently from the marginal distribution, see \cite{Trawl} for further details on their definition. In the case of trawl processes, we also have available in the literature likelihood-based methods for estimating their parameters; see \cite{Trawls2} for further details. 

In general, the generalized method of moments is employed to estimate the parameters of an MMAF, see \cite{CS19,STOU, MSTOU}.

\section{Appendix}

\subsection{Bounds for the $\theta$-lex coefficients of MMAF}
\label{appb1}


In \cite[Proposition 3.11]{CSS20}, it is given a general methodology to show that an MMAF $\bb{Z}$ is $\theta$-lex weakly dependent. Given that the definition of $\theta$-lex-weak dependence used in the paper slightly differs from the one given in \cite{CSS20}, the proof of Proposition \ref{proposition:mmathetaweaklydep} differs from the one of \cite[Proposition 3.11]{CSS20}. Proposition \ref{example:spatiotemporaldata} is a novel computations of a bound of the $\theta$-lex coefficients of an MMAF when the kernel does not depend on the spatial component.

Before giving a detailed account of these proofs, let us state first some notations.
Let $r >0$, $\{(t_{j_1},x_{j_1}),\ldots(t_{j_v},x_{j_v})\}= \Gamma^{\prime}\in \R^{1+d}$ and $\{(t_{i_1},x_{i_1}),\ldots,(t_{i_u},x_{i_u})\}= \Gamma \in V_{\Gamma^{\prime}}^r$ such that $|\Gamma|=u$ and $|\Gamma^{\prime}|=v$ for $(u,v)\in \N\times \N$. We call the \emph{truncated} (influenced) MMAF the vector
\begin{align}
	\label{truncated_rf}
	\bb{Z}_{\Gamma'}^{(\psi)}&= \left(\bb{Z}^{(\psi)}_{t_{j_1}}(x_{j_1}),\ldots,\bb{Z}^{(\psi)}_{t_{j_v}}(x_{j_v})\right)^{\top},
\end{align}
where $\psi:=\psi(r)$ for $r>0$. In particular, for all $a\in \{1,\ldots,u\}$ and a $b \in \{1,\ldots,v\}$ , $\psi$ has to be chosen such that it exists a set  $B^{\psi}_{t_{j_b}}(x_{j_b})$ with the following properties.
\begin{itemize}
	\item $|B^{\psi}_{t_{j_b}}(x_{j_b})|\to \infty$ as $r \to \infty$ for all $ b$, and
	\item  $I_{i_a}=H\times A_{t_{i_a}}(x_{i_a})$ and $I_{j_b}=H\times B^{\psi}_{t_{j_b}}(x_{j_b}) $ are disjoint sets or intersect on a set $H\times O$, where $O\in\R^{1+d}$ and $dim(O)<d+1$, for all $a$ and $b$
\end{itemize}

Let us now assume that it is possible to construct the sets $B^{\psi}_{t_{j_b}}(x_{j_b})$. 
Then, since $\pi\times\lambda_{1+d}(H\times O)=0$ and by the definition of a L\'evy basis, it follows that
\begin{align*}
	\bb{Z}_{t_{i_a}}(x_{i_a}) &= \int_H\int_{A_{t_{i_a}}(x_{i_a})}f(A,t_{i_a}-s,x_{i_a}-\xi)\Lambda(dA,ds,d\xi) \text{ and }\\
	\bb{Z}_{t_{j_b}}^{(\psi)}(x_{j_b}) &= \int_H\int_{B^{\psi}_{t_{j_b}}(x_{j_b})} f(A,t_{j_b}-s,x_{j_b}-\xi)\Lambda(dA,ds,d\xi),
\end{align*}
and $\bb{Z}_\Gamma$ and $\bb{Z}_{\Gamma'}^{(\psi)}$ are independent. Hence, for $F\in\mathcal{G}_{u}^*$ and $G\in\mathcal{G}_v$, $F(\bb{Z}_\Gamma)$ and $G(\bb{Z}_{\Gamma'}^{(\psi)})$ are also independent. Now 
\begin{align}
	\begin{aligned}
		&|Cov(F(\bb{Z}_{\Gamma}),G(\bb{Z}_{\Gamma'}))| \\
		&\leq |Cov(F(\bb{Z}_{\Gamma}),G(\bb{Z}_{\Gamma'}^{(\psi)}))|+|Cov(F(\bb{Z}_{\Gamma}),G(\bb{Z}_{\Gamma'})-G(\bb{Z}_{\Gamma'}^{(\psi)}))|\\
		&= |E[(G(\bb{Z}_{\Gamma'})-G(\bb{Z}_{\Gamma'}^{(\psi)}) )F(\bb{Z}_{\Gamma})]-E[G(\bb{Z}_{\Gamma'})-G(\bb{Z}_{\Gamma'}^{(\psi)})]E[F(\bb{Z}_{\Gamma})]| \\
		&\leq 2 \norm{F}_\infty E[|G(\bb{Z}_{\Gamma'})-G(\bb{Z}_{\Gamma'}^{(\psi)})|]\leq 2 \text{Lip}(G) \norm{F}_\infty \sum_{l=1}^{v} E[|\bb{Z}_{t_{j_l}}(x_{j_l})-\bb{Z}_{t_{j_l}}^{(\psi)}(x_{j_l})|] =\\
		&= 2 \text{Lip}(G) \norm{F}_\infty v E[|\bb{Z}_{t_{j_1}}(x_{j_1})-\bb{Z}_{t_{j_1}}^{(\psi)}(x_{j_1})|], 
	\end{aligned}
\end{align}
because an (influenced) MMAF is a stationary random field. To show that a field satisfy Definition \ref{thetaweaklydependent}, is then enough to prove that $E[|\bb{Z}_{t_{j_1}}(x_{j_1})-\bb{Z}_{t_{j_1}}^{(\psi)}(x_{j_1})|]$ in the above inequality converges to zero as $r \to \infty$. The proofs of Proposition \ref{proposition:mmathetaweaklydep} and \ref{example:spatiotemporaldata} below differ in the definition of the sequence $\psi$ and the sets $B^{\psi}_{t_{j_b}}(x_{j_b})$.

\begin{Proposition}\label{proposition:mmathetaweaklydep}
	Let $\Lambda$ be an $\R$-valued L\'evy basis with characteristic quadruplet $(\gamma,\sigma^2,\nu,\pi)$, $f:H\times\R^{1+d}\ra \R $ a $\mathcal{B}(H\times\R^{1+d})$-measurable function and $\bb{Z}_t(x)$ be defined as in (\ref{mmaf_i}).
	\begin{itemize}
		\item [(i)] If $\int_{|x|>1} x^2\nu(dx)<\infty$, $\gamma+\int_{|x|>1}x\nu(dx)=0$ and $f\in L^2(H\times \R^{1+d})$, then $\bb{Z}$ is $\theta$-lex-weakly dependent and
		\begin{align*}
			\theta_{lex}(r)\leq 2 \Big(\int_H\int_{-\infty}^{\rho(r)} Var(\Lambda^{\prime}) \int_{\norm{\xi}\leq cs} f(A,-s,-\xi)^2 dsd\xi\pi(dA)\Big)^{\frac{1}{2}}.
		\end{align*}
		\item [(ii)] If $\int_{|x|>1} \, x^2\nu(dx)<\infty$ and $f\in L^2(H\times \R^{1+d})\cap L^1(H\times \R^{1+d})$, then $\bb{Z}$ is $\theta$-lex-weakly dependent and
		\begin{align*}
			\theta_{lex}(r)\leq 2& \bigg(\int_H\int_{-\infty}^{\rho(r)} Var(\Lambda^{\prime}) \int_{\norm{\xi}\leq cs} f(A,-s)^2 \, ds d\xi\pi(dA)\\
			&+\bigg|\int_S\int_{-\infty}^{\rho(r)} \E(\Lambda^{\prime}) \int_{\norm{\xi}\leq cs}  f(A,-s) \,   dsd\xi\pi(dA)\bigg|^2 \bigg)^{\frac{1}{2}}.
		\end{align*}
		\item [(iii)] If $\int_{\R } |x| \, \nu(dx)<\infty$, $\sigma^2=0$ and $f\in L^1(H \times \R^{1+d})$ with $\gamma_0$ defined in (\ref{gamma0}), then $\bb{Z}$ is $\theta$-lex-weakly dependent and 
		\begin{align*}
			\theta_{lex}(r)\leq 2& \bigg(\int_H\int_{-\infty}^{\rho(r)} \int_{\|\xi\|\leq cs} |f(A,-s)\gamma_0| \,ds\, d\xi\pi(dA)\\
			+&\int_H\int_{-\infty}^{\rho(r)} \int_{\|\xi\|\leq cs} \int_{\R} |f(A,-s)x| \,\nu(dx)\,ds\pi(dA)\bigg).
		\end{align*}
		
	\end{itemize}
	The results above hold for all $r>0$ with 
	\begin{align}\label{eq:psi}
		\rho(r)=\frac{-r \min(1/c,1)}{\sqrt{(d+1)(c^2+1)}},
	\end{align}
	$Var(\Lambda^{\prime}) =\sigma^2+\int_{\R}x^2 \, \nu(dx)$ and $\E(\Lambda^{\prime}) =\gamma+\int_{|x|\geq1} x\nu(dx)$.
\end{Proposition}

\begin{proof}[Proof of Proposition \ref{proposition:mmathetaweaklydep}]
	In this proof, we assume that $B^{\psi}_{t_{j_b}}(x_{j_b})=A_{t_{j_b}}(x_{j_b})\backslash V_{(t_{j_b},x_{j_b})}^{\psi} $,
	where $\psi=\psi(r):= \frac{r}{\sqrt{(d+1)(c^2+1)}}$. 
	
	\begin{itemize}
		\item[(i)] 
		Using the translation invariance of $A_t(x)$ and $V_{(t,x)}^{(\psi)}$ we obtain
		\begin{align*}
			&E[|\bb{Z}_{t_{j_1}}(x_{j_1})-\bb{Z}_{t_{j_1}}^{(\psi)}(x_{j_1})|] \leq \left( \int_H \int_{A_{0}(0)\cap V_{0}^{\psi}}  Var(\Lambda^{\prime}) f(A,-s,-\xi)^2 d\xi ds  \pi(dA) \right)^{\frac{1}{2}}\\
			&\quad = \left( \int_H \int_{-\infty}^{\frac{-r \min(1/c,1)}{\sqrt{(d+1)(c^2+1)}}} Var(\Lambda^{\prime})\int_{\norm{\xi}\leq cs} f(A,-s,-\xi)^2  \, d\xi ds \pi(dA) \right)^{\frac{1}{2}},
		\end{align*}
		where we have used Proposition \ref{proposition:MMAmoments}-(ii) to bound the $L_1$-distance from above. Overall, we obtain
		\begin{align*}
			\theta_{lex}(r)\leq 2\left( \int_H \int_{-\infty}^{\frac{-r \min(1/c,1)}{\sqrt{(d+1)(c^2+1)}}} Var(\Lambda^{\prime}) \int_{\norm{\xi}\leq cs} f(A,-s,-\xi)^2  d\xi ds \pi(dA) \right)^{\frac{1}{2}},
		\end{align*}
		which converges to zero as $r$ tends to infinity by applying the dominated convergence theorem.
		\item[(ii)] By applying Proposition \ref{proposition:MMAmoments}-(i) and (ii), we obtain 
		\begin{align*}
			&E[|\bb{Z}_{t_{j_1}}(x_{j_1})-\bb{Z}_{t_{j_1}}^{(\psi)}(x_{j_1})|]
		\end{align*}
		\begin{align*}
			&\leq \Bigg( \int_H \int_{-\infty}^{\frac{-r \min(1/c,1)}{\sqrt{(d+1)(c^2+1)}}} Var(\Lambda^{\prime}) \int_{\norm{\xi}\leq cs} f(A,-s,-\xi)^2 d\xi ds \pi(dA)\\
			&\quad + \left(\int_H\int_{-\infty}^{\frac{-r \min(1/c,1)}{\sqrt{(d+1)(c^2+1)}}} \E(\Lambda^{\prime}) \int_{\norm{\xi}\leq cs} f(A,-s,-\xi)d\xi ds \pi(dA)  \right)^2 \Bigg)^{\frac{1}{2}}.
		\end{align*}
		Finally, we proceed similarly to proof (i) and obtain the desired bound.
		\item[(iii)] We apply now Proposition \ref{proposition:MMAmoments}-(iii). Then,
		\begin{align*}
			&E[|\bb{Z}_{t_{j_1}}(x_{j_1})-\bb{Z}_{t_{j_1}}^{(\psi)}(x_{j_1})|]\\
			&\leq \Bigg( \int_S \int_{-\infty}^{\frac{-r \min(1/c,1)}{\sqrt{(d+1)(c^2+1)}}}\int_{\norm{\xi}\leq cs} |f(A,-s,-\xi)\gamma_0| d\xi ds  \pi(dA)\\
			&\quad + \int_H\int_{-\infty}^{\frac{-r \min(1/c,1)}{\sqrt{(d+1)(c^2+1)}}} \int_{\norm{\xi}\leq cs}\int_{\R} |f(A,-s,-\xi)y| \nu(dy)d\xi ds \pi(dA)   \Bigg).
		\end{align*}
		The bound for the $\theta$-lex-coefficients is obtained following the proof line in (i).
	\end{itemize}
	\hfill
\end{proof}

Proposition \ref{proposition:mmathetaweaklydep} gives general bounds for the $\theta$-lex coefficients of MMAF. For example, it can also be used to compute upper bounds for the $\theta$-lex-coefficients of an MSTOU process for $d >2$ which Proposition \ref{example:spatiotemporaldata} does not cover.

\begin{Corollary}
	\label{gen_coeff}
	Let $\bb{Z}$ be an MSTOU process as in Definition \ref{prop_mstou} and  $(\gamma,\sigma^2,\nu,\pi)$ be the characteristic quadruplet of its driving L\'evy basis. Moreover, let the mean reversion parameter $A$ be $Gamma(\alpha,\beta)$ distributed
	with density $l(A)=\frac{\beta^{\alpha}}{\Gamma(\alpha)} A^{\alpha-1} \exp(-\beta A)$ where $\alpha >d+1$ and $\beta>0$.
	\begin{enumerate}
		\item[(i)] If $\int_{|x|>1} x^2 \,\nu(dx)<\infty$ and $\gamma+\int_{|x|>1}x\nu(dx)=0$, then $\bb{Z}$ is $\theta$-lex-weakly dependent. Let $c\in[0,1]$, then for
		\begin{align*}
			d&=1,\quad && \theta_{lex}(r)\leq2\left(\frac{c Var(\Lambda^{\prime}) \beta^\alpha}{2\Gamma(\alpha)} \left(\frac{\Gamma(\alpha-2)}{(2\psi+\beta)^{\alpha-2}}+\frac{2\psi\Gamma(\alpha-1)}{(2\psi+\beta)^{\alpha-1}} \right)\right)^{\frac{1}{2}}, \\
			\textrm{and for}  & && \\
			d& \geq 2,\quad &&\theta_{lex}(r)\leq2\left(V_d(c)\frac{d!Var(\Lambda^{\prime}) \beta^{\alpha} }{2^{d+1}} \sum_{k=0}^d \frac{(2\psi)^k}{k!(2\psi+\beta)^{\alpha-d-1+k}}\frac{\Gamma(\alpha-d-1+k)}{\Gamma(\alpha)}\right)^{\frac{1}{2}}.
		\end{align*}
		Let $c>1$, then for
		\begin{alignat*}{2}
			d&\in\N,\ \ &&\theta_{lex}(r)\leq2\Bigg(V_d(c)\frac{d!Var(\Lambda^{\prime}) \beta^{\alpha} }{2^{d+1}} \sum_{k=0}^d \frac{\left(\frac{2\psi}{c}\right)^k}{k!\left(\frac{2\psi}{c}+\beta\right)^{\alpha-d-1+k}}\frac{\Gamma(\alpha-d-1+k)}{\Gamma(\alpha)}\Bigg)^{\frac{1}{2}}.
		\end{alignat*}
		The above implies that, in general, $\theta_{lex}(r)=\mathcal{O}(r^{\frac{(d+1)-\alpha}{2}})$. 
		\item [(ii)] If $\int_{\R} |x| \,\nu(dx)<\infty$, $\Sigma=0$ and $\gamma_0$ as defined in (\ref{gamma0}), then $\bb{Z}$ is $\theta$-lex-weakly dependent. Let $c \in (0,1]$, then for 
		\begin{alignat*}{2}
			d&\in\N,\qquad&& \theta_{lex}(r)\leq2V_d(c) d! \beta^{\alpha} \gamma_{abs}  \sum_{k=0}^d \frac{\psi^k}{k!(\psi+\beta)^{\alpha-d-1+k}}\frac{\Gamma(\alpha-d-1+k)}{\Gamma(\alpha)},
		\end{alignat*}
		whereas for $c>1$ and
		\begin{alignat*}{2}
			d&\in\N,\qquad &&\theta_{lex}(r)\leq2V_d(c)d! \beta^{\alpha}\gamma_{abs} \sum_{k=0}^d \frac{\left(\frac{\psi}{c}\right)^k}{k!\left(\frac{\psi}{c}+\beta\right)^{\alpha-d-1+k}}\frac{\Gamma(\alpha-d-1+k)}{\Gamma(\alpha)},
		\end{alignat*}
		where $\gamma_{abs}=|\gamma_0|+\int_\R|x|\nu(dx)$, $V_d(c)$ denotes the volume of the $d$-dimensional ball with radius $c$, and $\psi:=\psi(r)=\frac{1}{\sqrt{c^2+1}} \frac{r}{d+1}$. 
	\end{enumerate}
\end{Corollary}

\begin{proof}
	Proof of this corollary can be obtained by modifying the proof of \cite[Section 3.7]{CSS20} in line with the calculations performed in Proposition \ref{proposition:mmathetaweaklydep}.
\end{proof}

The results of the Corollary above imply that, in general, $\theta_{lex}(r)=\mathcal{O}(r^{(d+1)-\alpha})$. We give now the proof of Proposition \ref{example:spatiotemporaldata}.

\begin{proof}[Proof of Proposition \ref{example:spatiotemporaldata}]

	\begin{enumerate}
		\item[(i)]  Without loss of generality, let us determine the truncated set when $(t_{j_b},x_{j_b})=(0,0)$. We use, to this end, two auxiliary ambit sets translated by a value $\psi>0$ along the spatial axis, namely, the cones $A_{0}(\psi)$ and $A_{0}(-\psi)$, as illustrated in Figure \ref{figure:truncatedformstou}-(a),(c) for $c\leq1$ and in Figure \ref{figure:truncatedformstou}-(b),(d) for $c>1$. Then, we set the truncated integration set to $B_0^{\psi}(0)=A_{0}(0)\backslash (A_{0}(\psi)\cup A_{0}(-\psi))$. 
		Since $(t_{i_a},x_{i_a})\in V_{(0,0)}^r$, it is sufficient to choose $\psi$ such that the integration set of $\bb{Z}_{0}^{(\psi)}(0)$ is a subset of $(V_{(0,0)}^r)^c$. To this end, the three intersecting points $(\tfrac{-\psi}{2c},\tfrac{-\psi}{2})$, $(\tfrac{-\psi}{c},0)$ and $(\tfrac{-\psi}{2c},\tfrac{\psi}{2})$ have to be inside the set $(V_{(0,0)}^r)^c$, as illustrated in Figure \ref{figure:truncatedformstou}-(e) for $c\leq1$ and in Figure \ref{figure:truncatedformstou}-(f) for $c>1$. Clearly, this leads to the conditions $\psi\leq rc$, $\psi\leq 2r$ and $\psi\leq 2rc$, which are satisfied for $\psi=r\min(2,c)$. Hence, by using Proposition \ref{proposition:MMAmoments}-(ii), we have that 
		\begin{align*}
			&\theta_{lex}(r) \leq 2 Var(\Lambda^{\prime})^{1/2} \left( \int_0^{\infty} \int_{A_0(0)\cap (A_0(\psi)\cup A_0(-\psi))}f(A,-s)^2 ds d\xi \pi(d\lambda) \right)^{1/2}\\
			&=  2 Var(\Lambda^{\prime})^{1/2}  \Bigg( \int_0^{\infty} \int_{A_0(0)\cap A_0(\psi) }f(A,-s)^2 ds d\xi \pi(d\lambda) \\
			&+\int_0^{\infty} \int_{A_0(0)\cap  A_0(-\psi)}f(A,-s)^2 ds d\xi \pi(d\lambda)\\ &\qquad\qquad\quad -\int_0^{\infty} \int_{A_0(0)\cap  A_0(\psi)\cap A_0(-\psi)}f(A,-s)^2 ds d\xi \pi(d\lambda) \Bigg)^{1/2}\\
			&\leq  2 \sqrt{2Cov(\bb{Z}_0(0),\bb{Z}_0(r\min(2,c)))}.
		\end{align*}
		which converges to zero as $r \to \infty$ for the dominated convergence theorem.
		\begin{figure}
			\centering
			\subfigure[\label{Plot1} Integration set $A_0(0)$ for $c=1/\sqrt{2}$ together with the complement of  $V_{(0,0)}^r$ for $r=3$.]{\includegraphics[width=.48\textwidth]{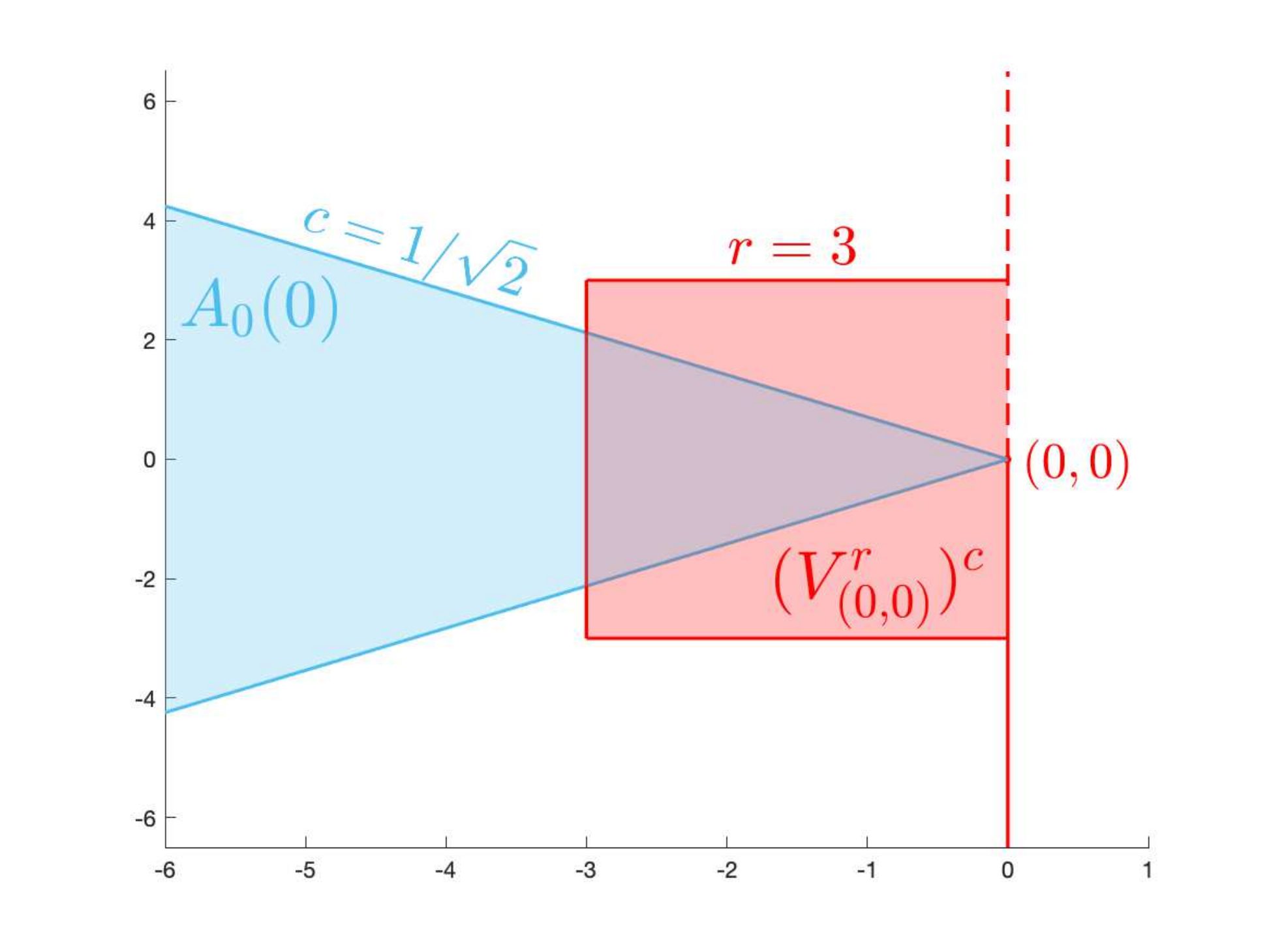}} \hspace{.02\textwidth}
			\subfigure[\label{Plot2} Integration set $A_0(0)$ for $c=2\sqrt{2}$ together with the complement of  $V_{(0,0)}^r$ for $r=3$.]{\includegraphics[width=0.48\textwidth]{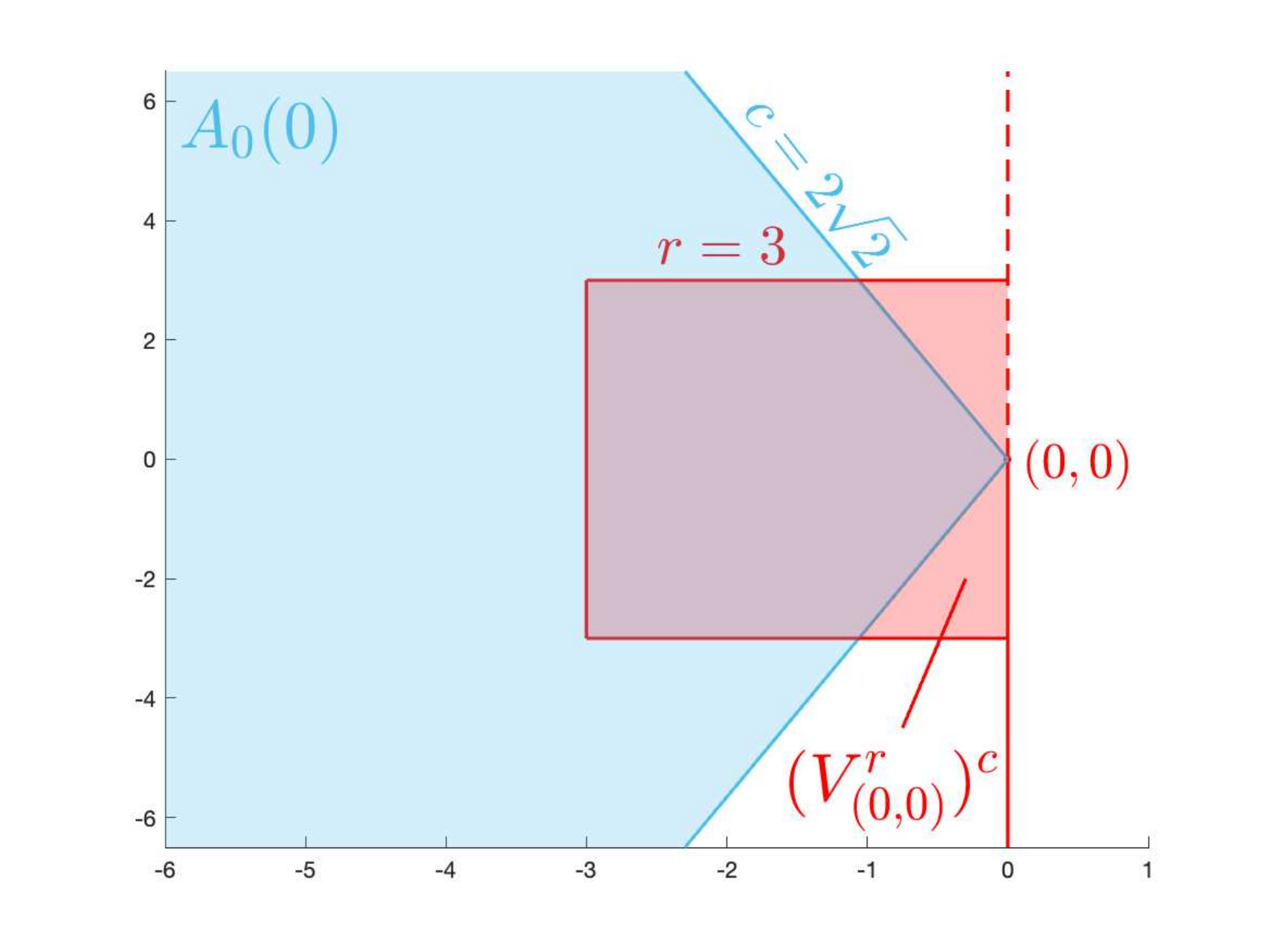}} \hspace{.02\textwidth}
			\subfigure[\label{Plot3} Integration set $A_0(0)$ together with $A_0(\psi)$ and $A_0(-\psi)$ for $\psi=r\min(2,c)$.]{\includegraphics[width=.48\textwidth]{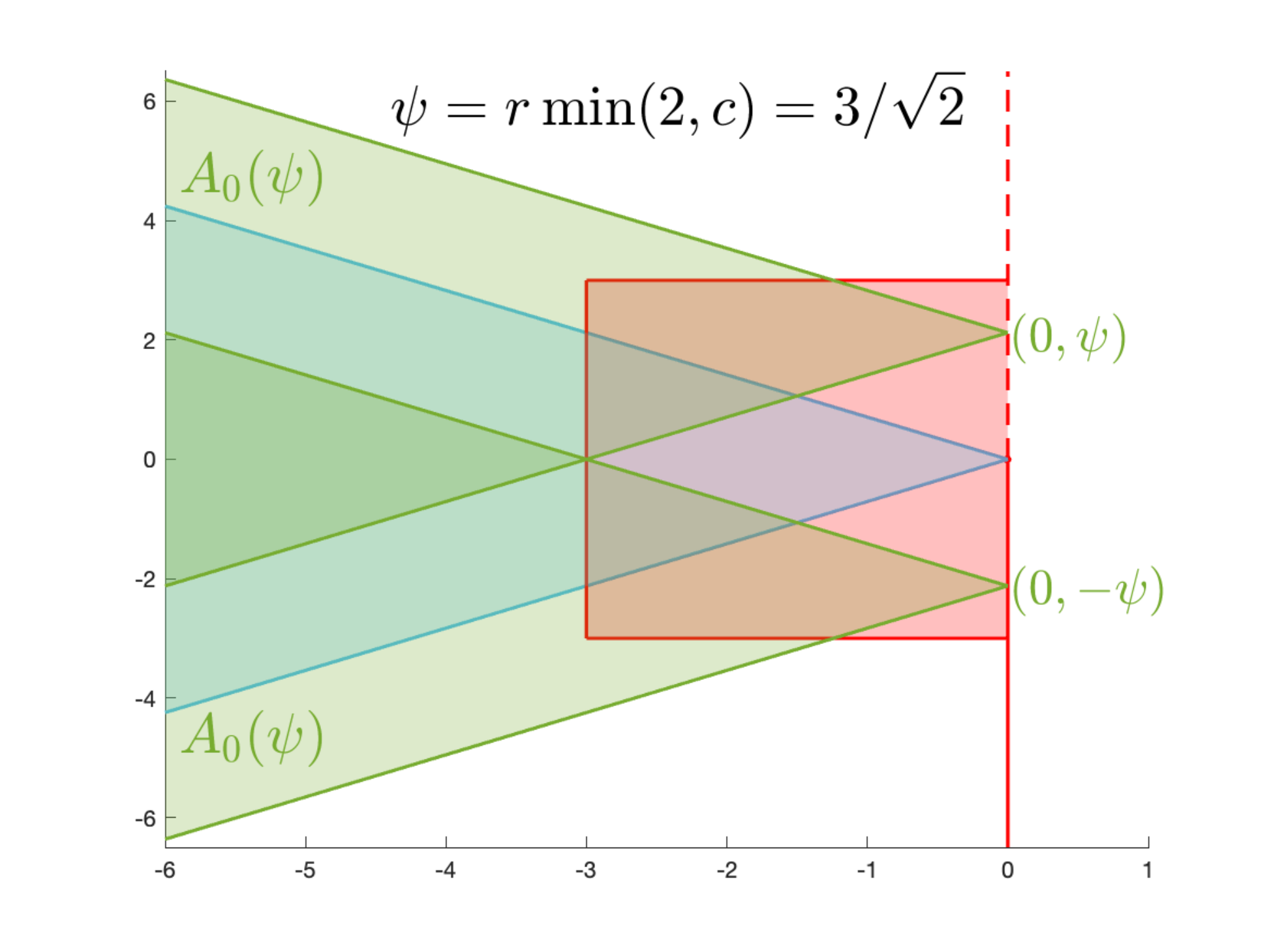}}\hspace{.02\textwidth}
			\subfigure[\label{Plot4} Integration set $A_0(0)$ together with $A_0(\psi)$ and $A_0(-\psi)$ for $\psi=r\min(2,c)$.]{\includegraphics[width=.48\textwidth]{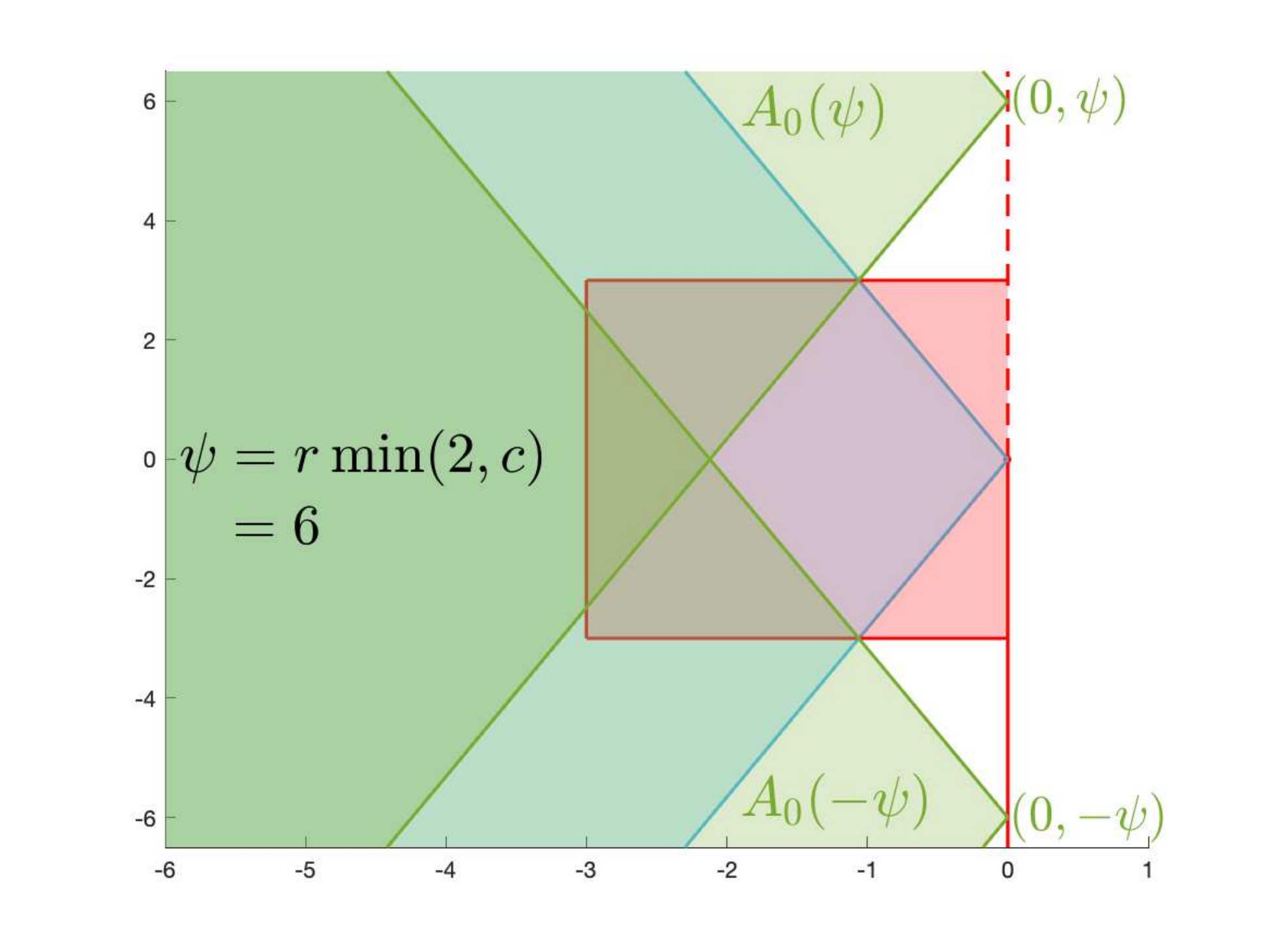}} \hspace{.02\textwidth}
			\subfigure[\label{Plot5} Integration set of $\bb{Z}_0^{(\psi)}(0)$ together with the complement of $V_{(0,0)}^r$.]{\includegraphics[width=0.48\textwidth]{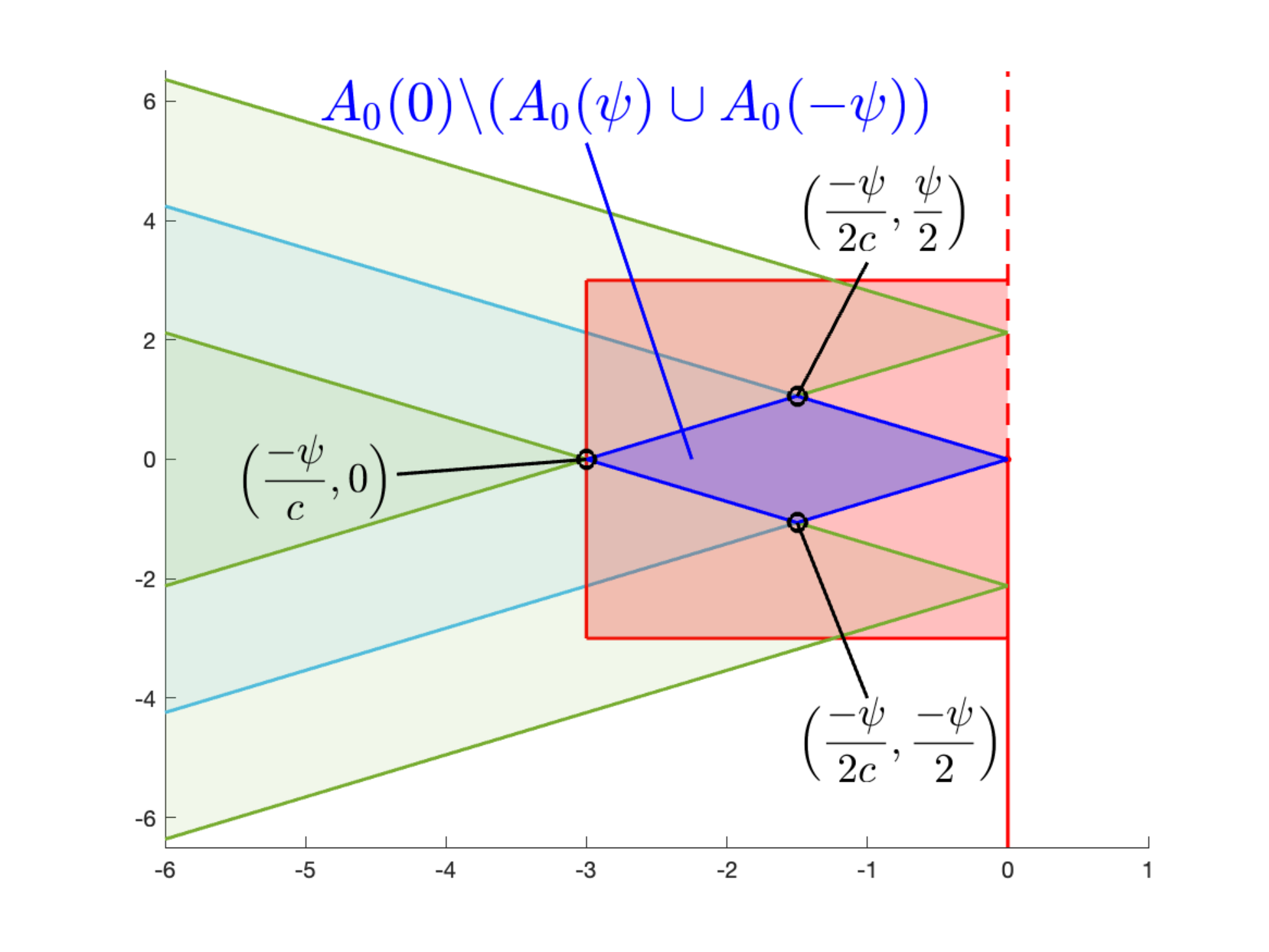}} \hspace{.02\textwidth}
			\subfigure[\label{Plot6} Integration set of $\bb{Z}_0^{(\psi)}(0)$ together with the complement of $V_{(0,0)}^r$.]{\includegraphics[width=.48\textwidth]{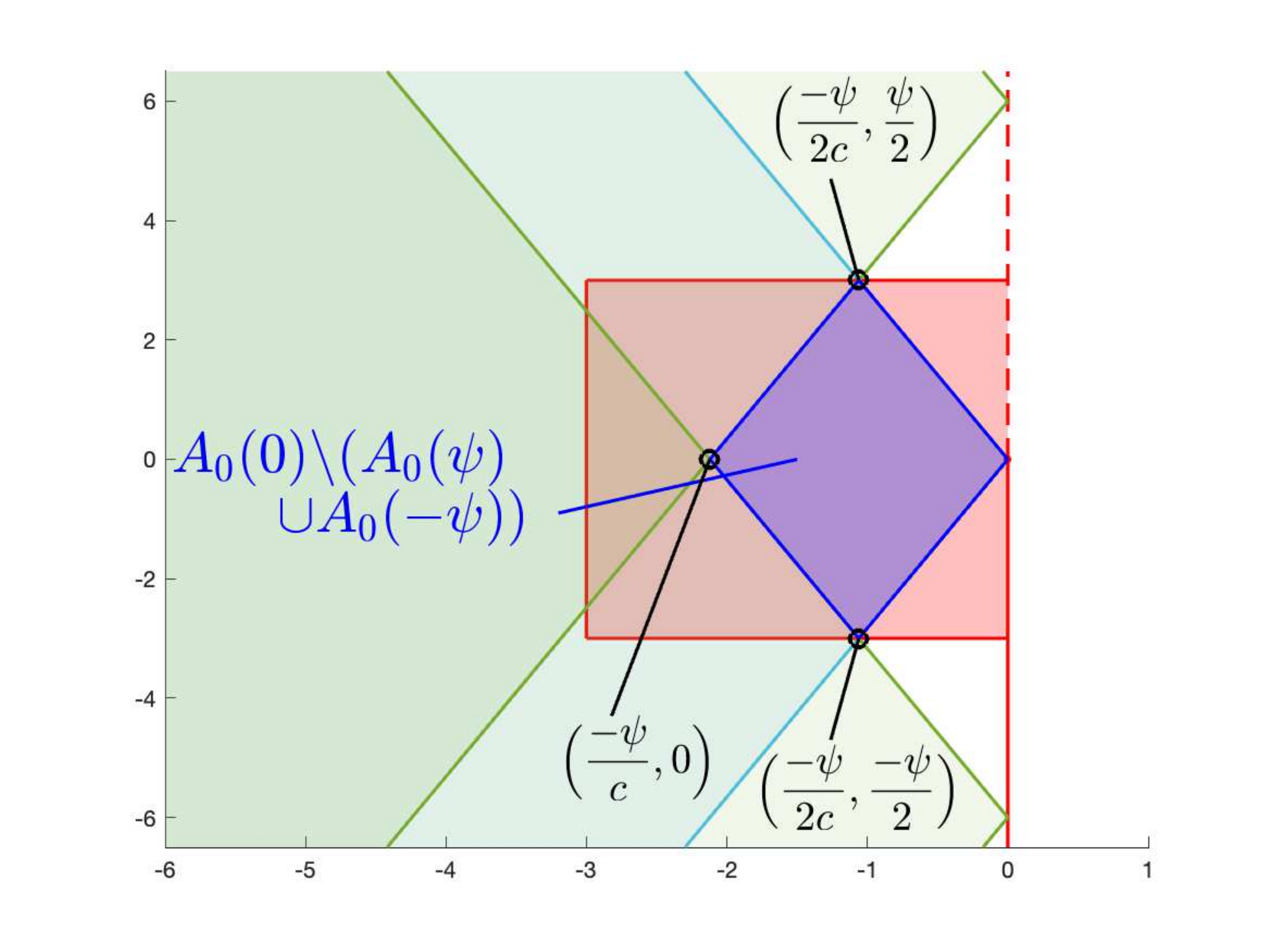}}\vspace{-0.2cm}
			\caption{Exemplary integration set and truncated integration set of an MMAF $\bb{Z}_{t}(x)$ for $(t,x)=(0,0)$.}\label{figure:truncatedformstou}
		\end{figure}

		\item[(ii)] In this proof, we indicate the spatial components and write $ x_{i_a}=(y_{i_a},z_{i_a})\in\R\times\R$ and $ x_{j_b}=(y_{j_b},z_{j_b})\in\R\times\R$ for $a\in\{1,\ldots,u\}$ and $b\in\{1,\ldots,v\}$. Without loss of generality, let us then determine the truncated set when $(t_{j_b},y_{j_b},z_{j_b})=(0,0,0)$. To this end, we use four additional ambit sets that are translated by a value $\psi>0$ along both spatial axis, namely, the cones $A_{0}(\psi,\psi)$, $A_{0}(\psi,-\psi)$, $A_{0}(-\psi,\psi)$ and $A_{0}(-\psi,-\psi)$, as illustrated in Figure \ref{figure:truncatedformstoud2}-(c) for $c\leq1$ and in Figure \ref{figure:truncatedformstoud2}-(d) for $c>1$). Then, we set the truncated integration set to $B_0^{\psi}(0)=A_{0}(0,0)\backslash( A_{0}(\psi,\psi)\cup A_{0}(\psi,-\psi)\cup A_{0}(-\psi,\psi)\cup A_{0}(-\psi,-\psi))$.
		Since $(t_{i_a},y_{i_a},z_{i_a})\in V_{(0,0,0)}^r$, it is sufficient to choose $\psi$ such that the integration set of $\bb{Z}_{0}^{(\psi)}(0,0)$ is a subset of $(V_{(0,0,0)}^r)^c$, i.e. \begin{align}\label{eq:conditiononB}
			\sup_{b\in B_0^{\psi}(0)} \lVert b \rVert_\infty \leq r.
		\end{align}
		In the following we prove that the choice $\psi = r\min(1,c/\sqrt{2})$ is sufficient for (\ref{eq:conditiononB}) to hold. We investigate cross sections of the truncated integration set $B_0^{\psi}(0)$ along the time axis. For a fixed time point $t$, we call this cross-section $B^t$ and, similarly, we denote the cross-section of an ambit set by $A^t$. Note that the cross sections of our ambit sets along the time axis are circles with radius $|ct|$ (see also Figure \ref{figure:crosssections}).
		\begin{enumerate}
			\item[$t \in \Big{(} \tfrac{-\psi}{\sqrt{2}c},0 \Big{]}$:] As the distance between the center of the circle $A_0^t(0,0)$ and the centers of the circles $A_0^t(\psi,\psi)$, $A_0^t(-\psi,\psi)$, $A_0^t(\psi,-\psi)$, $A_0^t(-\psi,-\psi)$ is $\sqrt{2}\psi$, respectively, the set $A_{0}^t(0,0)$ (which is a circle with radius $|ct|$) is disjoint from every of the additional ambit sets' cross-sections at $t$ and hence $B^t= A^t_0(0,0)$ (see Figure \ref{figure:crosssections}-(a)). Clearly, we obtain 
			\begin{align}\label{eq:supB1}
				\sup_{t \in  \big{(}-\psi/(\sqrt{2}c),0 \big{]} } 
				\sup_{b\in B^t}\lVert b \rVert =\sup_{t \in  \big{(}-\psi/(\sqrt{2}c),0 \big{]} }  \max(c|t|,|t|)= \max\left(\frac{\psi}{\sqrt{2}},\frac{\psi}{\sqrt{2} c} \right).
			\end{align}
			\item[$t \in \Big{(} \tfrac{-\psi}{c},\tfrac{-\psi}{\sqrt{2}c} \Big{]}$:] For such $t$ the set $A_{0}^t(0,0)$ intersects with every additional ambit sets' cross-section (see Figure \ref{figure:crosssections}-(b)). However, as the additional ambit sets' cross-sections do not intersect with each other, the point $p_1(t)=(t,c|t|,0)\in B^t$ on the boundary of $A_0^t(0,0)$ (see the red point in Figure \ref{figure:crosssections}-(b)) is not excluded from $B^t$ by any additional ambit set. Note that symmetry makes it sufficient to look at $p_1(t)$. Hence, we obtain 
			\begin{align}\label{eq:supB2}
				\sup_{t \in  \big{(}-\psi/c, -\psi/(\sqrt{2}c) \big{]} } \sup_{b\in B^t}\lVert b \rVert =\sup_{t \in  \big{(}-\psi/c, -\psi/(\sqrt{2}c) \big{]} } \lVert p_1(t) \rVert = \max\left(\psi,\frac{\psi}{c} \right).
			\end{align}
			\item[$t \in \Big{(} \tfrac{-\sqrt{2}\psi}{c} ,\tfrac{-\psi}{c} \Big{]}$:] For such $t$ the set $A_{0}^t(0,0)$ intersects with every additional ambit sets' cross-section. Such intersection additionally restrict $A_0^t(0,0)$ (see Figure \ref{figure:crosssections}-(c)). Straightforward calculations show that the point where the boundaries of $A_0^t(-\psi,\psi)$ and $A_0^t(-\psi,\psi)$ as well as the set $A_0^t(0,0)$ intersect, say $p_2(t)$, is given by $(t,0,\psi - \sqrt{(ct)^2-\psi^2})$ (see red point in Figure \ref{figure:crosssections}-(c)). Note that symmetry makes it sufficient to look at $p_2(t)$. We obtain 
			\begin{align}\label{eq:supB3}
				\sup_{t \in  \big{(} -\sqrt{2}\psi/c,-\psi/c \big{]} } \sup_{b\in B^t}\lVert b \rVert =\sup_{t \in  \big{(} -\sqrt{2}\psi/c,-\psi/c \big{]} }  \lVert p_2(t) \rVert = \max\left(\psi,\frac{\sqrt{2}\psi}{c} \right).
			\end{align}
			\item[$t \leq\tfrac{-\sqrt{2}\psi}{c}$:] In the following, we show that for such $t$, the set $A^t(0,0)$ is entirely included in the union of the additional ambit sets' cross sections. Clearly, this is true if the upper point where the boundaries of $A_0^t(-\psi,\psi)$ and $A_0^t(-\psi,\psi)$ intersect, say $p_3(t)$, is outside of $A_0^t(0,0)$ (see the red point in Figure \ref{figure:crosssections}-(d)). Note that symmetry makes it sufficient to look at $p_3(t)$.  As straightforward calculations show that $p_3(t)= (t,0,\psi + \sqrt{(ct)^2-\psi^2})$, this is true if $\psi + \sqrt{(ct)^2-\psi^2} \geq c |t|$, or equivalently $(\psi + \sqrt{(ct)^2-\psi^2})^2 \geq (ct)^2$. Moreover, we have
			\begin{align}\label{eq:supB4}
				\psi^2 + 2\psi\sqrt{(ct)^2-\psi^2} + (ct)^2-\psi^2\geq (ct)^2 \iff \psi\geq 0.
			\end{align}
		\end{enumerate}
		In view of condition (\ref{eq:conditiononB}) we combine (\ref{eq:supB1}), (\ref{eq:supB2}) and (\ref{eq:supB3}) and set $\psi = r \min(1,c/\sqrt{2})$, which also satisfies (\ref{eq:supB4}).
		
		In addition to the cross sectional views from Figure \ref{figure:crosssections}, we give a full three-dimensional view of the set $B_0^{\psi}(0)$ for $c\leq1$ in Figure \ref{figure:truncatedformstoud2}-(e) and for $c>1$ in Figure \ref{figure:truncatedformstoud2}-(f) that highlight the points on the boundary of $B_0^{\psi}(0)$ with maximal $\infty$-norm for $\psi = r \min(1,c/\sqrt{2})$. \\
		
		\noindent  Therefore, because of Proposition \ref{proposition:MMAmoments}-(ii), we can conclude that
		\begin{small}
			\begin{align*}
				&\theta_{lex}(r) \leq 2 Var(\Lambda^{\prime})^{1/2}\left( \int_0^{\infty} \int_{A_0(0,0)\cap (A_0(\psi,\psi)\cup A_0(-\psi,\psi)\cup A_0(\psi,-\psi)\cup A_0(-\psi,-\psi))}f(A,-s)^2 ds d\xi \pi(d\lambda) \right)^{1/2}\\
				&=2 Var(\Lambda^{\prime})^{1/2}\Bigg( \int_0^{\infty} \int_{A_0(0,0)\cap A_0(\psi,\psi) }f(A,-s)^2 ds d\xi \pi(d\lambda)\\
				&\qquad+\int_0^{\infty} \int_{A_0(0,0)\cap A_0(-\psi,\psi) }f(A,-s)^2 ds d\xi \pi(d\lambda)\\
				&\qquad+\int_0^{\infty} \int_{A_0(0,0)\cap A_0(\psi,-\psi) }f(A,-s)^2 ds d\xi \pi(d\lambda)\\
				&\qquad\qquad\qquad+\int_0^{\infty} \int_{A_0(0,0)\cap A_0(-\psi,-\psi) }f(A,-s)^2 ds d\xi \pi(d\lambda)\\
				&\qquad-\int_0^{\infty} \int_{A_0(0)\cap  A_0(\psi,-\psi)\cap A_0(-\psi,-\psi)}f(A,-s)^2 ds d\xi \pi(d\lambda)\\
				&\qquad -\int_0^{\infty} \int_{A_0(0)\cap  A_0(-\psi,\psi)\cap (A_0(\psi,-\psi)\cup A_0(-\psi,-\psi))}f(A,-s)^2 ds d\xi \pi(d\lambda)\\
				&\qquad -\int_0^{\infty} \int_{A_0(0)\cap  A_0(\psi,\psi)\cap (A_0(-\psi,\psi)\cup A_0(\psi,-\psi)\cup A_0(-\psi,-\psi))}f(A,-s)^2 ds d\xi \pi(d\lambda) \Bigg)^{1/2}\\
				&\leq 2  \sqrt{2Cov(\bb{Z}_0(0,0),\bb{Z}_0(\psi,\psi))+2Cov(\bb{Z}_0(0,0),\bb{Z}_0(\psi,-\psi))}.
			\end{align*}
		\end{small}
		
	\end{enumerate}
	which converges to zero as $r \to \infty$ for the dominated convergence theorem.
\end{proof}

\begin{figure}
	\centering
	\subfigure[\label{Plot15} Cross sections of the auxiliary ambit sets and $A_0^t(0,0)$ for $t=-1$ and $c=2$.]{\includegraphics[width=.40\textwidth]{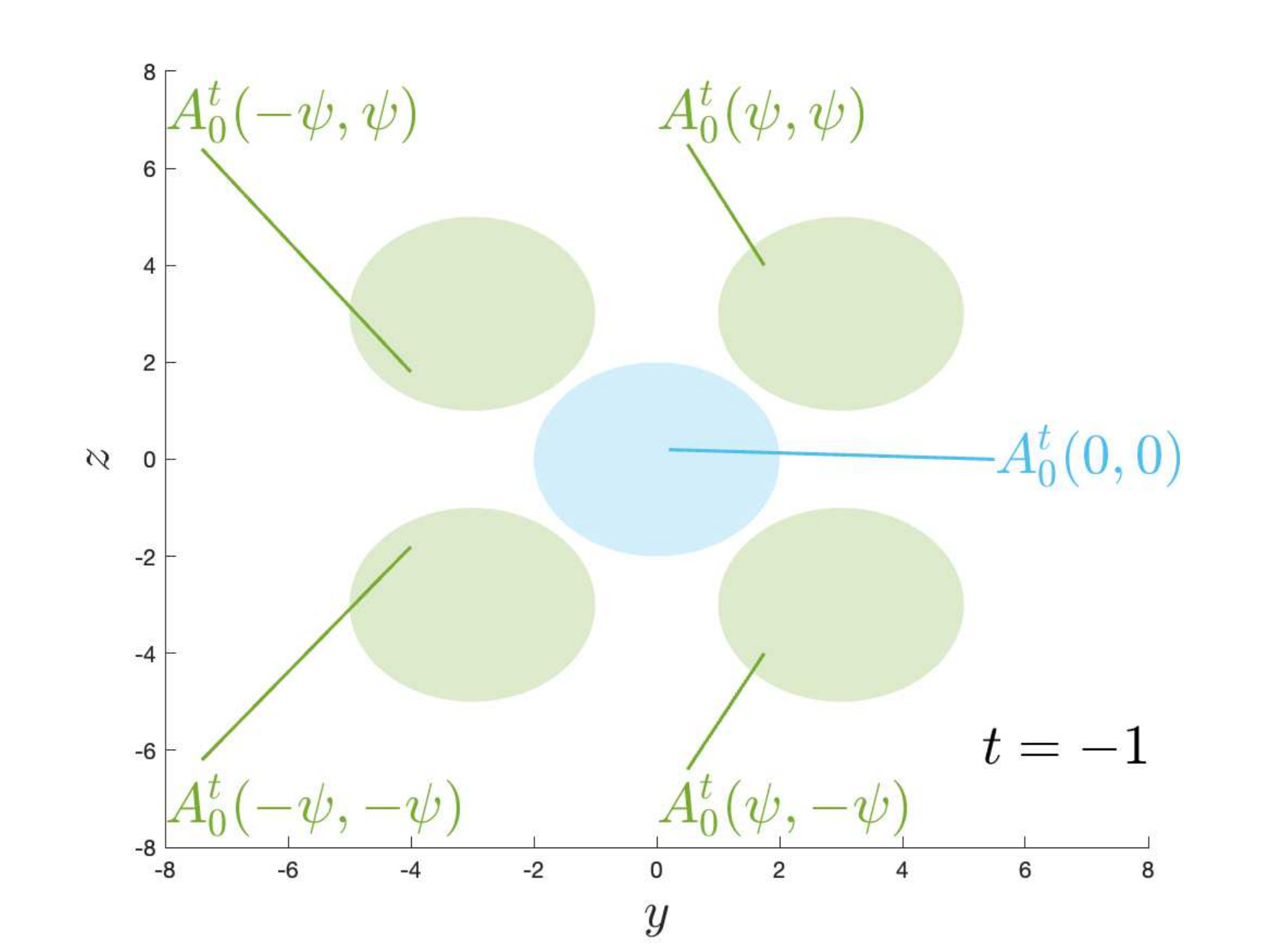}} \hspace{.02\textwidth}
	\subfigure[\label{Plot16} Cross sections of the auxiliary ambit sets and $A_0^t(0,0)$ for $t=-5/4$ and $c=2$.]{\includegraphics[width=0.40\textwidth]{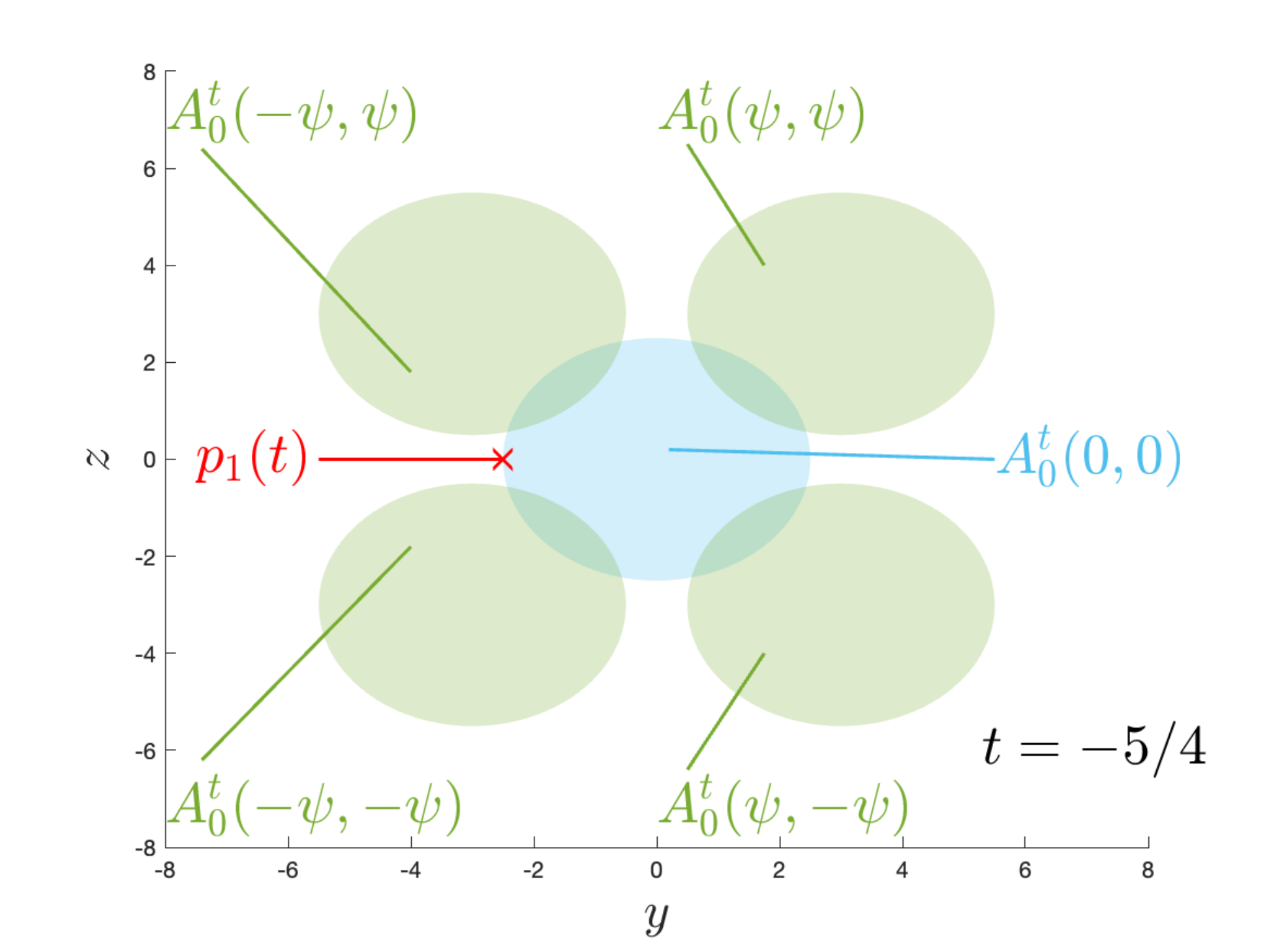}} \hspace{.02\textwidth}
	\subfigure[\label{Plot17} Cross sections of the auxiliary ambit sets and $A_0^t(0,0)$ for $t=-7/4$ and $c=2$.]{\includegraphics[width=0.40\textwidth]{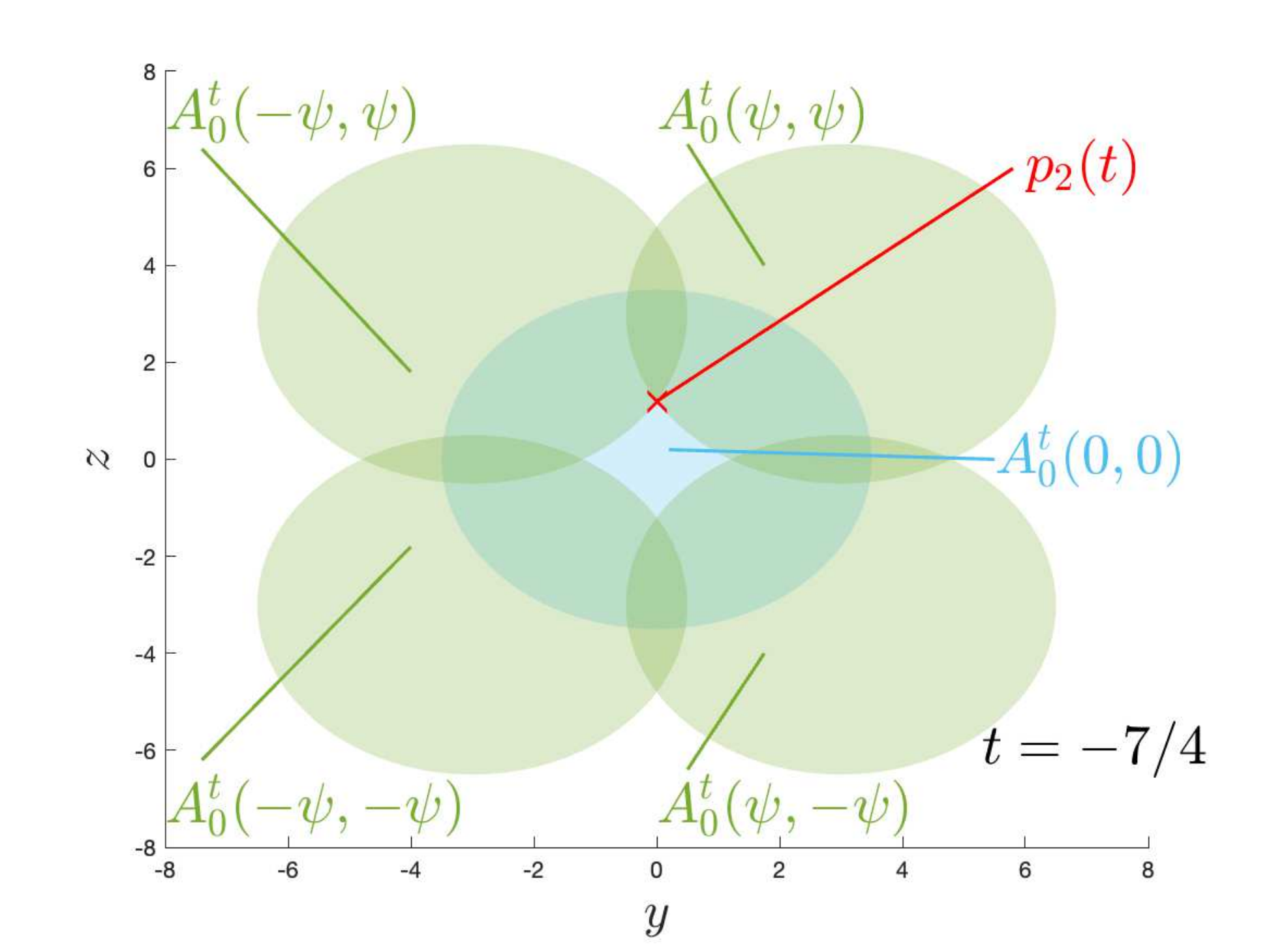}} \hspace{.02\textwidth}
	\subfigure[\label{Plot18} Cross sections of the auxiliary ambit sets and $A_0^t(0,0)$ for $t=-2.2$ and $c=2$.]{\includegraphics[width=.40\textwidth]{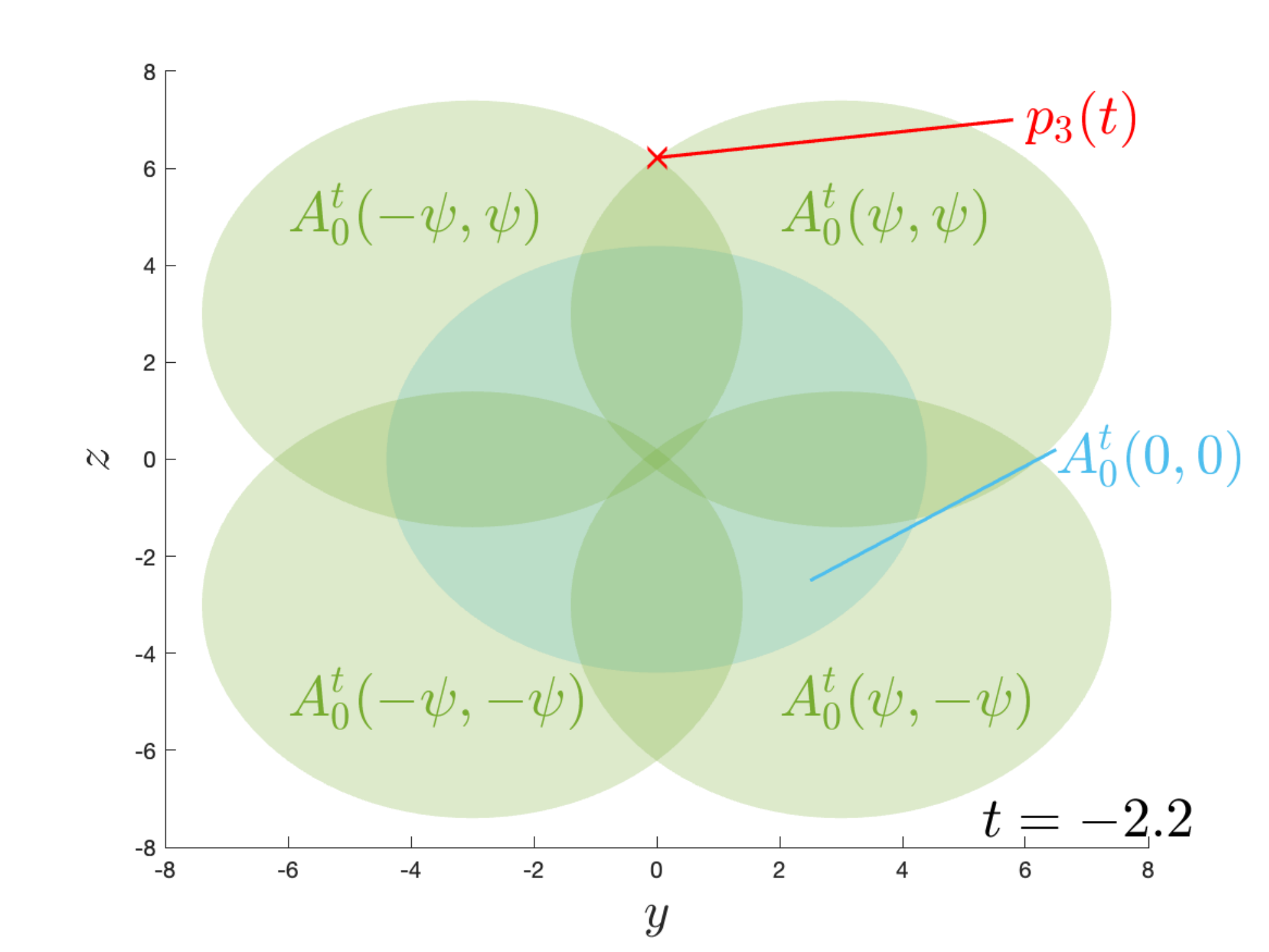}}\vspace{-0.2cm}
	\caption{Cross sections of the auxiliary ambit sets and $A_0^t(0,0)$ at different time points for $\psi =3$.}\label{figure:crosssections}
\end{figure}

\begin{figure}
	\centering
	\subfigure[\label{Plot1d2} Integration set $A_0(0,0)$ for $c=1/\sqrt{2}$ together with the complement of  $V_{(0,0,0)}^r$ for $r=3$.]{\includegraphics[width=.48\textwidth]{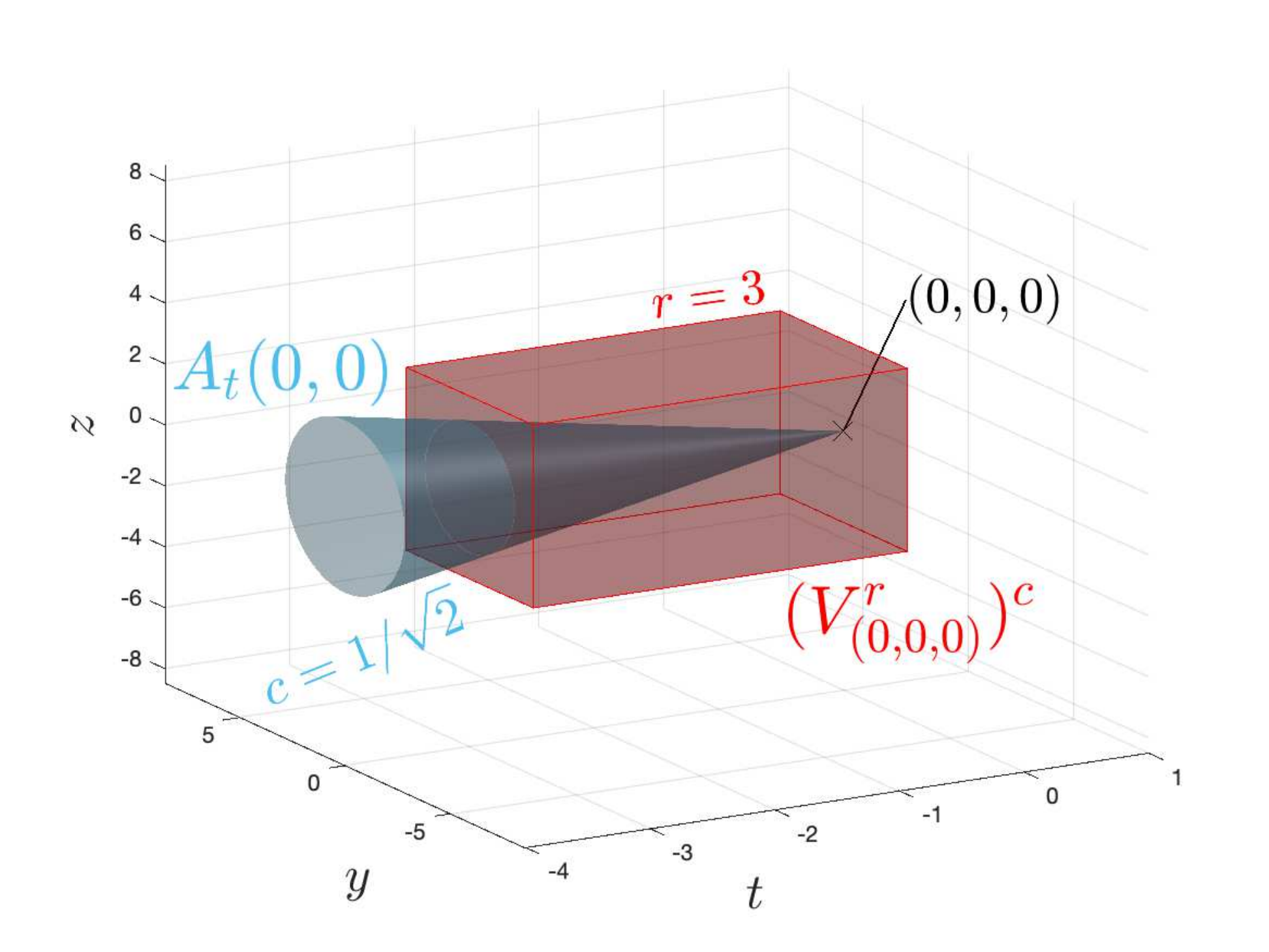}} \hspace{.02\textwidth}
	\subfigure[\label{Plot2d2}  Integration set $A_0(0,0)$ for $c=2$ together with the complement of  $V_{(0,0,0)}^r$ for $r=3$.]{\includegraphics[width=0.48\textwidth]{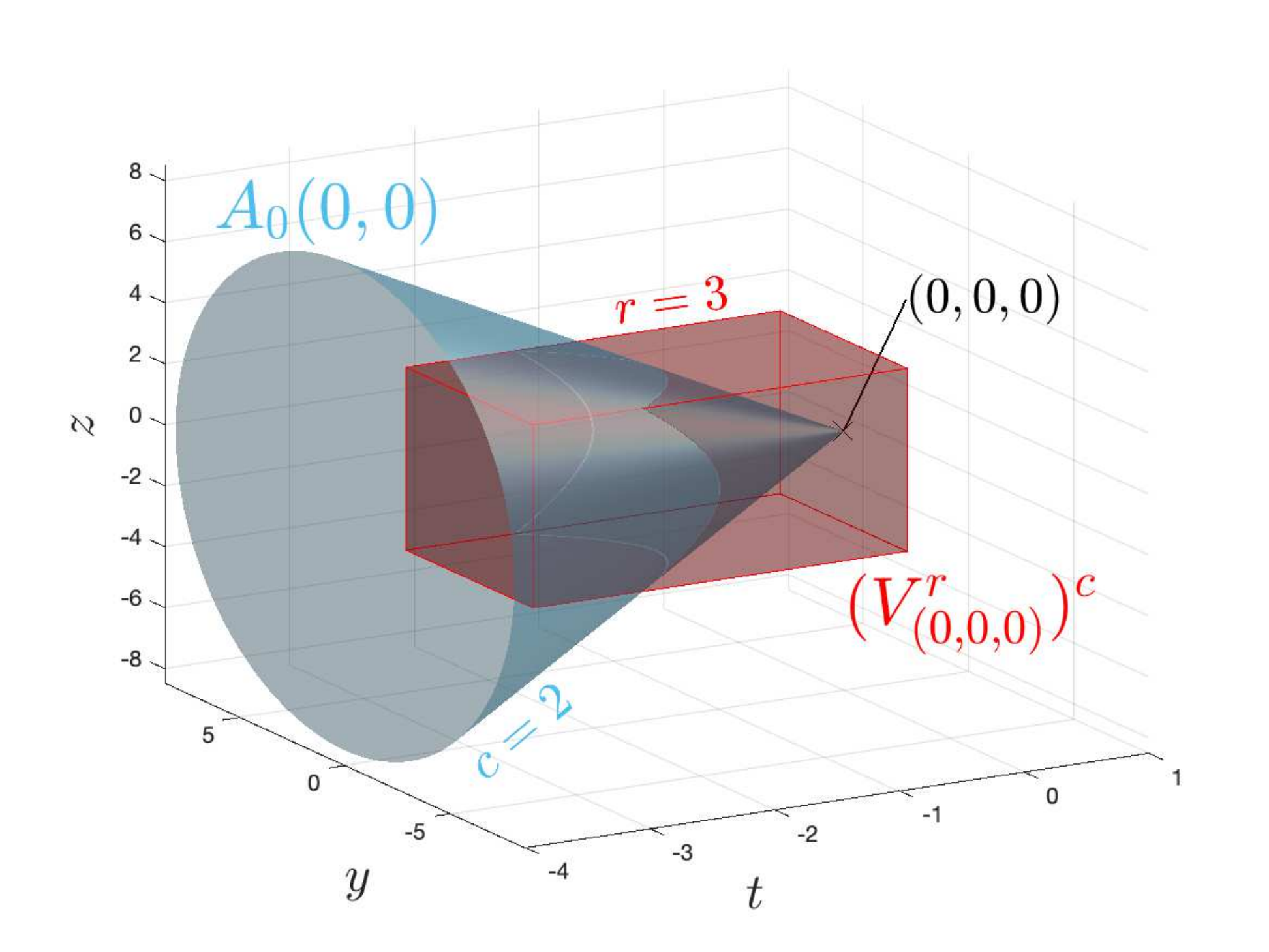}} \hspace{.02\textwidth}
	\subfigure[\label{Plot3d2} Integration set $A_0(0,0)$ together with $A_0(\psi,\psi)$, $A_0(-\psi,\psi)$, $A_0(\psi,-\psi)$ and $A_0(-\psi,-\psi)$ for $\psi=r\min(1,c/\sqrt{2})$.]{\includegraphics[width=.48\textwidth]{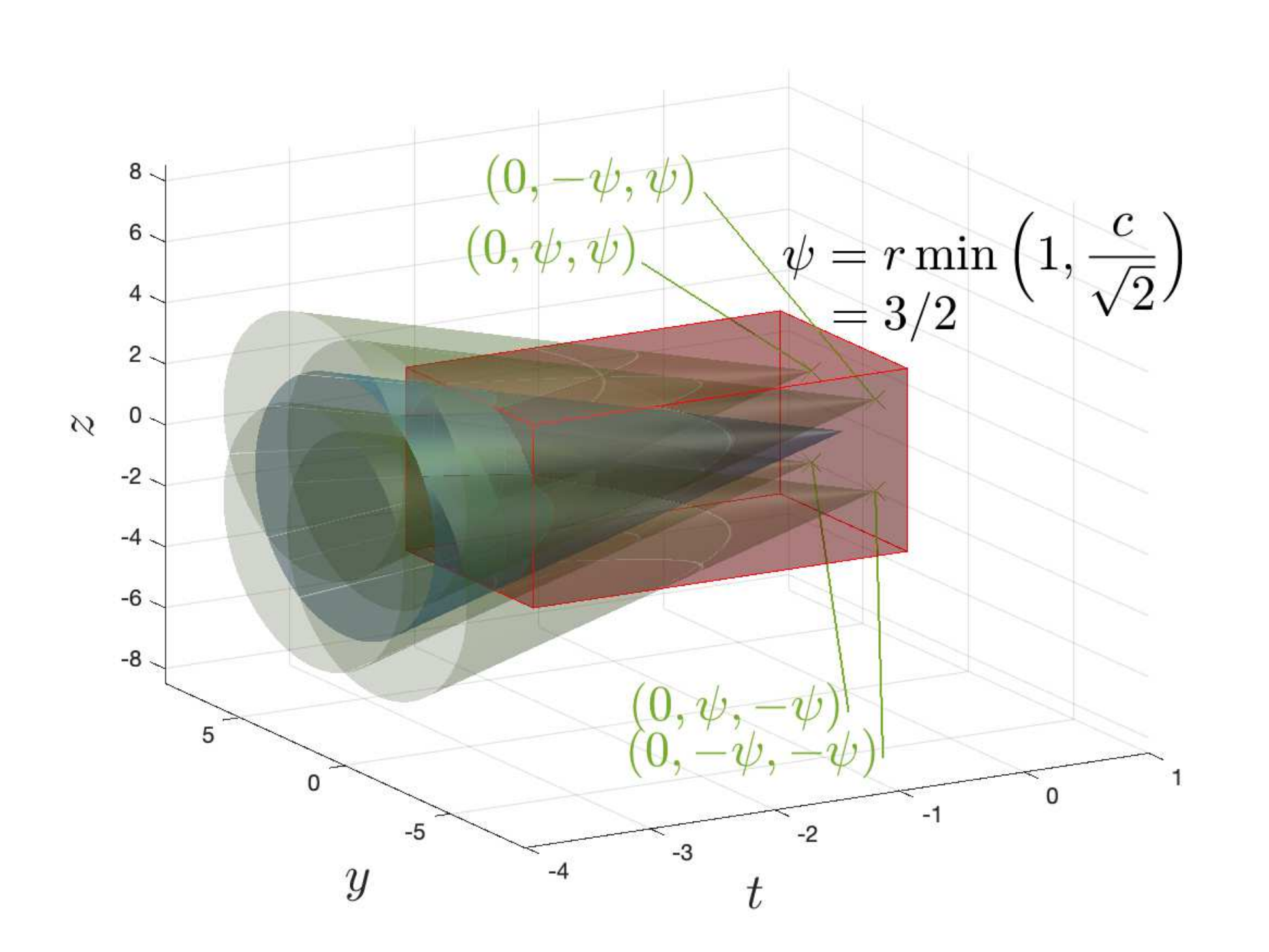}}\hspace{.02\textwidth}
	\subfigure[\label{Plot4d2} Integration set $A_0(0,0)$ together with $A_0(\psi,\psi)$, $A_0(-\psi,\psi)$, $A_0(\psi,-\psi)$ and $A_0(-\psi,-\psi)$ for $\psi=r\min(1,c/\sqrt{2})$.]{\includegraphics[width=.48\textwidth]{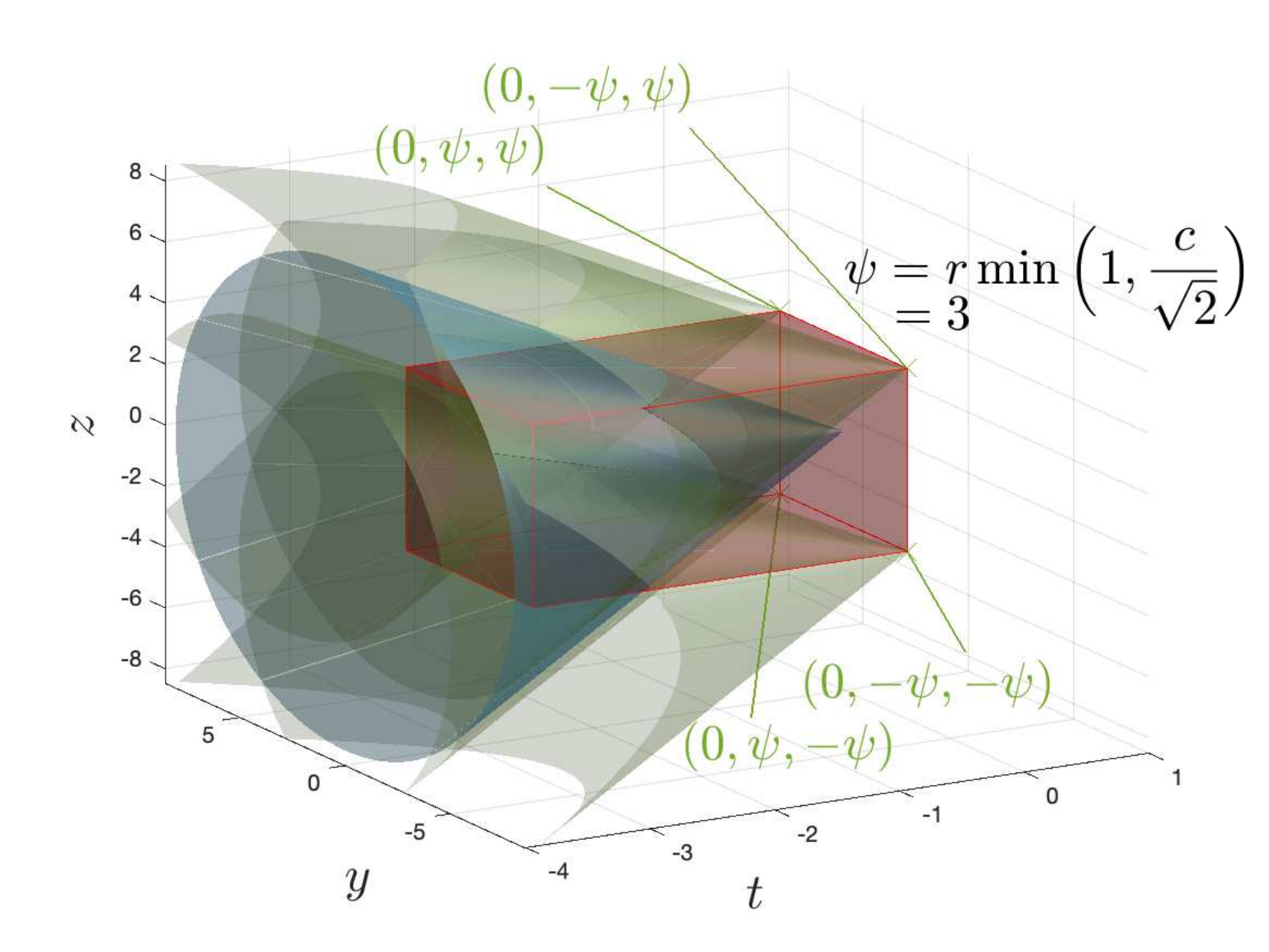}} \hspace{.02\textwidth}
	\subfigure[\label{Plot7d2} $B$ is the integration set of $\bb{Z}_0^{(\psi)}(0,0)$. In addition, we illustrate $A_0(-\psi,-\psi)$ for $c=1/\sqrt{2}$ together with the complement of  $V_{(0,0,0)}^r$ for $r=3$.]{\includegraphics[width=.48\textwidth]{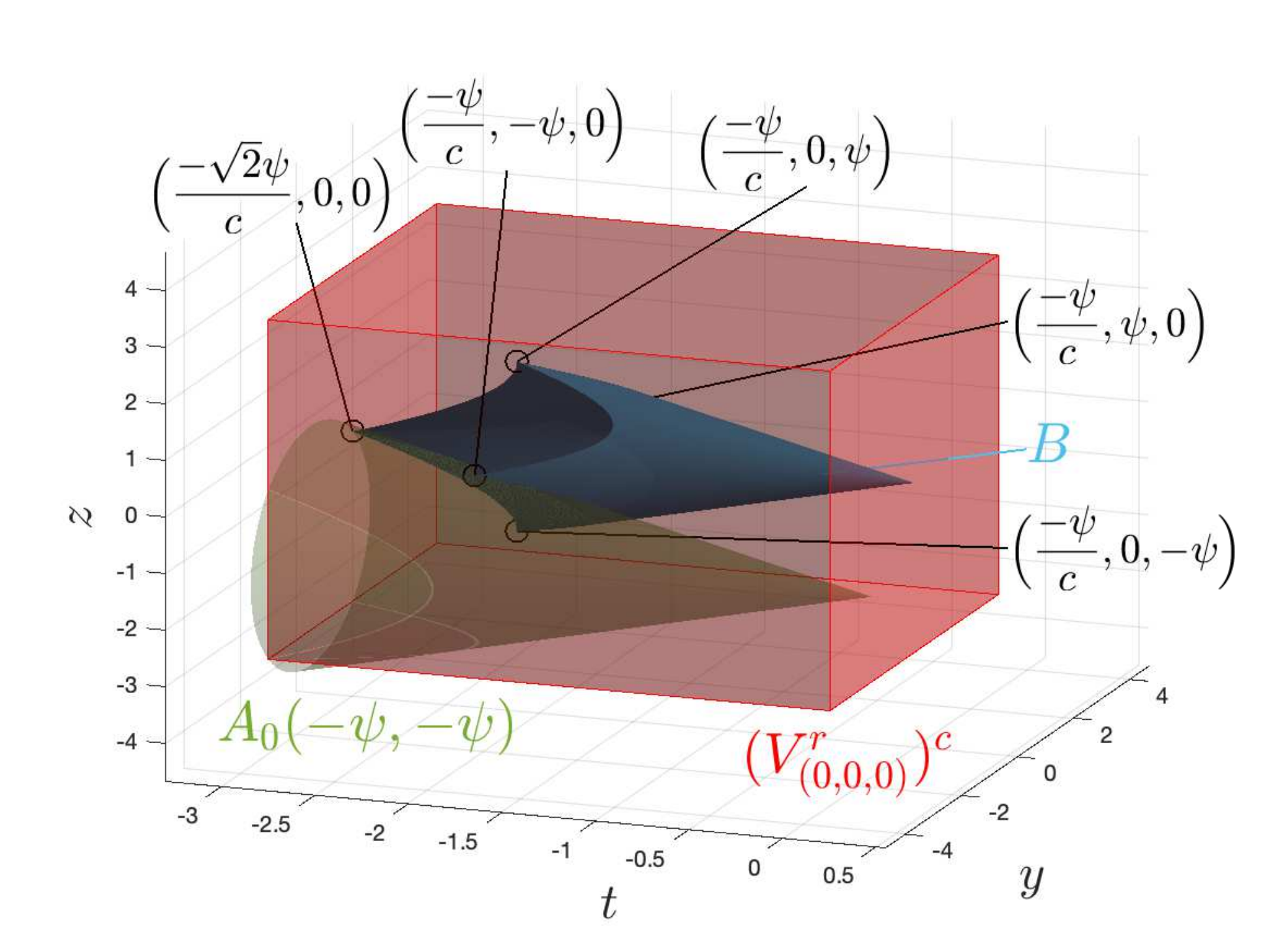}} \hspace{.02\textwidth}
	\subfigure[\label{Plot8d2} $B$ is the integration set of $\bb{Z}_0^{(\psi)}(0,0)$. In addition, we illustrate $A_0(-\psi,-\psi)$ for $c=2$ together with the complement of  $V_{(0,0,0)}^r$ for $r=3$.]{\includegraphics[width=0.48\textwidth]{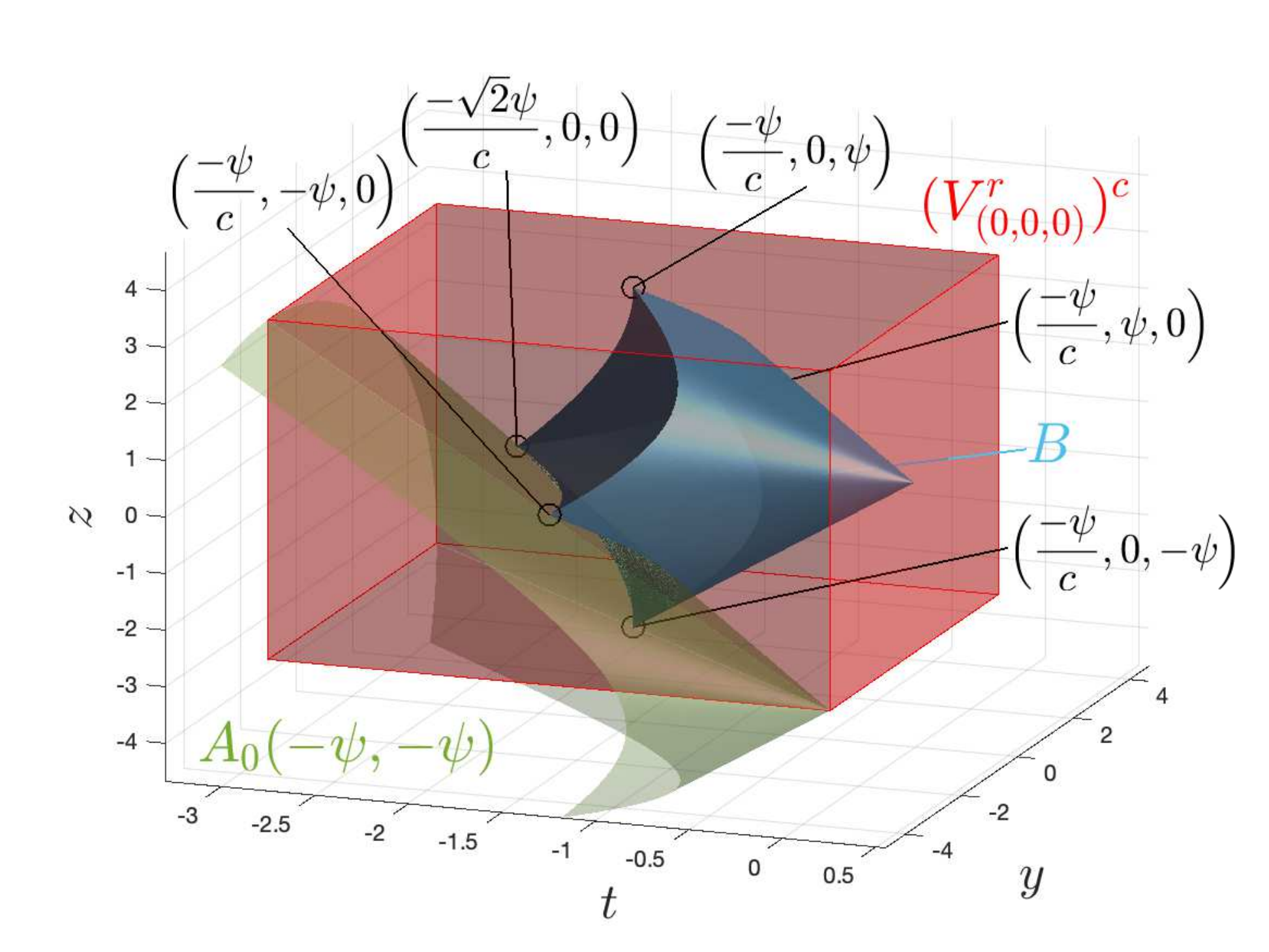}} \vspace{-0.5cm}
	\caption{Integration set and truncated integration set of an MMAF $\bb{Z}_{t}(y,z)$ for $d=2$.}\label{figure:truncatedformstoud2}
\end{figure}

\subsection{Proofs of Section \ref{sec4}}

\begin{proof}[Proof of Proposition \ref{heredithary}]
	We drop the bold notations indicating random fields and stochastic processes in the following.
	Let $h \in \mathcal{H}$, we call $L_i=L(h(X_i),Y_i)$ for $i \in \Z$, $Z_t^{(M)}(x):=Z_t(x) \vee (-M) \wedge M$ for $M>1$, and $L^{(M)}_i=L(h(X_i^{(M)}),Y_i^{(M)})$ where  
	\[
	X_i^{(M)}=L^{-(M)}_p(t_0+ia,x^*), \, \,\,\, \textrm{and} \,\,\,\,\, Y_i^{(M)}=Z_{t_0+ia}^{(M)}(x^*), \,\,\,\textrm{for $i \in \Z$},
	\]
	and
	\[
	L^{-(M)}_p(t,x^*)=\{Z_{s}^{(M)}(\xi): (s,\xi)\in \Z\times \mathbb{L}, \,\, \|x^*-\xi \| \leq c \, (t-s)\,\, \textrm{and}\,\, t-s\leq p \}.
	\]
	for $t=t_0+ia$ with $i \in \Z$.
	For $u \in \N$, $i_1\leq i_2 \leq \ldots \leq i_u < i_u+k = j$ with $k \in \N$, let us consider the marginal of the field
	
	\begin{equation}
		\label{sample_gen}
		\Big((X_{i_1},Y_{i_1}),\ldots, (X_{i_u},Y_{i_u}),(X_{j},Y_{j}) \Big),
	\end{equation}
	and let us define
	\[
	\Gamma=\{\textrm{$(t_i,x_i) \in \Z^{1+d}$: $Z_{t_i}(x_i) \in L_p^-(t_0+{i_s}a,x^*)$ or $ (t_i,x_i)=(t_0+{i_s}a,x^*)$ for $s=1,\ldots,u$}  \},
	\]
	and 
	\[
	\Gamma^{\prime}=\{ \textrm{$(t_i,x_i) \in \Z^{1+d}$: $Z_{t_i}(x_i) \in L_p^-(t_0+{j}a,x^*)$ or $ (t_i,x_i)=(t_0+{j}a,x^*)$} \}.	
	\]
	Then $r=dist(\Gamma,\Gamma^{\prime})$. In particular $\Gamma \in V_{\Gamma^{\prime}}^{r}$, and  $r =(j-i_u)a-p$.
	For $F \in \mathcal{G}^*_u$ and $G \in \mathcal{G}_1$, then
	\begin{align}
		|Cov(F(L_{i_1}, \ldots,L_{i_u}), G(L_j))|& \label{total} \\
		\leq &|Cov(F(L_{i_1}, \ldots,L_{i_u}), G(L_j)-G(L_j^{(M)}))| \label{quattro_gen}\\
		& + |Cov(F(L_{i_1}, \ldots,L_{i_u}), G(L_j^{(M)}))|. \label{cinque_gen}
	\end{align}
	
	The summand (\ref{quattro_gen}) is less than or equal to
	\begin{align*}
		2 \|F\|_{\infty} Lip(G) \E[|L_j-L_j^{(M)}|] &\leq 2 \|F\|_{\infty} Lip(G) ( \E[|Y_j-Y_j^{(M)}|] +\E[|h(X_j)-h(X_j^{(M)})|] ) \\
		& \leq 2 \|F\|_{\infty} Lip(G) (Lip(h) a(p,c)+1) \E[|Z_{t_1}(x_1)-Z_{t_1}^{(M)}(x_1)|]
	\end{align*}
	by stationarity of the field Z, and because $L$ and $h$ are Lipschitz functions.
	Moreover, the function $G(L_j^{(M)})$ belongs to $\mathcal{G}_{a(p,c)+1}$. Let $(X,Y), (X^{\prime},Y^{\prime}) \in \R^{a(p,c)+1}$, then
	\begin{align*}
		|G(L(h(X^{(M)}), Y^{(M)})) & -G(L(h(X^{\prime(M)}), Y^{\prime(M)})) | \leq Lip(G) |L(h(X^{(M)}), Y^{(M)})-L(h(X^{\prime(M)}), Y^{\prime(M)})|\\
		&\leq Lip(G) (|h(X^{(M)}))-h(X^{\prime(M)})| + |Y^{(M)}-Y^{\prime(M)}|\\
		&\leq Lip(G) (Lip(h)+1) (\|X-X^{\prime}\|_1 + |Y-Y^{\prime}|),
	\end{align*}
	and $Lip(G(L_j^{(M)})) \leq  Lip(G) (Lip(h)+1) $.
	
	Because $Z$ is a $\theta$-lex weakly dependent random field, (\ref{cinque_gen}) is less than or equal to
	\[
	\tilde{d} \|F\|_{\infty}Lip(G)(Lip(h)+1) \theta_{lex}(r).
	\]
	We choose now $M=r$ and obtain that (\ref{total}) is less than or equal than
	\[
	\|F\|_{\infty} Lip(G) \tilde{d}(Lip(h)a(p,c)+1) \Big( \frac{2}{\tilde{d}} \E[|Z_{t_1}(x_1)-Z_{t_1}^{(r)}(x_1)|] +\theta_{lex}(r) \Big).
	\]
	The quantity above converges to zero as $r \to \infty$. Therefore, $(L_i)_{i \in \Z}$ is a $\theta$-weakly dependent process.
\end{proof}

\begin{Remark}
	Note that when working with a general $\theta$-weakly dependent random field, as for example in Proposition \ref{heredithary}, we do not employ the truncated random field defined in (\ref{truncated_rf}). This happens because a $\theta$-lex weakly dependent field is not generally defined as an integral driven by a L\'evy basis on an ambit set. The field (\ref{truncated_rf}) can just be employed in the MMAF framework. In a more general framework, without further information on the structure of the field, we can just employ the truncated random field $\bb{Z}_t^{(M)}(x):=\bb{Z}_t(x) \vee (-M) \wedge M$. 
\end{Remark}

\begin{proof}[Proof of Proposition \ref{mmaf_absolute}]
	We drop the bold notations indicating random fields and stochastic processes in the following.
	We call $L_i=L(h(X_i),Y_i)$ for $i \in \Z$, as defined in Proposition \ref{proposition:mmathetaweaklydep}. Moreover we employ the truncated field $Z_t^{(\psi)}(x)$ and define $L^{(\psi)}=(L(h(X_i^{(\psi)}),Y_i^{(\psi)}))_{i \in \Z}$ where  
	\[
	X_i^{(\psi)}=L^{-(\psi)}_p(t_0+ia,x^*), \, \,\,\, \textrm{and} \,\,\,\,\, Y_i^{(\psi)}=Z_{t_0+ia}^{(\psi)}(x^*), \,\,\,\textrm{for $i \in \Z$},
	\]
	and
	\[
	L^{-(\psi)}_p(t,x^*)=\{Z_{s}^{(\psi)}(\xi): (s,\xi)\in \Z\times \mathbb{L}, \,\, \|x^*-\xi \| \leq c \, (t-s)\,\, \textrm{and}\,\, t-s\leq p \}.
	\]
	for $t=t_0+ia$ with $i \in \Z$.
	For $u \in \N$, $i_1\leq i_2 \leq \ldots \leq i_u < i_u+k = j$ with $k \in \N$, let us consider the marginal of the field
	
	\begin{equation}
		\label{sample}
		\Big((X_{i_1},Y_{i_1}),\ldots, (X_{i_u},Y_{i_u}),(X_{j},Y_{j}) \Big),
	\end{equation}
	and let us define
	\[
	\Gamma=\{\textrm{$(t_i,x_i) \in \Z^{1+d}$: $Z_{t_i}(x_i) \in L_p^-(t_0+{i_s}a,x^*)$ or $ (t_i,x_i)=(t_0+{i_s}a,x^*)$ for $s=1,\ldots,u$}  \},
	\]
	and 
	\[
	\Gamma^{\prime}=\{ \textrm{$(t_i,x_i) \in \Z^{1+d}$: $Z_{t_i}(x_i) \in L_p^-(t_0+{j}a,x^*)$ or $ (t_i,x_i)=(t_0+{j}a,x^*)$} \}.	
	\]
	Then $r=dist(\Gamma,\Gamma^{\prime})$. In particular $\Gamma \in V_{\Gamma^{\prime}}^{r}$, and  $r =(j-i_u)a-p$.
	For $F \in \mathcal{G}^*_u$ and $G \in \mathcal{G}_1$, then
	\begin{align}
		|Cov(F(L_{i_1}, \ldots,L_{i_u}), G(L_j))|& \nonumber \\
		\leq &|Cov(F(L_{i_1}, \ldots,L_{i_u}), G(L_j)-G(L_j^{(\psi)}))| \label{quattro}\\
		& + |Cov(F(L_{i_1}, \ldots,L_{i_u}), G(L_j^{(\psi)}))|. \label{cinque}
	\end{align}
	The summand (\ref{cinque}) is equal to zero because $\Gamma \in V_{\Gamma^{\prime}}^{r}$, see proof of Proposition \ref{example:spatiotemporaldata} for more details about this part of the proof.
	We can then bound (\ref{quattro}) from above by
	\begin{align}
		2 \|F\|_{\infty} Lip(G) \E[|L_j-L_j^{(\psi)}|] 
		& \leq 2 \|F\|_{\infty} Lip(G) (\E[|Y_j-Y_j^{(\psi)}|] + \E[|h(X_j)-h(X_j^{(\psi)})|]) \label{bound1}\\
		& \leq 2 \|F\|_{\infty} Lip(G) (Lip(h)a(p,c)+1) \E[|Z_{t_1}(x_1)-Z_{t_1}^{(\psi)}(x_1)|] \label{bound2}
	\end{align}
	where (\ref{bound1}) holds because $L$ is a function with Lipschitz constant equal to one, and (\ref{bound2}) holds given that $h$ is Lipschitz. 
	
	When we work with linear functions, we consider $h$ parameterized with respect to the set $B \in \R^{a(p,c)}$
	\begin{equation}
		\label{bound4}
		\E[|h_{\beta}(X_j)-h_{\beta}(X_j^{(\psi)})\|]=\E \Big[\Big|\sum_{l=1}^{a(p,c)} \beta_{1,l} (Z_{t_l}(x_l)-Z_{t_l}^{(\psi)}(x_l)) \Big| \Big].
	\end{equation}
	By stationarity of the field $Z$, we have that (\ref{bound4}) is smaller or equal than $\|\beta_1\|_1 \E[|Z_{t_1}(x_1)-Z_{t_1}^{(\psi)}(x_1)|]$.
	Overall, we have that (\ref{quattro}) is smaller or equal than
	\begin{align*}
		&2 \|F\|_{\infty} Lip(G) (Lip(h)a(p,c) +1) \E[|Z_{t_1}(x_1)-Z_{t_1}^{(\psi)}(x_1)|] \,\,\, \textrm{for $h$ a Lipschitz function,} 
	\end{align*}
	or it is smaller or equal than
	\begin{align*}
		&2 \|F\|_{\infty} Lip(G) (\|\beta_1\|_1 +1) \E[|Z_{t_1}(x_1)-Z_{t_1}^{(\psi)}(x_1)|] \,\,\, \textrm{for $h_{\beta}$ a linear function}.
	\end{align*}
	Because of the properties of the truncated field $Z_{t_1}^{(\psi)}(x_1)$, we have that the above bounds converge to zero as $r \to \infty$. Therefore, $L$ is a $\theta$-weakly dependent process.
\end{proof}

The proof of Theorem \ref{new} uses a blocks technique introduced in the papers \cite{Cmixing} and \cite{Thelight}. Such results are based on the use of several lemmas. To ease the complete understanding of the proof of Theorem \ref{new}, we prove these Lemmas below, given that they undergo several modifications in our framework. Let us start by partitioning a set $\{1,2,\ldots,m\}$ into $k$ blocks. Each block will contain $l=\lfloor\frac{m}{k} \rfloor$ terms. Let $h=m-k \,l<k$ denote the remainder when we divide $m$ by $k$. We now construct $k$ blocks such that the number of elements in the $j$th-block is defined by
\[
\bar{l}_j=\left\{ \begin{array}{ll}
	l+1 \, \,\, &\textrm{if} \, \,\,j=1,2,\ldots,h	\\
	l \, \, &\textrm{if} \,\,\, j=h+1,\ldots,k
\end{array} \right..
\]
Let $(\bb{U}_i)_{i \in \Z}$ a stationary process, and $\bb{V}_m=\sum_{i=1}^m \bb{U}_i$, for $j=1,\ldots, k$ we define the $j$-th block as 
\[
\bb{V}_{j,m}=\bb{U}_j+\bb{U}_{j+k}+\ldots \bb{U}_{j+(\bar{l}_j-1)\,k}=\sum_{i=1}^{\bar{l}_j} \bb{U}_{j+(i-1) \,k}
\]
such that
\[
\bb{V}_m=\sum_{j=1}^k \bb{V}_{j,m}=\sum_{j=1}^k \sum_{i=1}^{\bar{l}_j} \bb{U}_{j+(i-1) \,k}.
\]
For $j=1,2,\ldots,k$, let us define $p_j=\frac{\bar{l}_j}{m}$. It follows that $\sum_{j=1}^k p_j=\frac{1}{m} \sum_{j=1}^k \bar{l}_j=1$.

\begin{Lemma}
	\label{lem1}
	For all $s \in \R$
	\[
	\E\Big[\exp\Big(s\frac{\bb{V}_m}{m}\Big)\Big] \leq \sum_{j=1}^k p_j \E\Big[\exp\Big(s \frac{\bb{V}_{j,m}}{\bar{l}_j}\Big) \Big]
	\]
\end{Lemma}

The proof of the result above is due to Hoeffding \cite{Hoeffding}.

\begin{Lemma}
	\label{lem2}
	Let the assumptions of Theorem \ref{new} hold and define the process $(\bb{U}_i)_{i \in \Z}$ such that $\bb{U}_i:=f(\bb{Z}_i)-\E[f(\bb{Z}_i)]$. For all $j=1,2,\ldots,k$, $l \geq2$ and $0<s<\frac{3l}{|b-a|}$
	\begin{equation}
		\label{Mj}
		M_{\bar{l}_j}(s)= \Big | \E \Big[ \prod_{i=1}^{\bar{l}_j} \exp\Big(\frac{s\, \bb{U}_{j+(i-1) \,k} }{\bar{l}_j}\Big)\Big] -\prod_{i=1}^{\bar{l}_j} \E \Big[ \exp\Big(\frac{s\, \bb{U}_{j+(i-1) \,k} }{\bar{l}_j}\Big)  \Big] \Big|\leq \exp( s \, |b-a|) \theta(k) s
	\end{equation}
	The same result holds when defining the process $(\bb{U}_i)_{i \in \Z}$ for $\bb{U}_i=\E[f(\bb{X}_i)] -f(\bb{X}_i)$.
\end{Lemma}

\begin{proof}
	Let us first discuss the case when $\bb{U}_i=f(\bb{X}_i)-\E[f(\bb{X}_i)]$, we have that the process  $\bb{U}$ has mean zero and $|\bb{U}_i|\leq |b-a|$.	
	Let us define $\mathcal{F}_j=\sigma(\bb{U}_i, i\leq j)$.
	\begin{align*}
		M_{\bar{l}_j}&:= \Bigg| \E\Big[ \prod_{i=1}^{\bar{l}_j} \exp\Big( \frac{s \bb{U}_{j+(i-1)k}}{\bar{l}_j}  \Big)  \Big] - \prod_{i=1}^{\bar{l}_j} \E \Big[ \exp\Big ( \frac{s \bb{U}_{j+(i-1)k}}{\bar{l}_j} \Big)  \Big]   \Bigg|\\
		&= \Bigg| \E\Big[ \prod_{i=1}^{\bar{l}_j-1} \exp\Big( \frac{s \bb{U}_{j+(i-1)k}}{\bar{l}_j}  \Big) \E \Big[ \exp \Big( \frac{s \bb{U}_{j+(\bar{l}_j -1)k}}{\bar{l}_j}  \Big) \Big | \mathcal{F}_{j+(\bar{l}_j-2)k} \Big] \Big] - \prod_{i=1}^{\bar{l}_j} \E \Big[ \exp\Big ( \frac{s \bb{U}_{j+(i-1)k}}{\bar{l}_j} \Big)  \Big]   \Bigg|\\
		&\leq \Bigg| \E\Big[ \prod_{i=1}^{\bar{l}_j-1} \exp\Big( \frac{s \bb{U}_{j+(i-1)k}}{\bar{l}_j}  \Big) \Big( \E \Big[ \exp \Big( \frac{s \bb{U}_{j+(\bar{l}_j -1)k}}{\bar{l}_j}  \Big) \Big | \mathcal{F}_{j+(\bar{l}_j-2)k} \Big]- \E \Big[  \exp \Big( \frac{s \bb{U}_{j+(\bar{l}_j -1)k}}{\bar{l}_j}  \Big) \Big] \Big) \Big] \Bigg| \\
		&+ \Bigg| \E \Big[ \exp \Big( \frac{s \bb{U}_{j+(\bar{l}_j -1)k}}{\bar{l}_j}  \Big) \Big]  \Bigg |  \,\, \Bigg|  \E \Big[ \prod_{i=1}^{\bar{l}_j-1} \exp\Big( \frac{s \bb{U}_{j+(i-1)k}}{\bar{l}_j}  \Big) \Big]- \prod_{i=1}^{\bar{l}_j-1} \E \Big[ \exp\Big ( \frac{s \bb{U}_{j+(i-1)k}}{\bar{l}_j} \Big)  \Big]   \Bigg| 
	\end{align*}
	\begin{align*}
		& \leq \Bigg \|  \prod_{i=1}^{\bar{l}_j-1} \exp\Big( \frac{s \bb{U}_{j+(i-1)k}}{\bar{l}_j}  \Big)  \Bigg \|_{\infty} \E\Big[ \Big| \E \Big[ \exp \Big( \frac{s \bb{U}_{j+(\bar{l}_j -1)k}}{\bar{l}_j}  \Big) \Big | \mathcal{F}_{j+(\bar{l}_j-2)k} \Big]- \E \Big[  \exp \Big( \frac{s \bb{U}_{j+(\bar{l}_j -1)k}}{\bar{l}_j}  \Big) \Big] \Big| \Big] \\
		&+ \Bigg \|  \exp \Big( \frac{s \bb{U}_{j+(\bar{l}_j -1)k}}{\bar{l}_j}  \Big)   \Bigg \|_{\infty} M_{\bar{l}_j-1}
	\end{align*}
	The above is then less than or equal to
	\begin{align}
		& \exp\Big( \frac{s (\bar{l}_j-1) |b-a|}{\bar{l}_j} \Big) \exp\Big( \frac{-s |a|}{\bar{l}_j} \Big) \E\Bigg[ \Bigg| \E \Bigg[ \exp \Bigg( \frac{s f(\bb{Z}_{j+(\bar{l}_j -1)k})}{\bar{l}_j}  \Bigg) \Bigg | \mathcal{F}_{j+(\bar{l}_j-2)k} \Bigg] \label{difference}\\
		&- \E \Bigg[  \exp \Bigg( \frac{s f(\bb{Z}_{j+(\bar{l}_j -1)k})}{\bar{l}_j}  \Bigg) \Bigg] \Bigg| \Bigg] + \exp\Big( \frac{s |b-a|}{\bar{l}_j} \Big)  M_{\bar{l}_j-1}.\nonumber
	\end{align}
	Note that the function $g(x)=\frac{\exp\Big( \frac{s x}{\bar{l}_j} \Big)}{\exp \Big( \frac{s |b|}{\bar{l}_j}\Big) \frac{s}{\bar{l}_j}} 1_{\mathcal{A}}(x)$ is in $\mathcal{L}_1$ for each $s$, where $\mathcal{A}=\{x:|x|\leq |b-a|\}$. We then use the projective-type representation of the $\theta$-coefficients of $f(\bb{Z})$, see Remark \ref{mix_rule1}, and obtain that
	\begin{align*}
		M_{\bar{l}_j} \leq 	\exp\Big( \frac{s \bar{l}_j |b-a|}{\bar{l}_j} \Big) \theta(k) \frac{s}{\bar{l}_j} + \exp\Big( \frac{s |b-a|}{\bar{l}_j} \Big)  M_{\bar{l}_j-1}.
	\end{align*}
	Let now, $u= \exp\Big( \frac{s |b-a|}{\bar{l}_j} \Big)$, we have that
	\begin{align*}
		M_{\bar{l}_j} &\leq \theta(k) u^{\bar{l}_j}\frac{s}{\bar{l}_j}+ u M_{\bar{l}_j-1} \\
		& \leq (\bar{l}_j-2) \theta(k) u^{\bar{l}_j}\frac{s}{\bar{l}_j} +u^{\bar{l}_j-2} \Big| \E\Big[ \exp \Big( \frac{s \bb{U}_j}{ \bar{l}_j}\Big) \exp\Big( \frac{s\bb{U}_{j+k}}{\bar{l}_j} \Big) \Big] - \E \Big[ \exp\Big( \frac{s\bb{U}_j}{\bar{l}_j} \Big) \Big] \E \Big[ \exp\Big( \frac{s \bb{U}_{j+k}}{\bar{l}_j} \Big) \Big] \Big|\\
		&\leq (\bar{l}_j-2) \theta(k) u^{\bar{l}_j}\frac{s}{\bar{l}_j} +u^{\bar{l}_j-1} \Big| \E\Big[\exp\Big( \frac{s \bb{U}_{j+k}}{\bar{l}_j} \Big) |\mathcal{F}_j \Big]- \E\Big[\exp\Big( \frac{s U_{j+k}}{\bar{l}_j} \Big) \Big] \Big|\\
		& \leq (\bar{l}_j-1) \theta(k) u^{\bar{l}_j} \frac{s}{\bar{l}_j} =  (\bar{l}_j-1) \exp\Big( s  |b-a|\Big) \theta(k) \frac{s}{\bar{l}_j} .
	\end{align*}
	In conclusion, for all $j=1,\ldots,k$ (and remembering that $ \bar{l}_j=l$ or $\bar{l}_j=l+1$)
	\[
	M_{\bar{l}_j} \leq  \exp( s \, |b-a|) \theta(k) s.
	\]
	Similar calculations apply when $\bb{U}_i=\E[f(\bb{Z}_i)]-f(\bb{Z}_i)$.	
\end{proof}

\begin{Remark}
	Note that by showing Lemma \ref{lem2} for $\bb{U}_i=\E[f(\bb{Z}_i)]-f(\bb{Z}_i)$ there is a slight change in the proof at point (\ref{difference}). However, in the end, the result (\ref{Mj}) equally holds.
\end{Remark}

\begin{Lemma}
	\label{lem3}
	Let the Assumptions of Theorem \ref{new} hold and define the process $(\bb{U}_i)_{i \in \Z}$ such that $\bb{U}_i=f(\bb{Z}_i)-\E[f(\bb{Z}_i)]$. For all $j=1,2,\ldots,k$, $l \geq 2$ and $0< s < \frac{3l}{|b-a|}$ 
	\[
	\E\Big[ \exp\Big( \frac{s \bb{V}_{j,m}}{\bar{l}_j} \Big)\Big] \leq \exp \Bigg( \frac{s^2 \E[\bb{U}_1^2] }{2l \Big( 1-\frac{s|b-a|}{3l} \Big) } \Bigg)  +\exp( s \, |b-a|) \theta(k) s
	\]
	The same result holds when defining the process $(\bb{U}_i)_{i \in \Z}$ for $\bb{U}_i=\E[f(\bb{Z}_i)] -f(\bb{Z}_i)$.
\end{Lemma}

\begin{proof}
	\begin{align}
		\E\Big[ \exp\Big( \frac{s \bb{V}_{j,m}}{\bar{l}_j} \Big)\Big] &= \E \Big[ \exp \Big( \sum_{i=1}^{\bar{l}_j} \frac{s \bb{U}_{j+(i-1)\,k}}{\bar{l}_j} \Big)  \Big] \leq \prod_{i=1}^{\bar{l}_j} \E \Big[ \exp\Big( \frac{s \bb{U}_{j+(i-1)\,k}}{\bar{l}_j} \Big) \Big] \nonumber\\ 
		&+ \Big|\E \Big[ \prod_{i=1}^{\bar{l}_j}  \exp\Big( \frac{s \bb{U}_{j+(i-1)\,k}}{\bar{l}_j} \Big) \Big] -
		\prod_{i=1}^{\bar{l}_j} \E \Big[ \exp\Big( \frac{s \bb{U}_{j+(i-1)\,k}}{\bar{l}_j} \Big)  \Big] \Big| \nonumber\\
		&= \E \Big[ \exp\Big( \frac{s \bb{U}_{j+(i-1)\,k}}{\bar{l}_j} \Big)  \Big]^{\bar{l}_j} + M_{\bar{l}_j} \label{imp}
	\end{align}
	
	We have that $\E[\bb{U}_{j+(i-1)\,k}]=0$ by definition of the process $\bb{U}$, and $\frac{\bb{U}_{j+(i-1)\,k}}{\bar{l}_j}$ satisfies the Bernstein moment condition (Remark A1 \cite{Thelight}) with $K_1=\frac{|b-a|}{3\bar{l}_j}$. Hence, for $\bar{l}_j\geq 2$ and $0 <s <\frac{3\bar{l}_j}{|b-a|}$
	\begin{equation}
		\label{Bernstein}
		\E \Big[ \exp\Big( \frac{s \bb{U}_{j+(i-1)\,k}}{\bar{l}_j} \Big)  \Big] \leq \exp \Bigg( \frac{s^2 \E[(\bb{U}_{j+(i-1)k}/\bar{l}_j)^2] }{2 \Big( 1-\frac{s|b-a|}{3\bar{l}_j} \Big) } \Bigg). 
	\end{equation}
	Because $\bar{l}_j \geq l$, we can conclude that the inequality above holds for $l \geq 2$.
	Moreover, by stationarity of the process $\bb{U}$ and since for all $j=1,2,\ldots,k$, we can bound (\ref{imp}) uniformly with respect to the index $j$ by using Lemma \ref{lem2}, and noticing that
	\begin{align*}
		&0< s < \frac{3l}{|b-a|} \leq \frac{3\bar{l}_j}{|b-a|},\\
		&\textrm{and then}\\
		&\Big( 1-\frac{s|b-a|}{3\bar{l}_j} \Big) \geq \Big( 1-\frac{s|b-a|}{3l} \Big).
	\end{align*}
	
	The same proof applies when defining $\bb{U}_i=\E[f(\bb{Z}_i)]-f(\bb{Z}_i)$.
	
\end{proof}

\begin{proof}[Proof of Theorem \ref{new}]
	By combining Lemmas \ref{lem1}, \ref{lem2}, \ref{lem3}, we can bound the process $\bb{U}:=f(\bb{Z})-\E[f(\bb{Z})]$ for $ 0< s < \frac{3l}{|b-a|}$ as follows:
	
	\begin{align}
		\E\Big[\exp\Big(s \frac{1}{m}\sum_{i=1}^m f(\bb{Z}_i)-\E[f(\bb{Z}_i)]\Big)\Big]&= \E \Big[  \exp\Big( s \frac{1}{m} \sum_{i=1}^m \bb{U}_i \Big) \Big] = \E \Big[  \exp\Big( \frac{s}{m} \sum_{j=1}^k \bb{V}_{j,m} \Big) \Big] \nonumber\\
		&= \E \Big[  \exp\Big(  \frac{s}{m} \sum_{j=1}^k \sum_{i=1}^{\bar{l}_j} \bb{U}_{j+(i-1)\,k} \Big) \Big] \nonumber\\
		&\leq \sum_{j=1}^k p_j \E \Big[  \exp\Big( \frac{s \bb{V}_{j,m}}{\bar{l}_j} \Big)  \Big] \label{ciao}\\
		&\leq  \exp \Bigg(\frac{s^2 Var(f(\bb{Z}_1))}{2 l \Big(1-\frac{s|b-a|}{3l} \Big)}  \Bigg) + \exp( s \, |b-a|) \theta(k) s \label{ciao2},
	\end{align}
	where $\bb{V}_{j,m}=\sum_{i=1}^{\bar{l}_j} \bb{U}_{j+(i-1) \,k}$.
	The inequalities (\ref{ciao}) and (\ref{ciao2}) hold because of Lemma \ref{lem1} and Lemma \ref{lem3}, respectively. We have then proved the inequality (\ref{laplace}). The same proof applies for showing the bound (\ref{laplace2}) by defining $\bb{U}_i=\E[f(\bb{Z}_i)]-f(\bb{Z}_i)$ .
\end{proof}

We remind the reader that the proof of Theorem \ref{prop_PAC1} and \ref{new_fixedtime} make use of the below Lemma that we recall for completeness.

\begin{Lemma}[Legendre transform of the Kullback-Leibler divergence function]
	\label{kl}
	For any $\pi \in \mathcal{M}^1_+(B)$, for any measurable function $h: B \to \R$ such that $\pi[\exp(h)]\leq \infty$, we have that
	\[
	\pi[\exp(h)] =\exp \Big( \sup_{\hat{\rho} \in \mathcal{M}^1_+(B)} \hat{\rho}[h]- KL(\hat{\rho},\pi) \Big),
	\]
	with the convention $\infty-\infty=-\infty$. Moreover, as soon as  $h$ is upper bounded on the support of $\pi$, the supremum with respect to $\hat{\rho}$ in the right-hand side is reached for the Gibbs distribution with Radon-Nikodym derivative w.r.t. $\pi$ equal to $\frac{\exp(h)}{\pi[\exp(h)]}$.
\end{Lemma}

The proof of Lemma (\ref{kl}) has been known since the work of Kullback \cite{KU} in the case of a finite space $B$, whereas the general case has been proved by Donsker and Varadhan \cite{DV76}. Given this result, we are now ready to prove our PAC Bayesian bounds.

\begin{proof}[Proof of Theorem \ref{prop_PAC1}]
	Let us choose $f(s)=\frac{s}{\epsilon}$ for $0<s<\epsilon$ , which satisfies the assumptions of Theorem \ref{new} and has support in $[0,1]$. We have that $R^{\epsilon}(h)-r^{\epsilon}(h)=\frac{\epsilon}{m}(\sum_{i=1}^m \E[f(\bb{L}_i^{\epsilon})]-f(\bb{L_i^{\epsilon}}))$. We have in this case that $\bb{U}_i=\E[f(\bb{L_i})]-f(\bb{L_i^{\epsilon}})$ such that $\E[\bb{U}_i]=0$ and $|\bb{U}_i|\leq 1$. Note that the process $f(\bb{L}_i^{\epsilon})_{i \in \Z}$ has the same $\theta$-weak coefficients of the process $(\bb{L}_i^{\epsilon})_{i \in \Z}$ because $f$ is a 1-Lipschitz function. By Theorem \ref{new} applied for $0 < \epsilon \sqrt{l} < 3 \sqrt{l}$, and the bound (\ref{bound3}), 
	\begin{align}
		\E\Big[\exp\Big(\sqrt{l}\, (R^{\epsilon}(h)-r^{\epsilon}(h))  \Big)\Big]=\E\Big[\exp\Big(\epsilon \sqrt{l} \frac{1}{m} \sum_{i=1}^m \bb{U}_i \Big)\Big]&\leq \exp\Big( \frac{3\epsilon^2 }{2(3-\epsilon)  }  \Big) \nonumber\\
		&+ 3 \sqrt{l} \exp( 3 \sqrt{l})\theta(k) \label{uao}
	\end{align}
	where the last equality holds because of the particular shape of the chosen function $f$.
	
	We follow the scheme of proof developed by \cite{BG16}.
	\begin{align}
		\label{scheme}
		&\sqrt{l}\, \hat{\rho}[R^{\epsilon}(h)]-\hat{\rho}[r^{\epsilon}(h)]= \hat{\rho}[\sqrt{l}\, (R^{\epsilon}(h)-r^{\epsilon}(h))] \hspace{0.5cm} \nonumber \\
		&\leq KL(\hat{\rho} || \pi)  + \log(\pi[\exp(\sqrt{l} \, (R^{\epsilon}(h)-r^{\epsilon}(h)))]) \hspace{0.5cm} \textrm{($\mathbb{P}$-almost surely by Lemma \ref{kl})}.
	\end{align}
	We have that $ \pi[\exp(\sqrt{l} \, (R^{h}(\beta)-r^{\epsilon}(h)))]:= \bb{A_{m}}$ is a random variable on $\bb{S_m}$. By Markov's inequality, for $\delta \in (0,1)$
	\[
	\mathbb{P} \Big( \bb{A_m} \leq \frac{\E[\bb{A_m}]}{\delta} \Big) \geq 1-\delta.
	\]
	This in turn implies that with probability at least $1-\delta$ over $\bb{S_m}$ 
	\begin{align}
		\hat{\rho}[R^{\epsilon}(h)]-\hat{\rho}[r^{\epsilon}(h)] &\leq \frac{KL(\hat{\rho} || \pi) + \log \frac{1}{\delta}}{\sqrt{l}} +\frac{1}{\sqrt{l}} \log\Big(\pi\Big[ \E \Big[\exp \Big( \sqrt{l} (R^{\epsilon}(h)-r^{\epsilon}(h)) \Big) \Big] \Big]\Big) \label{p1} \\
		&\leq  \frac{KL(\hat{\rho} || \pi) + \log \frac{1}{\delta}}{\sqrt{l}} +\frac{1}{\sqrt{l}} \log\Big(\pi\Big[ \exp\Big( \frac{3\epsilon^2 }{2(3-\epsilon)  }  \Big)  + \nonumber\\
		& 3 \sqrt{l} \exp( 3 \sqrt{l})\theta(k) \Big]\Big) \label{p2},
	\end{align}
	where (\ref{p1}) holds by swapping the expectation over $\bb{S}_m$ and over $\pi$ using Fubini's Theorem, and (\ref{p2}) is obtained by using (\ref{uao}).
	Similarly, it can be proven that the bound (\ref{p2}) holds for $\hat{\rho}[R^{\epsilon}(h)]-\hat{\rho}[r^{\epsilon}(h)]$. We then conclude by using an union bound.
\end{proof}

\begin{proof}[Proof of Theorem \ref{oracle}]
	Let $\frac{d\bar{\rho}}{d \pi}=\frac{\exp(-\sqrt{m}r^{\epsilon}(h))}{\pi[\exp(-\sqrt{m}r^{\epsilon}(h))]}$.
	%
	By Lemma \ref{kl}, we have that
	
	\[
	\bar{\rho}=arg\inf_{\hat{\rho}} \Big( KL(\hat{\rho} || \pi) - \hat{\rho}[-\sqrt{m}r^{\epsilon}(h)]  \Big)= arg\inf_{\hat{\rho}} \Big ( \frac{KL(\hat{\rho} || \pi)}{\sqrt{m}} + \hat{\rho}[r^{\epsilon}(h)]  \Big), 
	\]
	and by using one side of the bound (\ref{PAC1}) for $k=1$, for all $  \delta \in (0,1)$ 
	\begin{align*}
		\label{oracle_p}
		\mathbb{P}\Big\{\bar{\rho}[R^{\epsilon}(h)] \leq \inf_{\hat{\rho}} \Big ( \hat{\rho}[r^{\epsilon}(h)] &+ \Big(KL(\hat{\rho},\pi)+ \log \Big(\frac{1}{\delta}\Big)  \Big) \frac{1}{\sqrt{m}} \Big) \\ &+\frac{1}{\sqrt{m}} \log\Big(\pi\Big[ \exp\Big( \frac{3\epsilon^2 }{2(3-\epsilon)  }  \Big)  
		+ 3 \sqrt{m} \exp( 3 \sqrt{m})\theta(k) \Big]\Big)  \Big\} \geq 1-\delta
	\end{align*}
	We now substitute to $\hat{\rho}[r^{\epsilon}(\beta)]$ the other side of the bound (\ref{PAC1}), we obtain that 
	\begin{align}
		\mathbb{P}\Big\{\bar{\rho}[R^{\epsilon}(h)] \leq \inf_{\hat{\rho}} \Big ( \hat{\rho}[R^{\epsilon}(h)] &+ \Big(KL(\hat{\rho},\pi)+ \log \Big(\frac{1}{\delta}\Big)  \Big) \frac{2}{\sqrt{m}} \Big) \\
		&+\frac{2}{\sqrt{m}} \log\Big(\pi\Big[ \exp\Big( \frac{3\epsilon^2 }{2(3-\epsilon)  }  \Big)  
		+ 3 \sqrt{m} \exp( 3 \sqrt{m})\theta(k) \Big]\Big) \Big\} \geq 1-2\delta
	\end{align}
	by using a union bound.
\end{proof}

\begin{proof}[Proof of Theorem \ref{anytimePAC}]

	First of all, we show that the process $(\exp(f_i(\bb{S},h))_{i \in \N_0}$ is a super-martingale. Note that
	\[
	\left\{ \begin{array}{ll}
		f_0(\bb{S},h) & =0 \\
		f_m(\bb{S},h)& = \eta \sum_{i=1}^m (L^{\epsilon}(h(X_i),Y_i)- \E[L^{\epsilon}(h(X_i),Y_i)| \mathcal{F}_{i-1}]) -\frac{\eta^2}{2} m \epsilon^2 , \,\, \textrm{for $m \in \N$}.
	\end{array} \right.
	\]
	In fact, the process $\eta \sum_{i=1}^m (L^{\epsilon}(h(X_i),Y_i)-\E[L^{\epsilon}(h(\bb{X}_i),\bb{Y}_i)])$ minus the residual process is equal to a sum of martingale differences with respect to the filtration $(\mathcal{F}_{m})_{m \in \N_0}$, namely
	$$\eta \sum_{i=1}^m (L^{\epsilon}(h(X_i),Y_i)-\E[L^{\epsilon}(h(\bb{X}_i),\bb{Y}_i)| \mathcal{F}_{i-1}]).$$ 
	
	We have that 
	\begin{align*}
		\E_{m-1}[\exp(f_m(\bb{S},h)]
		&=\exp(f_{m-1}(\bb{S},h)) \E_{m-1}[ \exp(\eta(L^{\epsilon}(h(\bb{X}_m),\bb{Y}_m) \\ & - \E[L^{\epsilon}(h(\bb{X}_m),\bb{Y}_m)|\mathcal{F}_{m-1}] ) 
		-\frac{\eta^2}{2} \epsilon^2)] \leq 1,
	\end{align*}
	by applying the conditional  Hoeffding’s Lemma.	
	By direct application of Theorem 4 in \cite{Uniform}, we obtain that for all $m \geq 1$ and $\delta \in (0,1)$, we have that
	\begin{align*}
		\hat{\rho}\Big[\eta \sum_{i=1}^m (L^{\epsilon}(h(X_i),Y_i)&-\E[L^{\epsilon}(h(\bb{X}_i),\bb{Y}_i)]) \Big] \leq KL(\hat{\rho},\pi) + \log\Big(\frac{1}{\delta} \Big)+ \frac{\eta^2}{2} m \epsilon^2 \\
		&+ \hat{\rho} \Big[ \eta \sum_{i=1}^m \E[L^{\epsilon}(h(\bb{X}_i),\bb{Y}_i)|\mathcal{F}_{i-1}] -\E[L^{\epsilon}(h(\bb{X}_i),\bb{Y}_i)] \Big]
	\end{align*}
	with probability at least $1-\delta$.
	Dividing both side of the inequality by $m$, we obtain that for all $m \geq 1$ and $\delta \in (0,1)$ 
	\begin{align}
		\hat{\rho}[r^{\epsilon}(h)-R^{\epsilon}(h)] &\leq \frac{KL(\hat{\rho},\pi) + \log\Big(\frac{1}{\delta}\Big)}{\eta m} + \frac{\eta}{2} \epsilon^2 \nonumber \\
		&+ \hat{\rho} \Big[ \frac{1}{m} \sum_{i=1}^m\E[L^{\epsilon}(h(\bb{X}_i),\bb{Y}_i)|\mathcal{F}_{i-1}] -\E[L^{\epsilon}(h(\bb{X}_i),\bb{Y}_i)] \Big]. \label{coeff}
	\end{align}
	with probability at least $1-\delta$.
	
\end{proof}
\begin{proof}[Proof of Theorem \ref{new_fixedtime}]
	
	Let us apply the scheme of proof used in \cite[Theorem 1]{Hostile} to the residual process $\Delta_m(h)=\frac{1}{m} \sum_{i=1}^m\E[L^{\epsilon}(h(\bb{X}_i),\bb{Y}_i)|\mathcal{F}_{i-1}] -\E[L^{\epsilon}(h(\bb{X}_i),\bb{Y}_i)]$.
	We have that 
	\begin{align*}
		\hat{\rho}[ |\Delta_m(h)|]&= \int_{\mathcal{H}} |\Delta_m(h)| \frac{d\hat{\rho}}{d\pi} d\pi\\
		&\leq \Big( \int_{\mathcal{H}}  |\Delta_m(h)|^2 \, d\pi\Big)^{\frac{1}{2}} \Big( \int_{\mathcal{H}} \Big(\frac{d\hat{\rho}}{d\pi}\Big)^2 \, d\pi \Big)^{\frac{1}{2}}  \,\,\,\,\text{(Cauchy-Schwarz's Inequality)} \\
		&\leq \Big( \epsilon \int_{\mathcal{H}}  |\Delta_m(h)| \, d\pi\Big)^{\frac{1}{2}} \Big( D_{\phi_2-1}(\hat{\rho},\pi)+1 \Big)^{\frac{1}{2}}.
	\end{align*}
	We then apply the Markov's inequality to $\int_{\mathcal{H}}  |\Delta_m(h)| \, d\pi$, and for $\delta \in (0,1)$ obtain that with probability at least $1-\delta$
	\begin{align*}
		\hat{\rho}[ |\Delta_m(h)|]	&\leq \Big( \epsilon  \frac{\E[\int_{\mathcal{H}}  |\Delta_m(h)| \, d\pi]}{\delta} \Big)^{\frac{1}{2}} \Big( D_{\phi_2-1}(\hat{\rho},\pi)+1 \Big)^{\frac{1}{2}} \\
		&\leq \Big( \epsilon  \frac{\pi[ \E[ |\Delta_m(h)|]}{\delta} \Big)^{\frac{1}{2}} \Big( D_{\phi_2-1}(\hat{\rho},\pi)+1 \Big)^{\frac{1}{2}}\,\,\,\,\text{(Fubini's Theorem)}\\
		&\leq \Big( \epsilon \frac{\pi[ \theta(1)]}{\delta} \Big)^{\frac{1}{2}} \Big( D_{\phi_2-1}(\hat{\rho},\pi)+1 \Big)^{\frac{1}{2}} \,\,\, \text{(Definition \ref{mix2})}
	\end{align*}
	The last inequality holds because $\bb{L}$ is a $\theta$-weakly dependent process as proven in Proposition \ref{mmaf_absolute}, and the truncation of this process trough the accuracy level $\epsilon$ means applying a projection function to the process $\bb{L}$ which has $Lip(h)=1$, see Erratum in the arXiv version of \cite{CS19} for a detailed explanation.
	We can then apply the projective-type representation of the $\theta$-coefficients in Remark \ref{mix_rule1}. We then use a union bound to combine the result above with the any-time bound in Corollary \ref{any_time_full}.
\end{proof}

\section*{Acknowledgments}
We thank the \emph{German Research Foundation (DFG)} for the financial support through the research grant GZ:CU 512/1-1. Moreover, we are grateful to the two anonymous referees and the editor for their helpful and insightful comments which considerably improved this work.

\bibliographystyle{plain}
\bibliography{Immaf}

\end{document}